\documentclass[10pt,twocolumn,letterpaper]{article}

\usepackage{iccv}
\usepackage{times}
\usepackage{epsfig}
\usepackage{graphicx}
\usepackage{amsmath}
\usepackage{amssymb}

\usepackage[pagebackref=true,breaklinks=true,letterpaper=true,colorlinks,bookmarks=false]{hyperref}

\usepackage[table]{xcolor}
\usepackage{amssymb}
\usepackage{amsmath,amssymb,amsfonts}
\usepackage{graphicx}
\usepackage[draft]{fixme}
\usepackage[boxruled]{algorithm2e}
\usepackage{paralist}                           
\usepackage{bm}
\usepackage{float}
\usepackage{multirow}
\usepackage{booktabs}
\usepackage{nicefrac}
\usepackage{nth}
\usepackage[bf,small,skip=0pt]{caption}
\usepackage{subcaption}
\usepackage{wrapfig}
\usepackage{dsfont}
\usepackage{mathtools}
\usepackage{dsfont}
\usepackage{adjustbox}
\usepackage{multirow}
\usepackage{url}
\usepackage{color}
\usepackage{wrapfig}
\usepackage{mathrsfs}
\usepackage{ctable}
\usepackage{amsthm}
\usepackage[makeroom]{cancel}

\usepackage{cleveref}

\crefformat{section}{\S#2#1#3} 
\crefformat{subsection}{\S#2#1#3}
\crefformat{subsubsection}{\S#2#1#3}

\usepackage{enumitem}
\setlist[itemize]{leftmargin=7.5mm}

\DeclareMathOperator*{\argmax}{arg\,max}
\DeclareMathOperator*{\argmin}{arg\,min}

\newtheorem{theorem}{Theorem}[]
\newtheorem{lemma}{Lemma}[]
\newtheorem{definition}{Definition}[]
\newtheorem{innercustomthm}{Theorem}
\newenvironment{customthm}[1]
  {\renewcommand\theinnercustomthm{#1}\innercustomthm}
  {\endinnercustomthm}

\numberwithin{equation}{section}

\newcommand{\supp}{Appx.}

\iccvfinalcopy 

\def\iccvPaperID{6405} 

\ificcvfinal\fi

\begin{document}

\title{Local Temperature Scaling for Probability Calibration}

\author{Zhipeng Ding \quad Xu Han \quad Peirong Liu \quad Marc Niethammer\\ 
University of North Carolina at Chapel Hill, Chapel Hill, USA\\ 
{\tt\small \{zp-ding, xhs400, peirong, mn\}@cs.unc.edu}
}




\maketitle
\ificcvfinal\thispagestyle{empty}\fi

\begin{abstract}
For semantic segmentation, label probabilities are often uncalibrated as they are typically only the by-product of a segmentation task. Intersection over Union (IoU) and Dice score are often used as criteria for segmentation success, while metrics related to label probabilities are not often explored. However, probability calibration approaches have been studied, which match probability outputs with experimentally observed errors. These approaches mainly focus on classification tasks, but not on semantic segmentation. Thus, we propose a learning-based calibration method that focuses on multi-label semantic segmentation. Specifically, we adopt a convolutional neural network to predict local temperature values for probability calibration. One advantage of our approach is that it does not change prediction accuracy, hence allowing for calibration as a post-processing step. Experiments on the COCO, CamVid, and LPBA40 datasets demonstrate improved calibration performance for a range of different metrics. We also demonstrate the good performance of our method for multi-atlas brain segmentation from magnetic resonance images.
\end{abstract}

\section{Introduction}

With the development of deep convolutional neural networks (CNNs), the accuracy of semantic segmentation has improved dramatically~\cite{cciccek20163d,long2015fully}. However, ideally semantic segmentation networks should not only be accurate, but should also indicate when they are likely incorrect. For example, an autonomous driving system might use deep convolutional neural networks to analyze a real-time scene from a camera~\cite{bojarski2016end}, the associated semantic segmentation of street scenes should provide accurate detections of pedestrians and other vehicles, and the system should recognize when such predictions are unreliable. Another example is the segmentation of brain tumors with a CNN~\cite{havaei2017brain}. If the segmentation network can not confidently segment critical regions of the brain, then a medical expert should decide or be alerted to such doubtful regions. 
Thus, it is important for semantic segmentation networks to generate both accurate label predictions \emph{and} accurate confidence measures.

However, due to overfitting, CNNs for semantic segmentation tend to be overconfident about predicted labels~\cite{garcia2017review,guo2017calibration,jena2019bayesian,li2019overfitting}. Approaches for joint prediction and calibration exist~\cite{kumar2018trainable,maddox2019simple,milios2018dirichlet,mukhoti2020calibrating}. However, they require changing the learning task and typically strive for calibration, but do not guarantee it. An alternative approach is to calibrate the resulting probabilities of a model via \emph{post-processing} so that they better reflect the true probabilities of being correct. This is the kind of approach we consider here as it easily applies to pre-trained networks and can even benfit joint prediction/calibration approaches. Probability calibration, first studied for classification~\cite{platt1999probabilistic}, generally addresses this problem via a hold-out validation dataset. 
  

Existing calibration approaches still have several limitations: (1) Most of the probability calibration approaches are designed for classification, thus are not guaranteed to work well for semantic segmentation (where it is also more challenging to annotate on a pixel/voxel level); (2) While there is limited work discussing probability calibration for semantic segmentation, this work either only applies to specific types of models (e.g., Bayesian neural networks~\cite{jena2019bayesian}) or only implicitly improves calibration performance (e.g., via model ensembling~\cite{mehrtash2019confidence} or multi-task learning~\cite{karimi2020improving}); (3) Most methods are designed to work for \emph{binary} classifications and approach multi-class problems by a decomposition into $k$ one-vs-rest binary calibrations (where $k$ denotes the number of classes). However, such a decomposition does not guarantee overall calibration (only for the individual subproblems before normalization) and the classification accuracy of the trained model may change after calibration as the probability order of labels may change.

Our goal is to develop a \emph{post-processing} calibration method for multi-label semantic segmentation, which retains label probability order and, therefore, a model's segmentation accuracy. Our work is inspired by temperature scaling (TS)~\cite{guo2017calibration} for classification probability calibration. As TS 
determines only \emph{one} global scaling constant, it cannot capture spatial miscalibration changes in images. We therefore (1) extend TS to multi-label semantic segmentation and (2) make it adaptive to local image changes. 

\textbf{Our contributions} are: 
(1) \emph{Spatially localized probability calibration:} We propose a learning-based local TS method that predicts a separate temperature scale for each pixel/voxel. (2) \emph{Completely separated accuracy-preserving post-processing:} Our approach is completely separated from the segmentation task, leaving the prediction accuracy unchanged. (3) \emph{Theoretical justification:} We provide a theoretical analysis for the effectiveness of our approach.
  (4) \emph{Comprehensive analysis:} We provide definitions and evaluation metrics for probability calibration for semantic segmentation and validate our approach both qualitatively and quantitatively. (5) \emph{Practical application:} We successfully apply our calibrated probabilities for multi-atlas segmentation label fusion in the field of medical image analysis. 

\section{Related Work}
\label{sec:related_work}

A variety of calibration approaches have been proposed, but none addresses our target setting.

\textbf{Bin-based Approaches.}
Non-parametric histogram binning~\cite{zadrozny2001obtaining} uses the average number of positive-class samples in each bin as the calibrated probability. Isotonic regression~\cite{zadrozny2002transforming} extends this approach by jointly optimizing bin boundaries and bin predictions
; it is one of the most popular non-parametric calibration methods. ENIR~\cite{naeini2016binary} further extends isotonic regression by relaxing the monotonicity assumption of isotonic regression. These bin-based methods do not consider correlations among neighboring pixels/voxels in semantic segmentation, while our proposed method captures correlations via convolutional filters.

\textbf{Temperature Scaling Approaches.}
Platt scaling~\cite{platt1999probabilistic} uses logistic regression for probability calibration.
Matrix scaling~\cite{guo2017calibration}, vector scaling~\cite{guo2017calibration}, and temperature scaling~\cite{hinton2015distilling,guo2017calibration} all generalize Platt scaling to multi-class calibration, among which temperature scaling is both effective and the simplest. ATS~\cite{mozafari2018attended} extends temperature scaling by using the conditional distribution on each class to address the calibration challenge on small validation datasets, for noisy labels, and highly accurate networks. BTS~\cite{ji2019bin} extends temperature scaling to a bin-wise setting and also uses data augmentation inside each bin to improve the calibration performance. However, unlike our approach (which extends temperature scaling) none of these approaches considers spatial variations for probability calibration.

\textbf{Bayesian Approaches.}
BBQ~\cite{naeini2015obtaining} extends binning via Bayesian averaging of the probabilities produced by all possible binning schemes. Bayes-Iso~\cite{allikivi2019non} extends isotonic regression by using Bayesian isotonic calibration to allow for more flexibility in the monotonic fitting and smoothness.
Jena et al.~\cite{jena2019bayesian} proposed to use a utility function focusing on the intermediate-layers of a Bayesian deep neural network to calibrate probabilities for image segmentation. Maronas et al.~\cite{maronas2020calibration} proposed decoupled Bayesian neural networks to calibrate classification probabilities. Bin-based Bayesian methods do not consider pixel/voxel correlations. Bayesian neural networks can capture spatial correlations, but require a Bayesian formulation in the first place. 
Furthermore, while Bayesian uncertainty quantification~\cite{kendall2017uncertainties} helps probability calibration, it may also not achieve it ({\supp}~\ref{app:additional_related_work}). Instead, our approach considers pixel/voxel correlations and can be used as a post-processing approach for any semantic segmentation method which generates probability outputs.

\textbf{Other Approaches.}
Mehrtash et al.~\cite{mehrtash2019confidence} found that model ensembling improves confidence calibration for medical image segmentation. A similar conclusion was also found in \cite{lakshminarayanan2017simple,zhang2020mix}, where an ensemble is used to produce good predictive uncertainty estimates. Karimi et al.~\cite{karimi2020improving} showed that multi-task learning can yield better-calibrated predictions than dedicated models trained separately. 
Note that ensembling or multi-task learning does not directly address probability calibration, 
instead they provide insight on how to obtain a better calibrated segmentation model. 
Leathart et al.~\cite{leathart2017probability} improved the calibration of classification tasks by building a decision tree over input tabular data, where the leaf nodes correspond to different calibration models.
Further, beta calibration~\cite{kull2017beyond} extends logistic calibration to overcome the situation where per-class score distributions are heavily skewed. Dirichlet calibration~\cite{kull2019beyond} uses the Dirichlet distribution to generalize beta calibration to multi-class problems. Rahimi et al.~\cite{rahimi2020intra} proposed to use neural network based intra order-preserving functions for calibration. These methods are also not directly designed for probability calibration of semantic segmentation, but focus on classification. Learning algorithms~\cite{kumar2018trainable,maddox2019simple,milios2018dirichlet,mukhoti2020calibrating} that jointly consider prediction and calibration also exist. 
Although they can help mitigate miscalibrations, they typically cannot entirely remove it. In fact, they can also benefit from our post-processing approach (\cref{sec:Tiramisu_exp}).
 
\section{Methodology}
\label{sec:method}
\subsection{Problem Statement}
\label{sec:problem_statement}

Our goal is the calibration of the predicted probabilities of deep semantic segmentation CNNs. Assume there is a pre-trained neural network $\mathcal{F}$, with an image $I$ as the input, which outputs a vector of logits at each location $x$. Each logit corresponds to a label, and the logit value reflects the label confidence. The predicted label is the one with the largest logit value; the corresponding confidence (probability of correctness) for each pixel/voxel is usually obtained via softmax of the logits. Specifically, the predicted confidence map and the corresponding segmentation map are
\begin{align}
\hat{P} (x) &= \max_{l \in L} \sigma_{SM}(\textbf{z}(x))^{(l)} = \max_{l \in L} \frac{\text{exp}(\textbf{z}(x)^{(l)})}{\sum_{j \in L} \text{exp}(\textbf{z}(x)^{(j)})}, \nonumber \\
\hat{S} (x) &=  \argmax_{l \in L} \textbf{z}(x)^{(l)},
\end{align}
where $\sigma_{SM}$ is the softmax function, $x$ denotes position, $L$ is the set of all labels, $l$ is the label index and $\textbf{z}(x)^{(l)} = z_l(x)$ is the logit that corresponds to label $l$ at location $x$.

The goal of probability calibration is to ensure that the confidence map $\hat{P}$ represents a true probability. For example, given a $10\times 10$ image, with label confidence of $0.7$ for each pixel, we would expect that 70 pixels should be correctly segmented. This can be formalized as follows:
\vspace{-0.5mm}
\begin{definition}
\label{def:perfect_calibration}
A semantic segmentation is perfectly calibrated in region $\Omega$ if 
\begin{equation}
\label{perfect_calibration}
\mathbb{P} (\hat{S}(x) = S(x) | \hat{P}(x) = p) = p,\, \forall p \in [0, 1], x \in \Omega
\end{equation}
where $S(x)$ and $\hat{S}(x)$ are the true and predicted segmentations at location $x$, respectively, $\hat{P}(x)$ is the confidence of the prediction $\hat{S}(x)$, and $\mathbb{P}$ is the probability measure.
\end{definition}
\vspace{-0.5mm}
In short, if the observed probability is the true probability, then the semantic segmentation model is well-calibrated. As it is difficult to work directly with this definition to assess miscalibration, we extend several visual and quantitative metrics~\cite{degroot1983comparison,murphy1977reliability,naeini2015obtaining,niculescu2005predicting,nixon2019measuring}, which have previously been proposed in the context of classification.

\vspace{-0.5mm}
\subsection{Calibration Setup}
\label{sec:calibration_setup}
\vspace{-0.5mm}
Assume the data split for a semantic segmentation network $\mathcal{F}$ is $D_{train}$ / $D_{val}$ / $D_{test}$, i.e. $\mathcal{F}$ is trained on the $D_{train}$ dataset, validated on the $D_{val}$ dataset to choose the best model, and finally tested on the $D_{test}$ dataset. Note that $D_{train}$, $D_{val}$, and $D_{test}$ are disjoint datasets. Miscalibration can be observed when evaluating $\mathcal{F}$ on $D_{test}$ for probability-related measures. Our goal is to calibrate the probability output of $\mathcal{F}$ on $D_{test}$. To this end, we train a calibration model $\mathcal{C}$ on the hold-out validation dataset $D_{val}$ via cross entropy loss, to obtain a better calibrated probability output of $\mathcal{F}$ on $D_{test}$. 
\vspace{-0.5mm}
\subsection{TS for Probability Calibration}
\label{sec:temp_scaling}
\vspace{-0.5mm}
Temperature scaling~\cite{guo2017calibration} has been proposed as a simple extension of Platt scaling~\cite{platt1999probabilistic} for post-hoc probability calibration for multi-class classifications. Specifically, temperature scaling estimates a single scalar parameter $T \in \mathbb{R}^+$, i.e., the temperature, to calibrate probabilities: $\hat{q} = \max_{l \in L} \sigma_{SM} (\textbf{z}/T)^{(l)}$, where $\hat{q}$ is the calibrated probability. 

We can directly extend temperature scaling to semantic segmentation by estimating \emph{one} global parameter $T \in \mathbb{R}^+$ for all pixels/voxels of all images: $\hat{\mathbb{Q}}_i (x, T) = \max_{l \in L} \sigma_{SM} (\textbf{z}_i (x)/T)^{(l)}$, where $\hat{\mathbb{Q}}_i$ is the calibrated probability map for the $i$-th image. As in~\cite{guo2017calibration}, we obtain this optimal value for $T$ by minimizing the following negative log-likelihood (NLL) w.r.t. a hold-out validation dataset:
\vspace{-0.5mm}
\begin{multline}
\label{eq:TS_opt}
T^* = \argmin_{T}  \left(-\sum_{i=1}^n \sum_{x \in \Omega} \log \Big( \sigma_{SM} \big ( \textbf{z}_i (x)/T \big )^{(S_i(x))} \Big)\right) \\
s.t. \quad T > 0,
\end{multline}
where $\Omega$ denotes the image space and $n$ the number of training images. However, temperature scaling in this way assumes that each image has the same distribution (i.e., the same temperature, $T$, for all images), which is unrealistic. We therefore propose to relax this assumption as follows:
\begin{definition}
\label{def:IBTS} 
Image-based temperature scaling (IBTS):

\begin{equation}
\label{eq:image_temperature_scaling}
\hat{\mathbb{Q}}_i (x, T_i) = \max_{l \in L} \sigma_{SM} (\textbf{z}_i (x)/T_i)^{(l)},
\end{equation} 
where $T_i \in \mathbb{R}^+$ is image-dependent.
\end{definition}

While this at first seems like a minor change to the standard temperature scaling approach, it is important to note that moving to an image-based temperature value, $T_i$ requires us to \emph{learn} a regressor which predicts this temperature value for each image, $I$. 
Therefore, we use a CNN~\cite{gu2018recent} to learn a mapping from ($\textbf{z}_i$, $I_i$) to $T_i$. Suppose the network is $\mathscr{F}$, then the optimization is
 \vspace{-1mm}
\begin{multline}
\label{eq:IBTS_opt}
\theta^* = \argmin_{\theta}  -\sum_{i=1}^n \sum_{x \in \Omega} \log \Big( \sigma_{SM} \big ( \frac{\textbf{z}_i (x)}{\mathscr{F}(\theta, \textbf{z}_i, I_i)} \big )^{(S_i(x))} \Big)  \\
s.t. \quad \mathscr{F}(\theta, \textbf{z}_i, I_i) > 0,
\end{multline}
where $\theta$ are the parameters of the network $\mathscr{F}$. The calibrated probability can be obtained by substituting $T_i^* = \mathscr{F}(\theta^*, \textbf{z}_i, I_i)$ in Eq.~\eqref{eq:image_temperature_scaling}.

\subsection{Local TS for Probability Calibration}
\label{sec:local_temp_scaling}

\begin{figure}
  \centering
  \includegraphics[width=\linewidth]{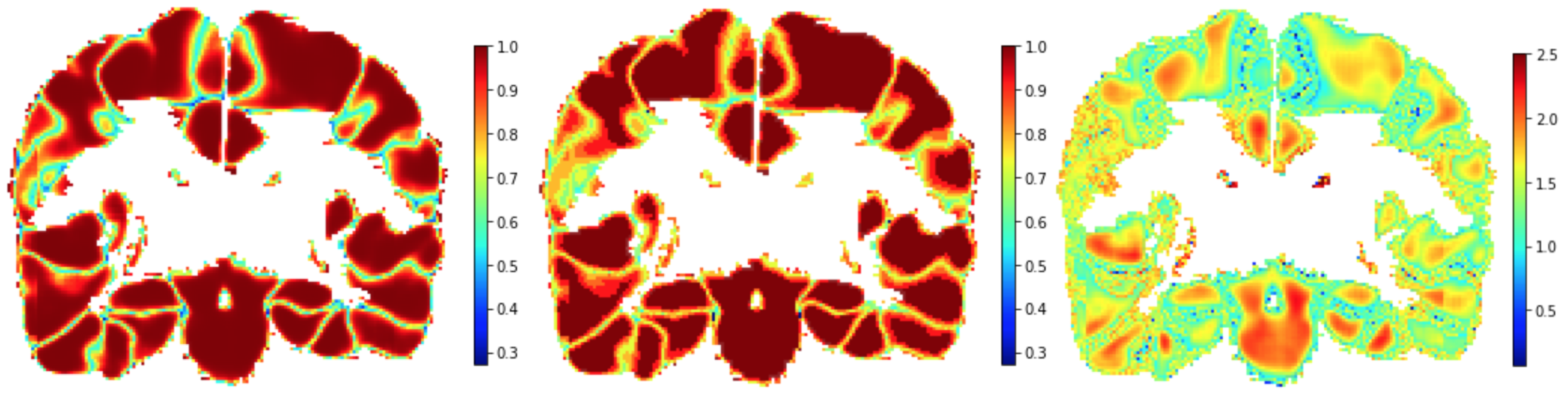}
  \caption{Left: Predicted probabilities (confidence) by a U-Net in \cref{sec:UNet_exp}. Middle: Average accuracy of each bin for 10 bins of reliability diagram with an equal bin width
  indicating different probability ranges that need to be optimized for different locations. Right: Temperature value map obtained via optimization, revealing different optimal localized TS values at different locations.}
  \label{fig:loc_diff}
\end{figure}

\begin{figure*}
  \centering
  \includegraphics[width=0.9\textwidth]{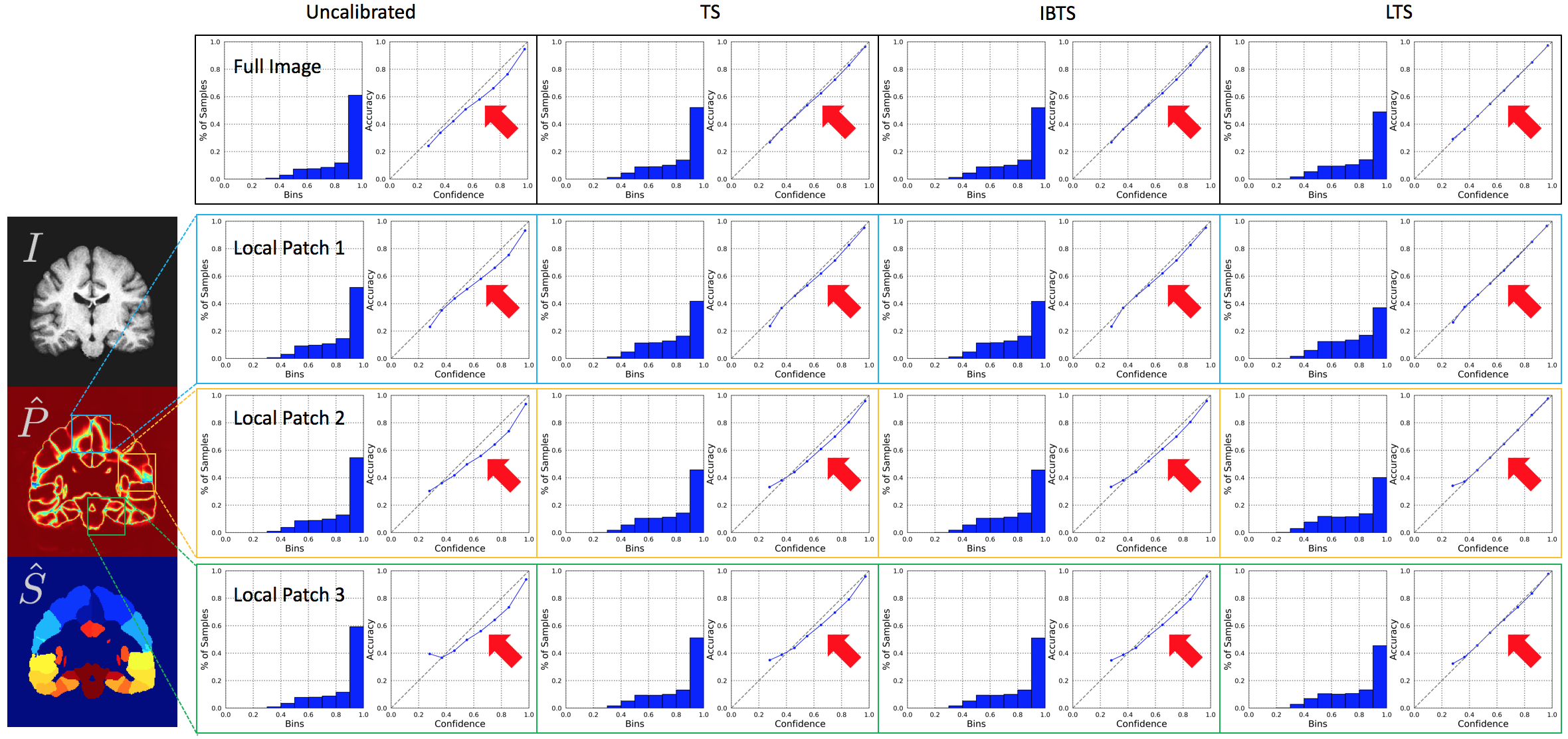}
  \caption{An example of global and local reliability diagrams for different methods for a U-Net segmentation experiment (\cref{sec:UNet_exp}). $I$ is the image, $\hat{P}$ is the predicted uncalibrated probability, and $\hat{S}$ is the predicted segmentation. Figures are displayed in couples, where the left figure is the probability distribution of pixels/voxels while the right figure is the reliability diagram (See {\supp}~\ref{app:metrics} for definitions). The top row shows the global reliability diagrams for different methods for the entire image. The three rows underneath correspond to local reliability diagrams for the different methods for different local patches. Note that TS and IBTS can calibrate probabilities well across the entire image. Visually, they are only slightly worse than LTS. However, when it comes to local patches, LTS can still successfully calibrate probabilities while TS and IBTS can not. In general, LTS improves local probability calibrations. More results are in {\supp}~\ref{app:reliability_diagram}.}
  \vspace{-1mm}
  \label{fig:local_rd}
\end{figure*}

Probabilities predicted by a deep CNN vary by location. Fig.~\ref{fig:loc_diff} illustrates that object interiors can usually be accurately predicted while predictions on boundary or near-boundary locations are more ambiguous. Thus the optimal temperature value may vary across locations. However, using a global parameter, $T$, or an image-based parameter, $T_i$, cannot account for such spatial variations. That this is a practical concern is illustrated in the uncalibrated reliability diagrams of Fig.~\ref{fig:local_rd} which shows that the confidence-vs-accuracy relation may indeed vary across an image. Hence, spatial variations should be considered for semantic segmentation. Therefore, we propose the following local temperature scaling (LTS) approach.

\begin{definition}
\label{def:LTS}
Local temperature scaling (LTS):
\begin{equation}
\label{eq:local_temperature_scaling}
\hat{\mathbb{Q}}_i (x, T_i(x)) = \max_{l \in L} \sigma_{SM} (\textbf{z}_i (x)/T_i (x))^{(l)},
\end{equation} 
where $T_i (x) \in \mathbb{R}^+$ is image \emph{and} location dependent.
\end{definition}

For $T_i(x) = 1$, no calibration occurs as the logits $\textbf{z}_i(x)$ do not change. For $T_i(x)>1$, confidence will be reduced, which helps counteract overconfident predictions. As $T_i(x) \to \infty$, the calibrated probabilities will approach $1/|L|$, which represents maximum uncertainty. For $T_i(x)<1$, prediction confidence will be increased. This will be helpful to counteract underconfident predictions. Lastly, as $T_i(x) \to 0$, the calibrated probabilities will become binary  ($\in\{0,1\}$), which represents minimum uncertainty. As $T_i(x)$ is positive, such a local scaling does not change the ordering of the probabilities over the different classes. Hence, the segmentation accuracy remains unchanged.

\begin{figure}[h!]
  \centering
  \includegraphics[width=0.9\linewidth]{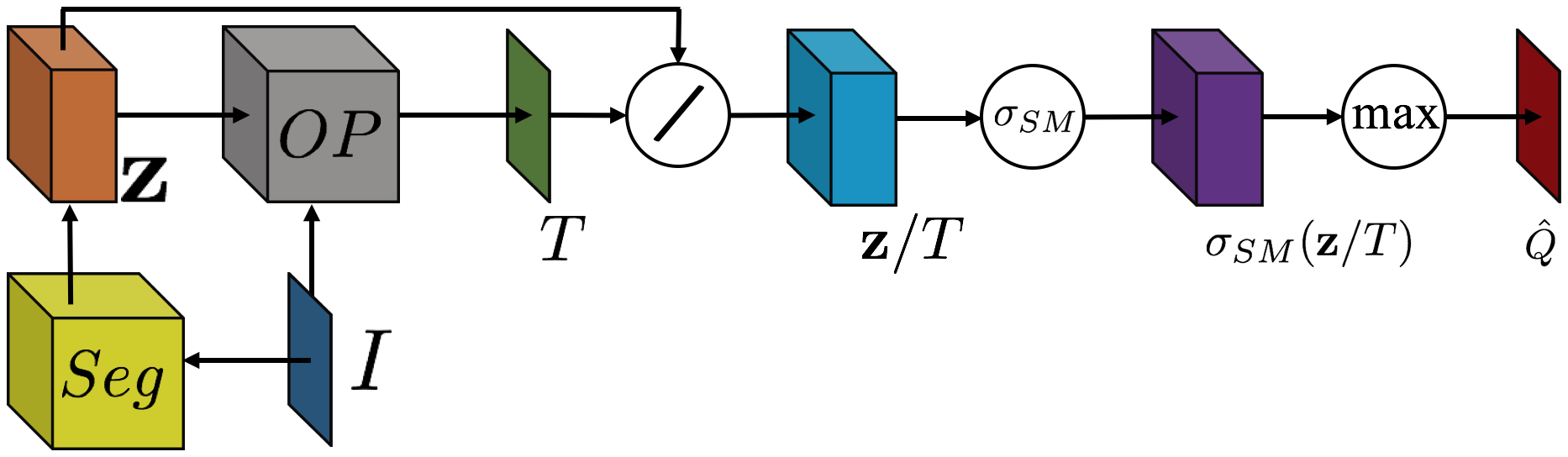}
  \caption{Architecture for probability calibration via (local) temperature scaling. The output logit map of a pre-trained semantic segmentation network (Seg) is locally scaled to produces the calibrated probabilities. $OP$ denotes optimization or prediction via a deep convolutional network to obtain the (local) temperature values. Details of this $OP$ unit can be found in {\supp}~\ref{app:LTS_networks}.
  \vspace{-1mm}}
  \label{fig:learning_framework}
\end{figure}

Another network $\mathscr{H}$, with parameter $\alpha$, can be used to learn this local mapping from ($\textbf{z}_i$, $I_i$) to $T_i(x)$. The optimization follows Eq.~\eqref{eq:IBTS_opt}, with $\mathscr{F}(\theta, \textbf{z}_i, I_i)$ replaced by $\mathscr{H}(\alpha, \textbf{z}_i, I_i, x)$, where $x$ indicates the spatial locations. Finally, we obtain $T_i(x)^* = \mathscr{H}(\alpha^*, \textbf{z}_i, I_i, x)$.

Fig.~\ref{fig:learning_framework} illustrates our high-level design for probability calibration. The input is a logit map $\textbf{z}$, usually obtained by a segmentation network (Seg). Together with the image $I$, it is then passed to an optimization unit or a prediction unit
to generate the temperature map. These temperature values are used to calibrate the logit map. The calibrated probabilities are, in turn, obtained via a softmax on the calibrated logits. Class labels do not change under this process and can still be obtained by determining the class with the largest predicted probability. {\supp}~\ref{app:LTS_networks} details the implementation. Training details are described in {\supp}~\ref{app:implementation}. 

\subsection{Theoretical Justification}
\label{sec:why_cross_entropy_and_temperature_scaling}
\textbf{Why does miscalibration happen?} One usually uses the loss corresponding to the negative log-likelihood (NLL) of the multinomial distribution~\cite{bengio2017deep,friedman2001elements} (i.e., the multi-class cross-entropy loss) to train semantic segmentation networks because minimizing it will minimize the Kullback-Leibler (KL) divergence between the ground-truth probability distribution and the predicted probability distribution. The minimum loss is achieved if and only if the predicted probability distribution recovers the ground-truth probability distribution~\cite{bengio2017deep,friedman2001elements}. For semantic segmentation, the NLL loss is minimized when $\hat{P}(x) = 1$ and $\hat{S}(x) = S(x)$, for all $x$. The segmentation error is minimized when $\textbf{z}(x)^{(S(x))} > \textbf{z}(x)^{(l)}$ for all $l \in L$ and $l \neq S(x)$. This indicates that even if the segmentation error is minimized to zero, the NLL loss may still be positive and the optimization will consequently try to continue reducing it to zero by pushing $\hat{P}(x)$ to one for $\hat{S}(x) = S(x)$. This explains how overconfidence occurs in the context of semantic segmentation. Note that this overconfidence also results in low-entropy distributions. 

\textbf{How to eliminate miscalibration?} As indicated in \cite{mukhoti2020calibrating} encouraging the predicted distribution to have higher entropy can help avoid overconfident predictions for deep CNNs, and can thereby improve calibration. Thus, to calibrate an overconfident semantic segmentation network, we need to simultaneously minimize the NLL loss w.r.t. the to-be-learned calibration parameters while assuring that the corresponding entropy of the calibrated probabilities stays sufficiently large to probabilistically describe empirically observable segmentation errors. Note that we minimize the NLL loss for the same reason as for segmentation (above): because the goal is to recover the true probability distribution. The difference is that for segmentation we optimize w.r.t. the segmentation network parameters while for calibration we optimize w.r.t. the calibration model parameters. 


\textbf{Why do we use (local) TS to calibrate probabilities?} Overconfident networks usually exhibit the phenomenon that the entropy of the output probabilities is much lower than the cross entropy on the testing dataset as shown in \cite{guo2017calibration,mukhoti2020calibrating}. Thus, we define overconfidence as entropy being lower than the cross entropy of probabilities ({\supp}~\ref{app:temp_entropy}; and similarly for underconfidence).
Specifically, we show the following theorem in {\supp}~\ref{app:temp_entropy}.


\noindent
\textbf{Theorem~\ref{thm:NLL_Entropy}.} \textit{When the to-be-calibrated segmentation network is overconfident, minimizing NLL w.r.t. TS, IBTS, and LTS results in solutions that are also the solutions of maximizing entropy of the calibrated probability w.r.t. TS, IBTS and LTS under the condition of overconfidence.
} 

For example, for TS, the above theorem can be mathematically expressed as follows,
\vspace{-2mm}
\begin{gather*}
    \argmin_{T} - \sum_{i=1}^n \sum_{x \in \Omega} \log \Big(\sigma_{SM}\big( \textbf{z}_i(x)/T \big)^{(S_i(x))}\Big)\\
    \Updownarrow\\
    \argmax_{T} -\sum_{i=1}^n \sum_{x \in \Omega} \sum_{l=1}^L \sigma_{SM}\big(\frac{\textbf{z}_i(x)}{T}\big)^{(l)} \log \Big(\sigma_{SM} \big(\frac{\textbf{z}_i(x)}{T}\big)^{(l)}\Big) \\
    s.t. \; 
	\sum_{i=1}^n\sum_{x \in \Omega}\sum_{l=1}^L \textbf{z}_i(x)^{(l)}\sigma_{SM}\big(\frac{\textbf{z}_i(x)}{T}\big)^{(l)} \geq \sum_{i=1}^n\sum_{x \in \Omega} \textbf{z}_i(x)^{(S_i(x))}
\end{gather*}
where $T>0$. Hence, our three different variants for probability calibration via temperature scaling (TS, IBTS, LTS)  will counteract the tendency of entropy minimization caused by the NLL loss discussed above. Training the segmentation network via the NLL loss followed by post-hoc probability calibration via temperature scaling is an effective approach to obtain high segmentation accuracy while avoiding overconfidence of the resulting label probabilities. \cref{sec:COCO_exp}-\cref{sec:VoteNet_exp} show experiments to support this claim.


\section{Experiments}
\label{sec:experiments}

We show the performance and behavior of our proposed TS approaches for semantic segmentation on the COCO dataset (\cref{sec:COCO_exp}), CamVid dataset (\cref{sec:Tiramisu_exp}) and LPBA40 dataset (a dataset of magnetic resonance (MR) images of the human brain) (\cref{sec:UNet_exp}). We further show how our probability calibration may influence downstream tasks, by exploring it in the context of multi-atlas segmentation on LPBA40 (\cref{sec:VoteNet_exp}). 

\textbf{Evaluation Metrics.} To assess the performance of probability calibration, we use five metrics, which were originally designed for classification, for semantic segmentation. Specifically, they are the reliability diagram~\cite{degroot1983comparison,murphy1977reliability,niculescu2005predicting}, expected calibration error~\cite{naeini2015obtaining} (ECE), maximum calibration error~\cite{naeini2015obtaining} (MCE), static calibration error~\cite{nixon2019measuring} (SCE), and adaptive calibration error~\cite{nixon2019measuring} (ACE). To make the above metrics applicable to semantic segmentation, we consider the predicted probabilities for each pixel/voxel as separate samples. We use 10 equally-sized (probability or sample size) bins to compute all these metrics. In \cref{sec:VoteNet_exp}, we additionally use average surface distance (ASD), surface Dice (SD), the 95-th percentile of the maximum symmetric distance (95MD), and average volume Dice (VD) to measure segmentation performance. Detailed definitions are in {\supp}~\ref{app:metrics}.

\textbf{Baseline Methods.} To illustrate the effectiveness of our proposed LTS approach (see Eq.~\eqref{eq:local_temperature_scaling}), we compare it to standard TS and IBTS (see Eq.~\eqref{eq:image_temperature_scaling}), where we directly assess if local adjustments can be properly predicted and if they are beneficial. 
While other probability calibration methods exist, as discussed in \cref{sec:related_work}, most are for classification and not for semantic segmentation. This is an important difference. For example, in semantic segmentation, nearby pixels/voxels are correlated with each other, whereas such relations do not apply to classification. Thus, simply considering each pixel/voxel as a classification data point is not appropriate. For completeness, however, we still choose several classic methods (\cref{sec:COCO_exp}) to compare against, i.e. isotonic regression (IsoReg)~\cite{zadrozny2002transforming}, vector scaling (VS)~\cite{guo2017calibration}, ensemble temperature scaling (ETS)~\cite{zhang2020mix}, and Dirichlet calibration with off-diagonal regularization (DirODIR)~\cite{kull2019beyond}. Furthermore, to illustrate that our method is also beneficial for joint training (\cref{sec:Tiramisu_exp}), we show the performance before and after using LTS for models trained with maximum mean calibration loss (MMCE)~\cite{kumar2018trainable} and focal loss (FL)~\cite{mukhoti2020calibrating}. 
All methods are fine-tuned with the best parameters via grid search. Details are in {\supp}~\ref{app:implementation}.

  \textbf{Evaluation Regions.} Since label boundaries are difficult to segment, these are the regions where most of the relevant miscalibrations are expected to occur (see also Fig.~\ref{fig:loc_diff}). For a refined analysis, we extract boundaries and their nearby regions (i.e., regions up to 2 pixels/voxels away from the boundary). We denote this evaluation region by \textit{Boundary} in all experiments. We also evaluate performance \emph{within} label regions (excluding the background, but including the respective \textit{Boundary} region). We denote this large region as \textit{All}. It is expected that the calibration inside the \textit{Boundary} region will be more challenging (as the prediction is more ambiguous) than the calibration inside the bigger \textit{All} region. {\supp}~\ref{app:region} shows examples of these regions for a 3D brain MR image. Furthermore, to evaluate the local probability calibration performance for an image segmentation, we also randomly select 10 small patches (72$\times$72 for 2D, 72$\times$72$\times$72 for 3D) and compute the same metrics as for the entire image. We report average performance (denoted \textit{Local-Avg}) and the worst case performance (denoted \textit{Local-Max}) across 10 patches. {\supp}~\ref{app:patch_size_metrcis} shows results for different patch sizes. Note that results in the \textit{All} region reflect the overall calibration performance for an image segmentation; results in the \textit{Boundary} region reflect the most challenging calibration performance for an image segmentation; results in the \textit{Local} region generally reflect whether the calibration method can handle spatial variations.

\textbf{Downstream MAS setting.} Multi-atlas segmentation (MAS) relies on transferring segmentations from a set of atlas images to a target image via deformable registration. The segmentation in the target space is then obtained by a label fusion method, which establishes a consensus among the registered atlas labels. We use the label fusion strategy by Wang et al.~\cite{wang2012multi}, which takes advantage of the label probabilities. Hence, better-calibrated probabilities should lead to better fusion accuracy (i.e., segmentation accuracy).

\textbf{Statistical Considerations.} To indicate the success of probability calibration, we use a Mann-Whitney U-test~\cite{mann1947test} to check for significant differences between the result of LTS and the results for all other baseline methods (UC, TS, IBTS, etc.). We use the Benjamini/Hochberg correction~\cite{benjamini1995controlling} for multiple comparisons with a false discovery rate of 0.05. Results are highlighted in green when LTS performs significantly better than the corresponding method (no color means no statistically significant differences).

\textbf{Datasets.} We use three datasets for our experiments: The Common Object in Context (COCO)~\cite{lin2014microsoft} dataset, the Cambridge-driving Labeled Video Database (CamVid)~\cite{brostow2008segmentation,brostow2009semantic}, and the LONI Probabilistic Brain Atlas (LPBA40)~\cite{shattuck2008construction} dataset. Detailed descriptions and the training/validation/testing splits are in {\supp}~\ref{app:implementation}.


\begin{table*}[!th] 
\centering
\begin{adjustbox}{max width=0.81\textwidth}
\begin{tabular}{cccccccccccccc}
\specialrule{.15em}{.05em}{.05em}
\multicolumn{1}{c}{\multirow{3}{*}{Dataset}} &
\multicolumn{1}{c}{\multirow{3}{*}{Method}} & \multicolumn{3}{c}{ECE(\%)$\downarrow$} & \multicolumn{3}{c}{MCE(\%)$\downarrow$} & \multicolumn{3}{c}{SCE(\%)$\downarrow$} & \multicolumn{3}{c}{ACE(\%)$\downarrow$}\\
\cmidrule(lr){3-5}
\cmidrule(lr){6-8}
\cmidrule(lr){9-11}
\cmidrule(lr){12-14}
\multicolumn{1}{r}{}& \multicolumn{1}{r}{}& \multicolumn{1}{c}{\multirow{2}{*}{\textit{All}}} & \multicolumn{1}{c}{\multirow{2}{*}{\textit{Boundary}}} & \textit{Local-Avg} & \multicolumn{1}{c}{\multirow{2}{*}{\textit{All}}} & \multicolumn{1}{c}{\multirow{2}{*}{\textit{Boundary}}} & \textit{Local-Avg} & \multicolumn{1}{c}{\multirow{2}{*}{\textit{All}}} & \multicolumn{1}{c}{\multirow{2}{*}{\textit{Boundary}}} & \textit{Local-Avg} & \multicolumn{1}{c}{\multirow{2}{*}{\textit{All}}} & \multicolumn{1}{c}{\multirow{2}{*}{\textit{Boundary}}} & \textit{Local-Avg} \\
\multicolumn{1}{r}{} & \multicolumn{1}{r}{} & \multicolumn{1}{r}{} & \multicolumn{1}{r}{} & [\textit{Local-Max}] & \multicolumn{1}{r}{} & \multicolumn{1}{r}{} & [\textit{Local-Max}] & \multicolumn{1}{r}{} & \multicolumn{1}{r}{} & [\textit{Local-Max}] & \multicolumn{1}{r}{} & \multicolumn{1}{r}{} & [\textit{Local-Max}] \\
\specialrule{.15em}{.05em}{.05em}
\multicolumn{1}{c}{\multirow{16}{*}{\shortstack{FCN\\COCO\\(1000)}}}& \multicolumn{1}{c}{\multirow{2}{*}{UC}}&\multicolumn{1}{c}{\multirow{2}{*}{12.44(17.87)}}&\multicolumn{1}{c}{\multirow{2}{*}{{\setlength{\fboxsep}{0pt}\colorbox{green!30}{24.41(7.23)}}}}&14.48(20.89) &\multicolumn{1}{c}{\multirow{2}{*}{27.66(22.23)}}&\multicolumn{1}{c}{\multirow{2}{*}{{\setlength{\fboxsep}{0pt}\colorbox{green!30}{38.61(7.22)}}}}&{\setlength{\fboxsep}{0pt}\colorbox{green!30}{34.90(23.89)}} &\multicolumn{1}{c}{\multirow{2}{*}{20.24(18.75)}}&\multicolumn{1}{c}{\multirow{2}{*}{{\setlength{\fboxsep}{0pt}\colorbox{green!30}{24.97(7.07)}}}}&{\setlength{\fboxsep}{0pt}\colorbox{green!30}{20.05(21.67)}} &\multicolumn{1}{c}{\multirow{2}{*}{20.19(18.73)}}&\multicolumn{1}{c}{\multirow{2}{*}{{\setlength{\fboxsep}{0pt}\colorbox{green!30}{24.46(7.26)}}}}&{\setlength{\fboxsep}{0pt}\colorbox{green!30}{19.86(21.68)}}\\
\multicolumn{1}{r}{}& \multicolumn{1}{r}{}& \multicolumn{1}{r}{} & \multicolumn{1}{r}{} & {\setlength{\fboxsep}{0pt}\colorbox{green!30}{[33.14(26.83)]}} & \multicolumn{1}{r}{} & \multicolumn{1}{r}{} & {\setlength{\fboxsep}{0pt}\colorbox{green!30}{[58.73(19.66)]}} & \multicolumn{1}{r}{} & \multicolumn{1}{r}{} & {\setlength{\fboxsep}{0pt}\colorbox{green!30}{[39.66(24.30)]}} & \multicolumn{1}{r}{} & \multicolumn{1}{r}{} & {\setlength{\fboxsep}{0pt}\colorbox{green!30}{[39.16(24.62)]}} \\
\cmidrule(lr){2-14}
&\multicolumn{1}{c}{\multirow{2}{*}{IsoReg~\cite{zadrozny2002transforming}}}&\multicolumn{1}{c}{\multirow{2}{*}{{\setlength{\fboxsep}{0pt}\colorbox{green!30}{12.55(14.22)}}}}&\multicolumn{1}{c}{\multirow{2}{*}{{\setlength{\fboxsep}{0pt}\colorbox{green!30}{16.27(6.62)}}}}&\setlength{\fboxsep}{0pt}\colorbox{green!30}{15.35(16.81)}&\multicolumn{1}{c}{\multirow{2}{*}{27.58(21.06)}}&\multicolumn{1}{c}{\multirow{2}{*}{33.36(10.01)}}&31.76(20.05)&\multicolumn{1}{c}{\multirow{2}{*}{\setlength{\fboxsep}{0pt}\colorbox{green!30}{22.28(15.35)}}}&\multicolumn{1}{c}{\multirow{2}{*}{\setlength{\fboxsep}{0pt}\colorbox{green!30}{17.20(6.42)}}}&\setlength{\fboxsep}{0pt}\colorbox{green!30}{21.65(17.77)}&\multicolumn{1}{c}{\multirow{2}{*}{\setlength{\fboxsep}{0pt}\colorbox{green!30}{22.19(15.35)}}}&\multicolumn{1}{c}{\multirow{2}{*}{\setlength{\fboxsep}{0pt}\colorbox{green!30}{16.40(6.77)}}}&\setlength{\fboxsep}{0pt}\colorbox{green!30}{21.41(17.82)}\\
\multicolumn{1}{r}{}& \multicolumn{1}{r}{}& \multicolumn{1}{r}{} & \multicolumn{1}{r}{} & \setlength{\fboxsep}{0pt}\colorbox{green!30}{[29.26(22.36)]}  & \multicolumn{1}{r}{} & \multicolumn{1}{r}{} & [43.24(23.70)]  & \multicolumn{1}{r}{} & \multicolumn{1}{r}{} & \setlength{\fboxsep}{0pt}\colorbox{green!30}{[37.13(19.38)]}  & \multicolumn{1}{r}{} & \multicolumn{1}{r}{} & \setlength{\fboxsep}{0pt}\colorbox{green!30}{[36.69(19.69)]}  \\
\cmidrule(lr){2-14}
&\multicolumn{1}{c}{\multirow{2}{*}{VS~\cite{guo2017calibration}}}&\multicolumn{1}{c}{\multirow{2}{*}{\setlength{\fboxsep}{0pt}\colorbox{green!30}{12.70(17.22)}}}&\multicolumn{1}{c}{\multirow{2}{*}{\setlength{\fboxsep}{0pt}\colorbox{green!30}{24.60(6.98)}}}&\setlength{\fboxsep}{0pt}\colorbox{green!30}{14.57(20.26)}&\multicolumn{1}{c}{\multirow{2}{*}{\setlength{\fboxsep}{0pt}\colorbox{green!30}{38.40(16.92)}}}&\multicolumn{1}{c}{\multirow{2}{*}{\setlength{\fboxsep}{0pt}\colorbox{green!30}{38.96(7.45)}}}&\setlength{\fboxsep}{0pt}\colorbox{green!30}{41.20(20.23)}&\multicolumn{1}{c}{\multirow{2}{*}{18.05(18.25)}}&\multicolumn{1}{c}{\multirow{2}{*}{\setlength{\fboxsep}{0pt}\colorbox{green!30}{25.00(6.90)}}}&18.13(21.07)&\multicolumn{1}{c}{\multirow{2}{*}{17.98(18.25)}}&\multicolumn{1}{c}{\multirow{2}{*}{\setlength{\fboxsep}{0pt}\colorbox{green!30}{24.55(7.09)}}}&17.92(21.07)\\
\multicolumn{1}{r}{}& \multicolumn{1}{r}{}& \multicolumn{1}{r}{} & \multicolumn{1}{r}{} & \setlength{\fboxsep}{0pt}\colorbox{green!30}{[29.89(17.28)]}  & \multicolumn{1}{r}{} & \multicolumn{1}{r}{} & \setlength{\fboxsep}{0pt}\colorbox{green!30}{[50.42(25.40)]}  & \multicolumn{1}{r}{} & \multicolumn{1}{r}{} & [32.31(18.43)]  & \multicolumn{1}{r}{} & \multicolumn{1}{r}{} & [32.22(18.40)]  \\
\cmidrule(lr){2-14}
&\multicolumn{1}{c}{\multirow{2}{*}{ETS~\cite{zhang2020mix}}}&\multicolumn{1}{c}{\multirow{2}{*}{\setlength{\fboxsep}{0pt}\colorbox{green!30}{12.54(14.27)}}}&\multicolumn{1}{c}{\multirow{2}{*}{\setlength{\fboxsep}{0pt}\colorbox{green!30}{15.68(6.79)}}}&\setlength{\fboxsep}{0pt}\colorbox{green!30}{15.42(16.88)}&\multicolumn{1}{c}{\multirow{2}{*}{27.36(21.01)}}&\multicolumn{1}{c}{\multirow{2}{*}{\textbf{33.27(10.09)}}}&30.92(20.34)&\multicolumn{1}{c}{\multirow{2}{*}{\setlength{\fboxsep}{0pt}\colorbox{green!30}{22.37(15.42)}}}&\multicolumn{1}{c}{\multirow{2}{*}{\setlength{\fboxsep}{0pt}\colorbox{green!30}{16.72(6.58)}}}&\setlength{\fboxsep}{0pt}\colorbox{green!30}{21.80(17.83)}&\multicolumn{1}{c}{\multirow{2}{*}{\setlength{\fboxsep}{0pt}\colorbox{green!30}{22.29(15.41)}}}&\multicolumn{1}{c}{\multirow{2}{*}{15.82(6.93)}}&\setlength{\fboxsep}{0pt}\colorbox{green!30}{21.57(17.87)}\\
\multicolumn{1}{r}{}& \multicolumn{1}{r}{}& \multicolumn{1}{r}{} & \multicolumn{1}{r}{} & \setlength{\fboxsep}{0pt}\colorbox{green!30}{[29.41(22.44)]}  & \multicolumn{1}{r}{} & \multicolumn{1}{r}{} & [42.72(24.68)]  & \multicolumn{1}{r}{} & \multicolumn{1}{r}{} & \setlength{\fboxsep}{0pt}\colorbox{green!30}{[37.33(19.41)]}  & \multicolumn{1}{r}{} & \multicolumn{1}{r}{} & \setlength{\fboxsep}{0pt}\colorbox{green!30}{[36.85(19.75)]}  \\
\cmidrule(lr){2-14}
&\multicolumn{1}{c}{\multirow{2}{*}{DirODIR~\cite{kull2019beyond}}}&\multicolumn{1}{c}{\multirow{2}{*}{11.32(12.61)}}&\multicolumn{1}{c}{\multirow{2}{*}{14.17(17.73)}}&\setlength{\fboxsep}{0pt}\colorbox{green!30}{15.09(18.99)}&\multicolumn{1}{c}{\multirow{2}{*}{26.66(18.43)}}&\multicolumn{1}{c}{\multirow{2}{*}{34.04(12.88)}}&32.54(24.79)&\multicolumn{1}{c}{\multirow{2}{*}{\setlength{\fboxsep}{0pt}\colorbox{green!30}{19.59(13.16)}}}&\multicolumn{1}{c}{\multirow{2}{*}{15.27(7.75)}}&\setlength{\fboxsep}{0pt}\colorbox{green!30}{18.55(19.44)}&\multicolumn{1}{c}{\multirow{2}{*}{\setlength{\fboxsep}{0pt}\colorbox{green!30}{19.67(13.15)}}}&\multicolumn{1}{c}{\multirow{2}{*}{15.33(7.47)}}&\setlength{\fboxsep}{0pt}\colorbox{green!30}{18.71(19.34)}\\
\multicolumn{1}{r}{}& \multicolumn{1}{r}{}& \multicolumn{1}{r}{} & \multicolumn{1}{r}{} & \setlength{\fboxsep}{0pt}\colorbox{green!30}{[26.85(23.36)]}  & \multicolumn{1}{r}{} & \multicolumn{1}{r}{} & [46.07(18.04)]  & \multicolumn{1}{r}{} & \multicolumn{1}{r}{} & \setlength{\fboxsep}{0pt}\colorbox{green!30}{[34.48(23.17)]}  & \multicolumn{1}{r}{} & \multicolumn{1}{r}{} & \setlength{\fboxsep}{0pt}\colorbox{green!30}{[34.46(23.18)]}  \\
\cmidrule(lr){2-14}
&\multicolumn{1}{c}{\multirow{2}{*}{TS~\cite{guo2017calibration}}}&\multicolumn{1}{c}{\multirow{2}{*}{{\setlength{\fboxsep}{0pt}\colorbox{green!30}{12.53(14.28)}}}}&\multicolumn{1}{c}{\multirow{2}{*}{{\setlength{\fboxsep}{0pt}\colorbox{green!30}{15.69(6.79)}}}}&{\setlength{\fboxsep}{0pt}\colorbox{green!30}{15.41(16.89)}} &\multicolumn{1}{c}{\multirow{2}{*}{27.27(20.95)}}&\multicolumn{1}{c}{\multirow{2}{*}{33.27(10.17)}}&\textbf{30.91(20.32)} &\multicolumn{1}{c}{\multirow{2}{*}{{\setlength{\fboxsep}{0pt}\colorbox{green!30}{22.36(15.42)}}}}&\multicolumn{1}{c}{\multirow{2}{*}{{\setlength{\fboxsep}{0pt}\colorbox{green!30}{16.73(6.59)}}}}&{\setlength{\fboxsep}{0pt}\colorbox{green!30}{21.78(17.85)}} &\multicolumn{1}{c}{\multirow{2}{*}{{\setlength{\fboxsep}{0pt}\colorbox{green!30}{22.28(15.42)}}}}&\multicolumn{1}{c}{\multirow{2}{*}{15.83(6.94)}}&{\setlength{\fboxsep}{0pt}\colorbox{green!30}{21.56(17.88)}}\\
\multicolumn{1}{r}{}& \multicolumn{1}{r}{}&  \multicolumn{1}{r}{} & \multicolumn{1}{r}{} & {\setlength{\fboxsep}{0pt}\colorbox{green!30}{[29.37(22.47)]}} & \multicolumn{1}{r}{} & \multicolumn{1}{r}{} & [42.71(24.66)] & \multicolumn{1}{r}{} & \multicolumn{1}{r}{} & {\setlength{\fboxsep}{0pt}\colorbox{green!30}{[37.34(19.42)]}} & \multicolumn{1}{r}{} & \multicolumn{1}{r}{} & {\setlength{\fboxsep}{0pt}\colorbox{green!30}{[36.85(19.76)]}} \\
\cmidrule(lr){2-14}
&\multicolumn{1}{c}{\multirow{2}{*}{IBTS}}&\multicolumn{1}{c}{\multirow{2}{*}{{\setlength{\fboxsep}{0pt}\colorbox{green!30}{11.92(13.83)}}}}&\multicolumn{1}{c}{\multirow{2}{*}{{\setlength{\fboxsep}{0pt}\colorbox{green!30}{16.35(7.13)}}}}&{\setlength{\fboxsep}{0pt}\colorbox{green!30}{14.80(16.63)}} &\multicolumn{1}{c}{\multirow{2}{*}{26.25(20.26)}}&\multicolumn{1}{c}{\multirow{2}{*}{33.29(9.96)}}&31.19(19.97) &\multicolumn{1}{c}{\multirow{2}{*}{{\setlength{\fboxsep}{0pt}\colorbox{green!30}{21.68(15.31)}}}}&\multicolumn{1}{c}{\multirow{2}{*}{{\setlength{\fboxsep}{0pt}\colorbox{green!30}{17.31(6.90)}}}}&{\setlength{\fboxsep}{0pt}\colorbox{green!30}{21.06(17.81)}} &\multicolumn{1}{c}{\multirow{2}{*}{{\setlength{\fboxsep}{0pt}\colorbox{green!30}{21.62(15.29)}}}}&\multicolumn{1}{c}{\multirow{2}{*}{{\setlength{\fboxsep}{0pt}\colorbox{green!30}{16.40(7.33)}}}}&{\setlength{\fboxsep}{0pt}\colorbox{green!30}{20.82(17.84)}}\\
\multicolumn{1}{r}{}& \multicolumn{1}{r}{}& \multicolumn{1}{r}{} & \multicolumn{1}{r}{} & {\setlength{\fboxsep}{0pt}\colorbox{green!30}{[28.89(21.99)]}} & \multicolumn{1}{r}{} & \multicolumn{1}{r}{} & [43.45(23.27)] & \multicolumn{1}{r}{} & \multicolumn{1}{r}{} & {\setlength{\fboxsep}{0pt}\colorbox{green!30}{[36.62(19.32)]}} &  \multicolumn{1}{r}{} & \multicolumn{1}{r}{} & {\setlength{\fboxsep}{0pt}\colorbox{green!30}{[36.09(19.63)]}}  \\
\cmidrule(lr){2-14}
&\multicolumn{1}{c}{\multirow{2}{*}{\textbf{LTS}}}&\multicolumn{1}{c}{\multirow{2}{*}{\textbf{10.04(11.54)}}}&\multicolumn{1}{c}{\multirow{2}{*}{\textbf{13.44(6.23)}}}&\textbf{12.26(14.74)} &\multicolumn{1}{c}{\multirow{2}{*}{\textbf{26.17(15.67)}}}&\multicolumn{1}{c}{\multirow{2}{*}{35.18(12.31)}}&31.66(17.66) &\multicolumn{1}{c}{\multirow{2}{*}{\textbf{16.92(13.89)}}}&\multicolumn{1}{c}{\multirow{2}{*}{\textbf{14.53(6.18)}}}&\textbf{16.78(16.38)} &\multicolumn{1}{c}{\multirow{2}{*}{\textbf{16.91(13.93)}}}&\multicolumn{1}{c}{\multirow{2}{*}{\textbf{15.16(5.92)}}}&\textbf{16.85(16.45)}\\
\multicolumn{1}{r}{}& \multicolumn{1}{r}{}& \multicolumn{1}{r}{} & \multicolumn{1}{r}{} & [\textbf{24.31(18.63)}]  & \multicolumn{1}{r}{} & \multicolumn{1}{r}{} & [\textbf{40.13(20.39)}]  & \multicolumn{1}{r}{} & \multicolumn{1}{r}{} & [\textbf{30.05(17.45)}]  & \multicolumn{1}{r}{} & \multicolumn{1}{r}{} & [\textbf{30.21(17.60)}]  \\
\hline
\hline
\multicolumn{1}{c}{\multirow{16}{*}{\shortstack{Tiramisu\\CamVid\\(233)}}}&\multicolumn{1}{c}{\multirow{2}{*}{UC}}&\multicolumn{1}{c}{\multirow{2}{*}{{\setlength{\fboxsep}{0pt}\colorbox{green!30}{7.79(4.94)}}}}&\multicolumn{1}{c}{\multirow{2}{*}{{\setlength{\fboxsep}{0pt}\colorbox{green!30}{22.79(5.76)}}}}&{\setlength{\fboxsep}{0pt}\colorbox{green!30}{9.23(10.63)}}&\multicolumn{1}{c}{\multirow{2}{*}{{\setlength{\fboxsep}{0pt}\colorbox{green!30}{22.64(12.72)}}}}&\multicolumn{1}{c}{\multirow{2}{*}{{\setlength{\fboxsep}{0pt}\colorbox{green!30}{30.42(10.65)}}}}&{\setlength{\fboxsep}{0pt}\colorbox{green!30}{30.33(16.63)}}&\multicolumn{1}{c}{\multirow{2}{*}{{\setlength{\fboxsep}{0pt}\colorbox{green!30}{9.91(5.02)}}}}&\multicolumn{1}{c}{\multirow{2}{*}{{\setlength{\fboxsep}{0pt}\colorbox{green!30}{24.62(5.69)}}}}&{\setlength{\fboxsep}{0pt}\colorbox{green!30}{13.16(11.72)}}&\multicolumn{1}{c}{\multirow{2}{*}{{\setlength{\fboxsep}{0pt}\colorbox{green!30}{9.90(5.01)}}}}&\multicolumn{1}{c}{\multirow{2}{*}{{\setlength{\fboxsep}{0pt}\colorbox{green!30}{24.43(5.75)}}}}&{\setlength{\fboxsep}{0pt}\colorbox{green!30}{13.15(11.73)}}\\
\multicolumn{1}{r}{}& \multicolumn{1}{r}{}& \multicolumn{1}{r}{} & \multicolumn{1}{r}{} & {\setlength{\fboxsep}{0pt}\colorbox{green!30}{[25.35(12.80)]}}  & \multicolumn{1}{r}{} & \multicolumn{1}{r}{} & {\setlength{\fboxsep}{0pt}\colorbox{green!30}{[56.15(14.61)]}}  & \multicolumn{1}{r}{} & \multicolumn{1}{r}{} & {\setlength{\fboxsep}{0pt}\colorbox{green!30}{[30.60(12.48)]}}  & \multicolumn{1}{r}{} & \multicolumn{1}{r}{} & {\setlength{\fboxsep}{0pt}\colorbox{green!30}{[30.60(12.46)]}} \\
\cmidrule(lr){2-14}
&\multicolumn{1}{c}{\multirow{2}{*}{TS~\cite{guo2017calibration}}}&\multicolumn{1}{c}{\multirow{2}{*}{3.45(3.52)}}&\multicolumn{1}{c}{\multirow{2}{*}{{\setlength{\fboxsep}{0pt}\colorbox{green!30}{12.66(5.43)}}}}&{\setlength{\fboxsep}{0pt}\colorbox{green!30}{7.31(7.72)}}&\multicolumn{1}{c}{\multirow{2}{*}{{\setlength{\fboxsep}{0pt}\colorbox{green!30}{16.02(11.09)}}}}&\multicolumn{1}{c}{\multirow{2}{*}{\setlength{\fboxsep}{0pt}\colorbox{green!30}{23.57(12.88)}}}&27.29(16.23)&\multicolumn{1}{c}{\multirow{2}{*}{{\setlength{\fboxsep}{0pt}\colorbox{green!30}{9.42(3.90)}}}}&\multicolumn{1}{c}{\multirow{2}{*}{\setlength{\fboxsep}{0pt}\colorbox{green!30}{17.85(4.55)}}}&{\setlength{\fboxsep}{0pt}\colorbox{green!30}{13.50(10.14)}}&\multicolumn{1}{c}{\multirow{2}{*}{{\setlength{\fboxsep}{0pt}\colorbox{green!30}{9.44(3.92)}}}}&\multicolumn{1}{c}{\multirow{2}{*}{\setlength{\fboxsep}{0pt}\colorbox{green!30}{17.61(4.59)}}}&{\setlength{\fboxsep}{0pt}\colorbox{green!30}{13.50(10.17)}}\\
\multicolumn{1}{r}{}& \multicolumn{1}{r}{}& \multicolumn{1}{r}{} & \multicolumn{1}{r}{} & {\setlength{\fboxsep}{0pt}\colorbox{green!30}{[17.69(11.91)]}}  & \multicolumn{1}{r}{} & \multicolumn{1}{r}{} & \setlength{\fboxsep}{0pt}\colorbox{green!30}{[37.25(18.98)]}  & \multicolumn{1}{r}{} & \multicolumn{1}{r}{} & \setlength{\fboxsep}{0pt}\colorbox{green!30}{[27.72(11.37)]}  & \multicolumn{1}{r}{} & \multicolumn{1}{r}{} & \setlength{\fboxsep}{0pt}\colorbox{green!30}{[27.76(11.33)]}  \\
\cmidrule(lr){2-14}
&\multicolumn{1}{c}{\multirow{2}{*}{IBTS}}&\multicolumn{1}{c}{\multirow{2}{*}{3.63(3.65)}}&\multicolumn{1}{c}{\multirow{2}{*}{{\setlength{\fboxsep}{0pt}\colorbox{green!30}{12.57(6.07)}}}}&{\setlength{\fboxsep}{0pt}\colorbox{green!30}{7.25(7.67)}}&\multicolumn{1}{c}{\multirow{2}{*}{{\setlength{\fboxsep}{0pt}\colorbox{green!30}{16.01(10.21)}}}}&\multicolumn{1}{c}{\multirow{2}{*}{\setlength{\fboxsep}{0pt}\colorbox{green!30}{23.24(13.00)}}}&27.04(15.94)&\multicolumn{1}{c}{\multirow{2}{*}{{\setlength{\fboxsep}{0pt}\colorbox{green!30}{9.47(3.89)}}}}&\multicolumn{1}{c}{\multirow{2}{*}{\setlength{\fboxsep}{0pt}\colorbox{green!30}{17.98(4.88)}}}&{\setlength{\fboxsep}{0pt}\colorbox{green!30}{13.48(10.12)}}&\multicolumn{1}{c}{\multirow{2}{*}{{\setlength{\fboxsep}{0pt}\colorbox{green!30}{9.49(3.91)}}}}&\multicolumn{1}{c}{\multirow{2}{*}{\setlength{\fboxsep}{0pt}\colorbox{green!30}{17.75(4.92)}}}&{\setlength{\fboxsep}{0pt}\colorbox{green!30}{13.48(10.16)}}\\
\multicolumn{1}{r}{}& \multicolumn{1}{r}{}& \multicolumn{1}{r}{} & \multicolumn{1}{r}{} & {\setlength{\fboxsep}{0pt}\colorbox{green!30}{[17.60(11.91)]}}  & \multicolumn{1}{r}{} & \multicolumn{1}{r}{} & \setlength{\fboxsep}{0pt}\colorbox{green!30}{[37.61(19.27)]}  & \multicolumn{1}{r}{} & \multicolumn{1}{r}{} & \setlength{\fboxsep}{0pt}\colorbox{green!30}{[27.69(11.38)]}  & \multicolumn{1}{r}{} & \multicolumn{1}{r}{} & \setlength{\fboxsep}{0pt}\colorbox{green!30}{[27.76(11.33)]}  \\
\cmidrule(lr){2-14}
&\multicolumn{1}{c}{\multirow{2}{*}{\textbf{LTS}}}&\multicolumn{1}{c}{\multirow{2}{*}{3.40(3.59)}}&\multicolumn{1}{c}{\multirow{2}{*}{11.80(5.20)}}&\textbf{6.89(7.64)}&\multicolumn{1}{c}{\multirow{2}{*}{\textbf{12.44(7.48)}}}&\multicolumn{1}{c}{\multirow{2}{*}{\setlength{\fboxsep}{0pt}\colorbox{green!30}{22.17(9.53)}}}&27.64(16.67)&\multicolumn{1}{c}{\multirow{2}{*}{\setlength{\fboxsep}{0pt}\colorbox{green!30}{8.76(4.05)}}}&\multicolumn{1}{c}{\multirow{2}{*}{\setlength{\fboxsep}{0pt}\colorbox{green!30}{17.77(4.26)}}}&\setlength{\fboxsep}{0pt}\colorbox{green!30}{12.66(10.04)}&\multicolumn{1}{c}{\multirow{2}{*}{\setlength{\fboxsep}{0pt}\colorbox{green!30}{8.73(4.03)}}}&\multicolumn{1}{c}{\multirow{2}{*}{\setlength{\fboxsep}{0pt}\colorbox{green!30}{17.32(4.32)}}}&\setlength{\fboxsep}{0pt}\colorbox{green!30}{12.61(10.07)}\\
\multicolumn{1}{r}{}& \multicolumn{1}{r}{}& \multicolumn{1}{r}{} & \multicolumn{1}{r}{} & \setlength{\fboxsep}{0pt}\colorbox{green!30}{[16.61(11.81)]}  & \multicolumn{1}{r}{} & \multicolumn{1}{r}{} & \setlength{\fboxsep}{0pt}\colorbox{green!30}{[37.92(20.47)]}  & \multicolumn{1}{r}{} & \multicolumn{1}{r}{} & \setlength{\fboxsep}{0pt}\colorbox{green!30}{[26.78(11.22)]}  & \multicolumn{1}{r}{} & \multicolumn{1}{r}{} & \setlength{\fboxsep}{0pt}\colorbox{green!30}{[26.76(11.22)]}  \\
\cmidrule(lr){2-14}
&\multicolumn{1}{c}{\multirow{2}{*}{MMCE~\cite{kumar2018trainable}}}&\multicolumn{1}{c}{\multirow{2}{*}{4.45(4.03)}}&\multicolumn{1}{c}{\multirow{2}{*}{--}}&--&\multicolumn{1}{c}{\multirow{2}{*}{18.83(10.82)}}&\multicolumn{1}{c}{\multirow{2}{*}{--}}&--&\multicolumn{1}{c}{\multirow{2}{*}{8.59(5.98)}}&\multicolumn{1}{c}{\multirow{2}{*}{--}}&--&\multicolumn{1}{c}{\multirow{2}{*}{8.50(5.00)}}&\multicolumn{1}{c}{\multirow{2}{*}{--}}&--\\
\multicolumn{1}{r}{}& \multicolumn{1}{r}{}& \multicolumn{1}{r}{} & \multicolumn{1}{r}{} & [--]  & \multicolumn{1}{r}{} & \multicolumn{1}{r}{} & [--]  & \multicolumn{1}{r}{} & \multicolumn{1}{r}{} & [--]  & \multicolumn{1}{r}{} & \multicolumn{1}{r}{} & [--]  \\
\cmidrule(lr){2-14}
&\multicolumn{1}{c}{\multirow{2}{*}{MMCE~\cite{kumar2018trainable}+\textbf{LTS}}}&\multicolumn{1}{c}{\multirow{2}{*}{4.15(3.54)}}&\multicolumn{1}{c}{\multirow{2}{*}{--}}&--&\multicolumn{1}{c}{\multirow{2}{*}{17.98(10.69)}}&\multicolumn{1}{c}{\multirow{2}{*}{--}}&--&\multicolumn{1}{c}{\multirow{2}{*}{7.28(3.80)}}&\multicolumn{1}{c}{\multirow{2}{*}{--}}&--&\multicolumn{1}{c}{\multirow{2}{*}{7.17(3.84)}}&\multicolumn{1}{c}{\multirow{2}{*}{--}}&--\\
\multicolumn{1}{r}{}& \multicolumn{1}{r}{}& \multicolumn{1}{r}{} & \multicolumn{1}{r}{} & [--]  & \multicolumn{1}{r}{} & \multicolumn{1}{r}{} & [--]  & \multicolumn{1}{r}{} & \multicolumn{1}{r}{} & [--]  & \multicolumn{1}{r}{} & \multicolumn{1}{r}{} & [--]  \\
\cmidrule(lr){2-14}
&\multicolumn{1}{c}{\multirow{2}{*}{FL~\cite{mukhoti2020calibrating}}}&\multicolumn{1}{c}{\multirow{2}{*}{3.47(3.11)}}&\multicolumn{1}{c}{\multirow{2}{*}{\textbf{8.68(5.45)}}}&\setlength{\fboxsep}{0pt}\colorbox{green!30}{9.01(7.19)}&\multicolumn{1}{c}{\multirow{2}{*}{14.77(13.28)}}&\multicolumn{1}{c}{\multirow{2}{*}{\textbf{17.62(13.53)}}}&28.37(15.86)&\multicolumn{1}{c}{\multirow{2}{*}{\setlength{\fboxsep}{0pt}\colorbox{green!30}{7.46(3.43)}}}&\multicolumn{1}{c}{\multirow{2}{*}{\textbf{14.08(4.49)}}}&\setlength{\fboxsep}{0pt}\colorbox{green!30}{14.09(9.78)}&\multicolumn{1}{c}{\multirow{2}{*}{\setlength{\fboxsep}{0pt}\colorbox{green!30}{7.43(3.45)}}}&\multicolumn{1}{c}{\multirow{2}{*}{\textbf{13.63(4.57)}}}&\setlength{\fboxsep}{0pt}\colorbox{green!30}{14.06(9.83)}\\
\multicolumn{1}{r}{}& \multicolumn{1}{r}{}& \multicolumn{1}{r}{} & \multicolumn{1}{r}{} & \setlength{\fboxsep}{0pt}\colorbox{green!30}{[13.84(11.67)]}  & \multicolumn{1}{r}{} & \multicolumn{1}{r}{} & \setlength{\fboxsep}{0pt}\colorbox{green!30}{[33.33(18.08)]}  & \multicolumn{1}{r}{} & \multicolumn{1}{r}{} & \setlength{\fboxsep}{0pt}\colorbox{green!30}{[23.60(12.11)]}  & \multicolumn{1}{r}{} & \multicolumn{1}{r}{} & \setlength{\fboxsep}{0pt}\colorbox{green!30}{[23.62(12.05)]}  \\
\cmidrule(lr){2-14}
&\multicolumn{1}{c}{\multirow{2}{*}{FL~\cite{mukhoti2020calibrating}+\textbf{LTS}}}&\multicolumn{1}{c}{\multirow{2}{*}{\textbf{3.13(3.64)}}}&\multicolumn{1}{c}{\multirow{2}{*}{11.06(5.55)}}&6.96(8.21)&\multicolumn{1}{c}{\multirow{2}{*}{14.51(11.07)}}&\multicolumn{1}{c}{\multirow{2}{*}{19.61(9.82)}}&\textbf{26.91(16.06)}&\multicolumn{1}{c}{\multirow{2}{*}{\textbf{6.78(4.05)}}}&\multicolumn{1}{c}{\multirow{2}{*}{15.28(4.76)}}&\textbf{11.85(10.69)}&\multicolumn{1}{c}{\multirow{2}{*}{\textbf{6.73(4.05)}}}&\multicolumn{1}{c}{\multirow{2}{*}{14.76(4.84)}}&\textbf{11.83(10.73)}\\
\multicolumn{1}{r}{}& \multicolumn{1}{r}{}& \multicolumn{1}{r}{} & \multicolumn{1}{r}{} & [\textbf{12.66(12.87)}]  & \multicolumn{1}{r}{} & \multicolumn{1}{r}{} & [\textbf{32.27(19.08)}]  & \multicolumn{1}{r}{} & \multicolumn{1}{r}{} & [\textbf{22.04(13.05)}]  & \multicolumn{1}{r}{} & \multicolumn{1}{r}{} & [\textbf{22.10(12.96)}]  \\
\hline
\hline
\multicolumn{1}{c}{\multirow{8}{*}{\shortstack{U-Net\\LPBA40\\(40)}}}&\multicolumn{1}{c}{\multirow{2}{*}{UC}}&\multicolumn{1}{c}{\multirow{2}{*}{{\setlength{\fboxsep}{0pt}\colorbox{green!30}{5.58(1.16)}}}}&\multicolumn{1}{c}{\multirow{2}{*}{{\setlength{\fboxsep}{0pt}\colorbox{green!30}{14.53(1.67)}}}}&{\setlength{\fboxsep}{0pt}\colorbox{green!30}{5.62(0.95)}} &\multicolumn{1}{c}{\multirow{2}{*}{{\setlength{\fboxsep}{0pt}\colorbox{green!30}{10.71(2.10)}}}}&\multicolumn{1}{c}{\multirow{2}{*}{{\setlength{\fboxsep}{0pt}\colorbox{green!30}{19.18(1.71)}}}}&{\setlength{\fboxsep}{0pt}\colorbox{green!30}{11.74(4.55)}} &\multicolumn{1}{c}{\multirow{2}{*}{{\setlength{\fboxsep}{0pt}\colorbox{green!30}{7.34(1.04)}}}}&\multicolumn{1}{c}{\multirow{2}{*}{{\setlength{\fboxsep}{0pt}\colorbox{green!30}{15.01(1.63)}}}}&{\setlength{\fboxsep}{0pt}\colorbox{green!30}{8.24(3.08)}} &\multicolumn{1}{c}{\multirow{2}{*}{{\setlength{\fboxsep}{0pt}\colorbox{green!30}{7.13(1.02)}}}}&\multicolumn{1}{c}{\multirow{2}{*}{{\setlength{\fboxsep}{0pt}\colorbox{green!30}{14.64(1.62)}}}}&{\setlength{\fboxsep}{0pt}\colorbox{green!30}{8.20(3.06)}}\\
\multicolumn{1}{r}{}& \multicolumn{1}{r}{}& \multicolumn{1}{r}{} & \multicolumn{1}{r}{} & {\setlength{\fboxsep}{0pt}\colorbox{green!30}{[10.23(2.82)]}} &  \multicolumn{1}{r}{} & \multicolumn{1}{r}{} & {\setlength{\fboxsep}{0pt}\colorbox{green!30}{[19.46(4.75)]}} & \multicolumn{1}{r}{} & \multicolumn{1}{r}{} & {\setlength{\fboxsep}{0pt}\colorbox{green!30}{[12.98(2.88)]}} &  \multicolumn{1}{r}{} & \multicolumn{1}{r}{} & {\setlength{\fboxsep}{0pt}\colorbox{green!30}{[12.93(2.83)]}} \\
\cmidrule(lr){2-14}
&\multicolumn{1}{c}{\multirow{2}{*}{TS~\cite{guo2017calibration}}}&\multicolumn{1}{c}{\multirow{2}{*}{{\setlength{\fboxsep}{0pt}\colorbox{green!30}{1.43(0.74)}}}}&\multicolumn{1}{c}{\multirow{2}{*}{{\setlength{\fboxsep}{0pt}\colorbox{green!30}{8.74(1.07)}}}}&2.24(1.93) &\multicolumn{1}{c}{\multirow{2}{*}{4.37(3.73)}}&\multicolumn{1}{c}{\multirow{2}{*}{{\setlength{\fboxsep}{0pt}\colorbox{green!30}{14.90(1.74)}}}}&{\setlength{\fboxsep}{0pt}\colorbox{green!30}{6.68(4.44)}} &\multicolumn{1}{c}{\multirow{2}{*}{6.47(0.91)}}&\multicolumn{1}{c}{\multirow{2}{*}{{\setlength{\fboxsep}{0pt}\colorbox{green!30}{10.06(1.10)}}}}&7.81(2.54) &\multicolumn{1}{c}{\multirow{2}{*}{6.30(0.90)}}&\multicolumn{1}{c}{\multirow{2}{*}{{\setlength{\fboxsep}{0pt}\colorbox{green!30}{9.46(1.06)}}}}&7.77(2.55)\\
\multicolumn{1}{r}{}& \multicolumn{1}{r}{}& \multicolumn{1}{r}{} & \multicolumn{1}{r}{} & {\setlength{\fboxsep}{0pt}\colorbox{green!30}{[5.66(2.49)]}} &  \multicolumn{1}{r}{} & \multicolumn{1}{r}{} & {\setlength{\fboxsep}{0pt}\colorbox{green!30}{[11.03(5.31)]}} & \multicolumn{1}{r}{} & \multicolumn{1}{r}{}  & [11.49(2.53)]  & \multicolumn{1}{r}{} & \multicolumn{1}{r}{} & [11.49(2.48)] \\
\cmidrule(lr){2-14}
&\multicolumn{1}{c}{\multirow{2}{*}{IBTS}}&\multicolumn{1}{c}{\multirow{2}{*}{{\setlength{\fboxsep}{0pt}\colorbox{green!30}{1.47(0.77)}}}}&\multicolumn{1}{c}{\multirow{2}{*}{{\setlength{\fboxsep}{0pt}\colorbox{green!30}{8.79(1.14)}}}}&{\setlength{\fboxsep}{0pt}\colorbox{green!30}{2.34(1.98)}} &\multicolumn{1}{c}{\multirow{2}{*}{4.40(3.65)}}&\multicolumn{1}{c}{\multirow{2}{*}{{\setlength{\fboxsep}{0pt}\colorbox{green!30}{14.96(1.75)}}}}&{\setlength{\fboxsep}{0pt}\colorbox{green!30}{6.79(4.36)}} &\multicolumn{1}{c}{\multirow{2}{*}{6.46(0.91)}}&\multicolumn{1}{c}{\multirow{2}{*}{{\setlength{\fboxsep}{0pt}\colorbox{green!30}{10.10(1.17)}}}}&7.80(2.55) &\multicolumn{1}{c}{\multirow{2}{*}{6.29(0.90)}}&\multicolumn{1}{c}{\multirow{2}{*}{{\setlength{\fboxsep}{0pt}\colorbox{green!30}{9.50(1.13)}}}}&7.76(2.56)\\
\multicolumn{1}{r}{}& \multicolumn{1}{r}{}& \multicolumn{1}{r}{} & \multicolumn{1}{r}{} & {\setlength{\fboxsep}{0pt}\colorbox{green!30}{[5.81(2.46)]}} & \multicolumn{1}{r}{} & \multicolumn{1}{r}{} & {\setlength{\fboxsep}{0pt}\colorbox{green!30}{[10.84(4.60)]}} & \multicolumn{1}{r}{} & \multicolumn{1}{r}{} & [11.51(2.54)] & \multicolumn{1}{r}{} & \multicolumn{1}{r}{} & [11.51(2.49)] \\
\cmidrule(lr){2-14}
&\multicolumn{1}{c}{\multirow{2}{*}{\textbf{LTS}}}&\multicolumn{1}{c}{\multirow{2}{*}{\textbf{0.90(0.51)}}}&\multicolumn{1}{c}{\multirow{2}{*}{\textbf{7.00(1.23)}}}&\textbf{1.90(1.38)} &\multicolumn{1}{c}{\multirow{2}{*}{\textbf{3.51(3.42)}}}&\multicolumn{1}{c}{\multirow{2}{*}{\textbf{12.33(1.96)}}}&\textbf{5.80(3.68)} &\multicolumn{1}{c}{\multirow{2}{*}{\textbf{6.27(0.93)}}}&\multicolumn{1}{c}{\multirow{2}{*}{\textbf{8.53(1.04)}}}&\textbf{7.60(2.49)} &\multicolumn{1}{c}{\multirow{2}{*}{\textbf{6.09(0.92)}}}&\multicolumn{1}{c}{\multirow{2}{*}{\textbf{7.93(1.08)}}}&\textbf{7.56(2.49)}\\
\multicolumn{1}{r}{}& \multicolumn{1}{r}{}& \multicolumn{1}{r}{} & \multicolumn{1}{r}{} & [\textbf{3.70(2.45)}] & \multicolumn{1}{r}{} & \multicolumn{1}{r}{} & [\textbf{9.29(4.73)}] & \multicolumn{1}{r}{} & \multicolumn{1}{r}{} & [\textbf{10.89(2.61)}] & \multicolumn{1}{r}{} & \multicolumn{1}{r}{} & [\textbf{10.87(2.58)}] \\
\hline
\hline
\multicolumn{1}{c}{\multirow{8}{*}{\shortstack{VoteNet+\\LPBA40\\(640)}}}&\multicolumn{1}{c}{\multirow{2}{*}{UC}}&\multicolumn{1}{c}{\multirow{2}{*}{{\setlength{\fboxsep}{0pt}\colorbox{green!30}{7.26(0.60)}}}}&\multicolumn{1}{c}{\multirow{2}{*}{{\setlength{\fboxsep}{0pt}\colorbox{green!30}{12.78(0.75)}}}}&{\setlength{\fboxsep}{0pt}\colorbox{green!30}{7.25(2.73)}}&\multicolumn{1}{c}{\multirow{2}{*}{{\setlength{\fboxsep}{0pt}\colorbox{green!30}{12.65(0.76)}}}}&\multicolumn{1}{c}{\multirow{2}{*}{{\setlength{\fboxsep}{0pt}\colorbox{green!30}{19.99(1.10)}}}}&{\setlength{\fboxsep}{0pt}\colorbox{green!30}{12.67(3.14)}}&\multicolumn{1}{c}{\multirow{2}{*}{{\setlength{\fboxsep}{0pt}\colorbox{green!30}{7.29(0.59)}}}}&\multicolumn{1}{c}{\multirow{2}{*}{{\setlength{\fboxsep}{0pt}\colorbox{green!30}{12.79(0.75)}}}}&{\setlength{\fboxsep}{0pt}\colorbox{green!30}{7.35(2.67)}}&\multicolumn{1}{c}{\multirow{2}{*}{{\setlength{\fboxsep}{0pt}\colorbox{green!30}{2.30(0.39)}}}}&\multicolumn{1}{c}{\multirow{2}{*}{{\setlength{\fboxsep}{0pt}\colorbox{green!30}{3.52(0.55)}}}}&{\setlength{\fboxsep}{0pt}\colorbox{green!30}{6.25(2.87)}}\\
\multicolumn{1}{r}{}& \multicolumn{1}{r}{}& \multicolumn{1}{r}{} & \multicolumn{1}{r}{} & {\setlength{\fboxsep}{0pt}\colorbox{green!30}{[11.16(1.77)]}} &  \multicolumn{1}{r}{} & \multicolumn{1}{r}{} & {\setlength{\fboxsep}{0pt}\colorbox{green!30}{[16.72(1.63)]}}  & \multicolumn{1}{r}{} & \multicolumn{1}{r}{} & {\setlength{\fboxsep}{0pt}\colorbox{green!30}{[11.22(1.78)]}} &  \multicolumn{1}{r}{} & \multicolumn{1}{r}{} & {\setlength{\fboxsep}{0pt}\colorbox{green!30}{[10.23(1.58)]}} \\
\cmidrule(lr){2-14}
&\multicolumn{1}{c}{\multirow{2}{*}{TS~\cite{guo2017calibration}}}&\multicolumn{1}{c}{\multirow{2}{*}{{\setlength{\fboxsep}{0pt}\colorbox{green!30}{5.07(0.59)}}}}&\multicolumn{1}{c}{\multirow{2}{*}{{\setlength{\fboxsep}{0pt}\colorbox{green!30}{9.48(0.77)}}}}&{\setlength{\fboxsep}{0pt}\colorbox{green!30}{5.08(2.48)}}&\multicolumn{1}{c}{\multirow{2}{*}{{\setlength{\fboxsep}{0pt}\colorbox{green!30}{8.44(0.84)}}}}&\multicolumn{1}{c}{\multirow{2}{*}{{\setlength{\fboxsep}{0pt}\colorbox{green!30}{18.69(1.27)}}}}&{\setlength{\fboxsep}{0pt}\colorbox{green!30}{8.54(3.39)}}&\multicolumn{1}{c}{\multirow{2}{*}{{\setlength{\fboxsep}{0pt}\colorbox{green!30}{5.11(0.58)}}}}&\multicolumn{1}{c}{\multirow{2}{*}{{\setlength{\fboxsep}{0pt}\colorbox{green!30}{9.69(0.80)}}}}&{\setlength{\fboxsep}{0pt}\colorbox{green!30}{5.29(2.39)}}&\multicolumn{1}{c}{\multirow{2}{*}{{\setlength{\fboxsep}{0pt}\colorbox{green!30}{2.12(0.37)}}}}&\multicolumn{1}{c}{\multirow{2}{*}{{\setlength{\fboxsep}{0pt}\colorbox{green!30}{3.38(0.52)}}}}&{\setlength{\fboxsep}{0pt}\colorbox{green!30}{4.62(2.44)}}\\
\multicolumn{1}{r}{}& \multicolumn{1}{r}{}& \multicolumn{1}{r}{} & \multicolumn{1}{r}{} & {\setlength{\fboxsep}{0pt}\colorbox{green!30}{[8.77(1.74)]}} &  \multicolumn{1}{r}{} & \multicolumn{1}{r}{} & {\setlength{\fboxsep}{0pt}\colorbox{green!30}{[13.14(2.08)]}} & \multicolumn{1}{r}{} & \multicolumn{1}{r}{} & {\setlength{\fboxsep}{0pt}\colorbox{green!30}{[8.90(1.78)]}} & \multicolumn{1}{r}{} & \multicolumn{1}{r}{} & {\setlength{\fboxsep}{0pt}\colorbox{green!30}{[8.21(1.59)]}} \\
\cmidrule(lr){2-14}
&\multicolumn{1}{c}{\multirow{2}{*}{IBTS}}&\multicolumn{1}{c}{\multirow{2}{*}{{\setlength{\fboxsep}{0pt}\colorbox{green!30}{2.77(0.37)}}}}&\multicolumn{1}{c}{\multirow{2}{*}{\textbf{4.06(0.45)}}}&{\setlength{\fboxsep}{0pt}\colorbox{green!30}{3.14(1.09)}}&\multicolumn{1}{c}{\multirow{2}{*}{{\setlength{\fboxsep}{0pt}\colorbox{green!30}{5.57(0.97)}}}}&\multicolumn{1}{c}{\multirow{2}{*}{{\setlength{\fboxsep}{0pt}\colorbox{green!30}{16.90(2.20)}}}}&{\setlength{\fboxsep}{0pt}\colorbox{green!30}{6.57(2.99)}}&\multicolumn{1}{c}{\multirow{2}{*}{{\setlength{\fboxsep}{0pt}\colorbox{green!30}{3.28(0.39)}}}}&\multicolumn{1}{c}{\multirow{2}{*}{\textbf{4.27(0.55)}}}&{\setlength{\fboxsep}{0pt}\colorbox{green!30}{3.96(1.26)}}&\multicolumn{1}{c}{\multirow{2}{*}{{\setlength{\fboxsep}{0pt}\colorbox{green!30}{0.69(0.26)}}}}&\multicolumn{1}{c}{\multirow{2}{*}{{\setlength{\fboxsep}{0pt}\colorbox{green!30}{2.30(0.40)}}}}&{\setlength{\fboxsep}{0pt}\colorbox{green!30}{3.15(1.06)}}\\
\multicolumn{1}{r}{}& \multicolumn{1}{r}{}& \multicolumn{1}{r}{} & \multicolumn{1}{r}{} & {\setlength{\fboxsep}{0pt}\colorbox{green!30}{[3.21(1.13)]}} & \multicolumn{1}{r}{} & \multicolumn{1}{r}{} & {\setlength{\fboxsep}{0pt}\colorbox{green!30}{[5.26(2.81)]}} & \multicolumn{1}{r}{} & \multicolumn{1}{r}{} & {\setlength{\fboxsep}{0pt}\colorbox{green!30}{[4.27(1.62)]}} & \multicolumn{1}{r}{} & \multicolumn{1}{r}{} & {\setlength{\fboxsep}{0pt}\colorbox{green!30}{[3.63(1.12)]}} \\
\cmidrule(lr){2-14}
&\multicolumn{1}{c}{\multirow{2}{*}{\textbf{LTS}}}&\multicolumn{1}{c}{\multirow{2}{*}{\textbf{0.71(0.33)}}}&\multicolumn{1}{c}{\multirow{2}{*}{4.18(0.73)}}&\textbf{1.64(0.94)}&\multicolumn{1}{c}{\multirow{2}{*}{\textbf{1.46(0.67)}}}&\multicolumn{1}{c}{\multirow{2}{*}{\textbf{11.55(1.68)}}}&\textbf{3.54(2.02)}&\multicolumn{1}{c}{\multirow{2}{*}{\textbf{1.24(0.49)}}}&\multicolumn{1}{c}{\multirow{2}{*}{4.87(0.83)}}&\textbf{2.52(1.26)}&\multicolumn{1}{c}{\multirow{2}{*}{\textbf{0.30(0.24)}}}&\multicolumn{1}{c}{\multirow{2}{*}{\textbf{2.14(0.43)}}}&\textbf{1.90(1.00)}\\
\multicolumn{1}{r}{}& \multicolumn{1}{r}{}& \multicolumn{1}{r}{} & \multicolumn{1}{r}{} & [\textbf{2.43(1.64)}] & \multicolumn{1}{r}{} & \multicolumn{1}{r}{} & [\textbf{4.52(3.26)}] & \multicolumn{1}{r}{} & \multicolumn{1}{r}{} & [\textbf{3.45(1.94)}] & \multicolumn{1}{r}{} & \multicolumn{1}{r}{} & [\textbf{2.69(1.35)}]\\
\specialrule{.15em}{.05em}{.05em}
\end{tabular} 
\end{adjustbox}
\caption{Calibration results for 4 different segmentation models on 4 different tasks. Results are reported in mean(std) format. The number of testing samples are listed in parentheses underneath each dataset name. UC denotes the uncalibrated result. $\downarrow$ denotes that lower is better. Best results are bolded and green indicates statistically significant differences w.r.t. LTS (FL+LTS for CamVid). Note that due to GPU memory limits, results of MMCE and MMCE+LTS are for downsampled images, thus can not be directly compared with other methods. The goal of including them is to show that LTS can improve MMCE. LTS generally achieves the best performance on almost all metrics in the \textit{All} region, \textit{Boundary} region and \textit{Local} region. Additional results are in {\supp}~\ref{app:additional_results}.
\vspace{-2mm}
}
\label{tab:All_metrics}
\end{table*}

\subsection{FCN semantic segmentation on COCO}
\label{sec:COCO_exp}


\textbf{General:} We use a Fully-Convolutional Network (FCN)~\cite{long2015fully} with a ResNet-101~\cite{he2016deep} backbone for semantic segmentation on the COCO dataset. Tab.~\ref{tab:All_metrics} shows our quantitative evaluation results for calibrating such a segmentation model. In the \textit{All} region, TS and IBTS do not improve calibration performance, possibly because the natural images in the COCO dataset are complex and vary significantly in type and shape, yet TS uses a global temperature value for all images. IBTS performs slightly better than TS on average because it uses an image-dependent temperature scaling to capture image variations, though it cannot explain the spatial image variations in the \textit{All} region. Furthermore, we observe that LTS is in general significantly better than classical methods, i.e. IsoReg~\cite{zadrozny2002transforming}, VS~\cite{guo2017calibration}, ETS~\cite{zhang2020mix} and DirODIR~\cite{kull2019beyond}. This is likely because these classical methods treat each pixel/voxel independently without considering their spatial correlations in semantic segmentation.

\textbf{Boundary:} The relatively low segmentation performance of the segmentation network suggests that such spatial variations might matter. Specifically, semantic segmentation results in a mean IOU of 63.7\%, indicating how challenging this dataset is. 
Further, all methods except VS~\cite{guo2017calibration} show significant improvements 
in the \textit{Boundary} region. This indicates that (1) these boundary regions share common miscalibration patterns, which can be captured by most methods, and (2) miscalibration effects are indeed, as expected, more pronounced in these boundary regions. 

\textbf{Local:} Different from the \textit{All} region, the \textit{Local} region is based on randomly extracted small patches of an image. Specifically, \textit{Local-Avg} reflects the average performance of local probability calibration while \textit{Local-Max} reflects the calibration performance in the most uncalibrated patch region thus measuring the worst-case calibration result. Results in ECE, SCE and ACE all suggests that LTS can calibrate the entire image region as well as local image regions. Other approaches result in significantly worse calibrations. 

\textbf{MCE:} Further, the MCE results illustrate that probability calibration for semantic segmentation is indeed very challenging compared with classification. This is because classification annotation is typically very accurate while per-pixel/voxel annotation of semantic segmentation can be difficult, especially at object boundaries. For example, in the extreme case, if one pixel/voxel is annotated wrong but predicted correct (or vice versa), then the accuracy is 0 while the prediction confidence is nearly 100\%. This will result in MCE values close to 100\% for bin based evaluation. 
Usually, these outliers make up only a small portion of all pixels/voxels in an image. Examples for such \textit{outliers} can be observed in Fig.~\ref{fig:local_rd} uncalibrated patch 1 and 3 at the lowest confidence point, where the percentage of samples is very small, but the accuracy-confidence difference is notable. Thus, for all experiments, we expect that MCE can be very high compared to the classification probability calibration literature. LTS can improve MCE values, but may still result in large MCE values.

\subsection{Tiramisu semantic segmentation on CamVid}
\label{sec:Tiramisu_exp}
\vspace{-2mm}

\textbf{General:} We use the Tiramisu segmentation model~\cite{jegou2017one} on the CamVid dataset. Tab.~\ref{tab:All_metrics} shows quantitative results for calibrating this segmentation model. Compared with the results for the COCO dataset, all four metrics are reduced greatly. This is mainly because the images in CamVid only contain 11 class street scenes and the images are relatively consistent for such scenes. Instead, images from the COCO dataset show different objects in different images. See {\supp}~\ref{app:dataset_variation} for details. Results are consistent with the COCO dataset. Specifically, (1) LTS can calibrate both the \emph{All} region probabilities as well as the local regions inside an image; (2) LTS is, in general, significantly better than TS and IBTS for most comparisons.

\vspace{-0.5mm}
\textbf{Joint Prediction and Calibration:} Further, we show that our approach is beneficial for methods that jointly optimize prediction and calibration~\cite{kumar2018trainable,mukhoti2020calibrating}. MMCE~\cite{kumar2018trainable} and FL~\cite{mukhoti2020calibrating} both consider
miscalibration when training semantic segmentation networks. 
Tab.~\ref{tab:All_metrics} shows that compared to the uncalibrated results, both MMCE and FL work significantly better. Furthermore, with LTS as a post-hoc calibration, calibration performance further consistently improves (except \emph{Boundary} regions for FL).
These findings are consistent with the results in \cite{mukhoti2020calibrating} where TS is used as a post-hoc calibration method and the authors show that MMCE+TS and FL+TS work consistently better than MMCE and FL. Hence, this favors our LTS as a successful post-hoc calibration method for segmentation. 

\begin{table*}[!t] 
\centering
\begin{adjustbox}{max width=0.85\textwidth}
\begin{tabular}{ccccccccccccc}
\specialrule{.15em}{.05em}{.05em} 
\multicolumn{1}{c}{\multirow{2}{*}{Method}} & \multicolumn{1}{c}{\multirow{2}{*}{ASD (mm)$\downarrow$}} & \multicolumn{1}{c}{\multirow{2}{*}{SD (\%)$\uparrow$}} & \multicolumn{1}{c}{\multirow{2}{*}{95MD (mm)$\downarrow$}} & \multicolumn{2}{c}{\multirow{1}{*}{VD (\%)$\uparrow$}} & \multicolumn{3}{c}{\multirow{1}{*}{VC(\textit{All}) (\%)}} & \multicolumn{3}{c}{\multirow{1}{*}{VC(\textit{Boundary}) (\%)}}\\ 
\cmidrule(lr){5-6}
\cmidrule(lr){7-9}
\cmidrule(lr){10-12}
\multicolumn{1}{r}{} & \multicolumn{1}{c}{} & \multicolumn{1}{c}{} & \multicolumn{1}{c}{} & \textit{All} & \textit{Boundary}
& rate & w$\rightarrow$c $\uparrow$ & c$\rightarrow$w $\downarrow$
& rate & w$\rightarrow$c $\uparrow$ & c$\rightarrow$w $\downarrow$\\
\specialrule{.15em}{.05em}{.05em}
Best Fusion&0.04(0.01)&99.06(0.23)&0.18(0.08)&98.99(0.19)&97.29(0.45)&20.53(1.13)&94.62(0.93)&0.00(0.00)&35.85(1.06)&94.11(0.90)&0.00(0.00)\\
Best Calibration&0.27(0.04)&93.51(1.01)&1.69(0.20)&93.71(0.73)&87.70(1.09)&13.96(0.43)&98.88(0.18)&0.00(0.00)&25.93(0.46)&98.68(0.21)&0.00(0.00)\\
\hline
\hline
UC&0.99(0.07)&75.89(1.79)&3.82(0.26)&81.19(1.09)&61.01(1.13)&-&-&-&-&-&-\\
TS&0.99(0.07)&75.85(1.80)&3.83(0.27)&81.21(1.08)&61.01(1.13)&0.45(0.03)&43.20(1.33)&40.16(1.23)&0.73(0.04)&39.34(1.32)&41.37(1.24)\\
IBTS&1.00(0.07)&75.75(1.82)&3.86(0.27)&81.20(1.08)&60.87(1.13)&1.43(0.12)&41.14(1.56)&43.27(1.35)&2.35(0.17)&36.93(1.45)&45.14(1.30)\\
\textbf{LTS}&\textbf{0.98(0.07)}&\textbf{75.96(1.78)}&\textbf{3.82(0.26)}&\textbf{81.27(1.07)}&\textbf{61.15(1.13)}&1.88(0.14)&42.42(1.43)&37.53(1.04)&2.96(0.18)&40.51(1.15)&35.59(1.01)\\
\specialrule{.15em}{.05em}{.05em}
\end{tabular} 
\end{adjustbox}
\caption{MAS label fusion results based on calibrated probabilities. $\downarrow$($\uparrow$) indicates that lower(higher) values are better. mm denotes millimeter. UC denotes uncalibrated results. VC denotes voxel annotation changes between the uncalibrated approach to the corresponding method: w$\rightarrow$c is from wrong voxel annotation to correct voxel annotation; c$\rightarrow$w is from correct voxel annotation to wrong voxel annotation. Rate is calculated based on the number of changes out of the possible number of changes. (Note that many voxel annotations can not change because all atlas annotations give the same label, thus a change in probability would not change the voxel annotation.) LTS generally improves segmentations slightly. After LTS probability calibration, JLF changes more voxels than for TS and IBTS. Further, the difference between the correct conversion and the incorrect conversion is improved over TS and IBTS. This indicates that JLF can produce better segmentations with a better probability calibration and suggests that downstream tasks may in general benefit from better calibration.}
\vspace{-1mm}
\label{tab:VoteNet_dice}
\end{table*}

\subsection{U-Net segmentation on LPBA40}
\label{sec:UNet_exp}
\vspace{-1.5mm}


\textbf{General:} We use a customized 3D U-Net~\cite{cciccek20163d} for the segmentation of the LPBA40 dataset. Tab.~\ref{tab:All_metrics} shows quantitative results for calibrating this segmentation model. All three methods calibrate the probabilities relatively well in this experiment. This might be because images have been affinely registered to a common atlas space, which reduces the variations of images and may make it easier for TS, IBTS and, LTS to calibrate both in the \textit{All} region and the \textit{Boundary} region. This might also explain the performance differences between the computer vision datasets and the medical imaging dataset in Tab.~\ref{tab:All_metrics}. See {\supp}~\ref{app:dataset_variation} for details. Differences between calibration performance among TS and IBTS are relatively small. However, LTS still performs best with respect to most metrics. 

\textbf{Spatial Variation:} Furthermore, when it comes to the \textit{Local} region analysis, LTS consistently works best. Fig.~\ref{fig:local_rd} visualizes such difference via reliability diagrams. The red arrows highlight that TS, IBTS and LTS calibrate probabilties for the whole image well but only LTS consistently performs well in the \textit{Local} region. This indicates the superiority of LTS's spatially-variant probability calibration.

\subsection{Downstream MAS label fusion on LPBA40}
\label{sec:VoteNet_exp}
\vspace{-1mm}
We use a customized VoteNet+~\cite{ding2020votenet+} for multi-atlas segmentation on the LPBA40 dataset. In this approach, a network (VoteNet+) is trained to locally predict if a labeled atlas that has been registered to the target image space should be considered trustworthy or not. Label fusion (among the registered atlas images) can then make use of these probabilities to obtain the multi-atlas segmentation results. It is these VoteNet+ probabilities that we seek to calibrate.


\textbf{Calibration Metrics:} Tab.~\ref{tab:All_metrics} shows our quantitative calibration results. Different from the U-Net experiments in \cref{sec:UNet_exp}, we observe bigger differences between the calibration approaches. This might be because the VoteNet+ calibration experiment has sufficient training data (as multi-atlas segmentation performs image registrations from each atlas image to each target image) whereas the experiments in \cref{sec:UNet_exp} are much more data-starved. Besides, as the labeled atlases are registered to the target image space via a flexible non-parametric registration approach, data variance is further reduced in comparison to the affine registrations used as preprocessing in \cref{sec:UNet_exp}. Tab.~\ref{tab:All_metrics} shows that all three methods calibrate probabilities well, and that performance order is consistent with model complexity. I.e.,  LTS performs better than IBTS, and IBTS performs better than TS. These differences are statistically significant. 

\textbf{Label Fusion with Probability:} Tab.~\ref{tab:All_metrics} only demonstrates that the calibration approaches can improve the calibration of the VoteNet+ output. To obtain the multi-atlas segmentation result, we need to use label fusion. 
As the joint label fusion (JLF) approach~\cite{wang2012multi} we use for this purpose can make use of the VoteNet+ label probabilities, it is natural to ask if improved calibration results translate to improved segmentations via JLF. Tab.~\ref{tab:VoteNet_dice} shows that while differences are small, consistent improvements can indeed be observed. Hence, our proposed LTS not only shows good calibration performance on traditional metrics (i.e. ECE, MCE, SCE and ACE), but can also benefit downstream tasks that are sensitive to accurate probabilities. For comparison, we also show two theoretical upper bounds. The \textit{Best Fusion} bound, which is obtained by assigning the correct label to the segmentation result if at least one atlas provides the right label; and the \textit{Best Calibration} bound, which is obtained by assigning a probability of 1 if the prediction by VoteNet+ is correct and $1/|L|$ otherwise, followed by JLF. 
We observe that there is still a large room to improve probability calibration as the obtained results are far from the two upper bounds.

\section{Conclusion and Future Work}
\label{sec:conclusion}
\vspace{-1mm}
We introduced LTS, a general temperature scaling method that allows for spatially-varying probability calibration in the context of multi-label semantic segmentation. Experiments on the COCO, CamVid and LPBA40 datasets show that LTS indeed outperforms probability calibration approaches which cannot account for spatially-varying miscalibration. LTS not only works for standard segmentation models but can also benefit models that aim to jointly optimize prediction and calibration. In addition, using a multi-atlas brain segmentation experiment we demonstrated that downstream tasks may benefit from improved probability calibration. Specifically, we showed that segmentation results obtained via joint label fusion improved when combined with probability calibration. Future work could focus on further calibration improvements. For example, LTS could be easily extended to a bin-wise setting as in~\cite{ji2019bin} or use distributions conditioned on classes as in~\cite{mozafari2018attended}.

\noindent
\section*{Acknowledgements}

This research was supported in part by Award Number 1R01-AR072013 from the National Institute of Arthritis and Musculoskeletal and Skin Diseases and by Award Number 2R42MH118845 of the National Institute of Mental Health. It was also supported by the National Science Foundation (NSF) under award number NSF EECS-1711776. The content is solely the responsibility of the authors and does not necessarily represent the official views of the NIH or the NSF. The authors have no conflicts of interest.

\newpage

{\small
\bibliographystyle{ieee_fullname}
\bibliography{./reference/allBibShort}
}

\newpage
\appendix

\onecolumn

{\centering 
\section*{Local Temperature Scaling for Probability Calibration \\Supplementary Material}
}

\noindent
This supplementary material provides additional details for our approach. Specifically, 
\begin{enumerate}
    \item {\supp}~\ref{app:additional_related_work} briefly introduces additional related work about uncertainty quantification. This section connects with \cref{sec:related_work} in the main manuscript and provides additional comments regarding uncertainty quantification approaches in relation to our approach. 
    
    \item {\supp}~\ref{app:LTS_networks} describes the networks we use for LTS and IBTS. This section connect with \cref{sec:local_temp_scaling} and Fig.~\ref{fig:learning_framework} in the main manuscript and provides details about the tree-like convolutional neural network we use to train the IBTS and LTS models. We emphasize that the network architecture is not our contribution, it is inspired and modified from \cite{lee2017generalizing} and other network architectures could also work.
    
    \item {\supp}~\ref{app:implementation} provides dataset descriptions and implementation details. This section connects with \cref{sec:experiments}, \cref{sec:COCO_exp}, \cref{sec:Tiramisu_exp}, \cref{sec:UNet_exp}, and \cref{sec:VoteNet_exp} in the main manuscript and details (1) the dataset we use; (2) the training/validation/testing data split of segmentation and calibration; (3) the specific hyper-parameters we use to train both segmentation models and calibration models; (4) the GitHub repositories for baseline calibration methods we compare against. 
    
    \item {\supp}~\ref{app:reliability_diagram} provides additional examples for local reliability diagrams. This section connects with \cref{sec:local_temp_scaling} and Fig.~\ref{fig:local_rd} in the main manuscript to additionally show the spatially-variant feature of our LTS approach. 
    
    \item {\supp}~\ref{app:temp_entropy} discusses our temperature scaling approaches from an entropy point of view. This section connects with \cref{sec:why_cross_entropy_and_temperature_scaling} in the main manuscript to prove the theorems to support our claims. Specifically, this section discusses the relation of entropy and cross entropy and uncovers why our temperature scaling approaches (TS, IBTS, LTS) works.  
    
    \item {\supp}~\ref{app:metrics} details the evaluation metrics we use for semantic segmentation. This section connects with \cref{sec:experiments}, Fig.~\ref{fig:local_rd}, Tab.~\ref{tab:All_metrics}, and Tab.~\ref{tab:VoteNet_dice} in the main manuscript to provide formal definitions for all our evaluation measures. 
    
    \item {\supp}~\ref{app:region} illustrates the \textit{Boundary} and \textit{All} evaluation regions. This section connects with \cref{sec:experiments} and Tab.~\ref{tab:All_metrics} in the main manuscript to illustrate a visual example of the different regions we evaluate. Note that the results in the \emph{All} region reflect the overall calibration performance for an image segmentation; results in the \emph{Boundary} region reflect the most challenging calibration performance for an image segmentation.
    
    \item {\supp}~\ref{app:patch_size_metrcis} shows evaluation results for the \emph{Local} region for different patch sizes. This section connects with \cref{sec:experiments} and Tab.~\ref{tab:All_metrics} in the main manuscript to indicate how the local patch size influences the quantitative results. Note that results in the \emph{Local} region generally reflect whether the calibration method can handle spatial variations. This is different from the \emph{All} and \emph{Boundary} regions discussed in {\supp}~\ref{app:region} above.
    
    \item {\supp}~\ref{app:dataset_variation} discusses variations across the different datasets. This section connects with \cref{sec:COCO_exp}, \cref{sec:Tiramisu_exp}, \cref{sec:UNet_exp}, \cref{sec:VoteNet_exp} and Tab.~\ref{tab:All_metrics} in the main manuscript and explains the different magnitudes of the quantitative results for different datasets. Specifically, the COCO dataset shows the biggest variantions, followed by the CamVid dataset and lastly LPBA40 exhibits the smallest variations. Due to the different levels of variation of the different datasets, the reported values in COCO are larger than those in CamVid and the smallest values are observed in LPBA40.
    
    \item {\supp}~\ref{app:additional_results} contains additional evaluation results besides the results presented in Tab.~\ref{tab:All_metrics}. This section connects with \cref{sec:Tiramisu_exp} and Tab.~\ref{tab:All_metrics} in the main manuscript to further strengthen our manuscript. 
      These results are line with the conclusions we obtain in \cref{sec:experiments}, i.e. our LTS approach generally works best among different baseline methods.
    
    \item {\supp}~\ref{app:JLF_deduction} provides details on joint label fusion for multi-atlas segmentation. This section connects with \cref{sec:VoteNet_exp} and Tab.~\ref{tab:VoteNet_dice} in the main manuscript to provide details about the downstream MAS label fusion task. Specifically, this section illustrates why the VoteNet+ based joint label fusion method is sensitive to accurate probability predictions, which in turn demonstrates that improved calibration of our approach results in improved fused segmentation results.
    
\end{enumerate}

\section{Additional Related Work}
\label{app:additional_related_work}
\noindent
Probability calibration can be used for uncertainty estimation~\cite{lakshminarayanan2016simple} as calibrated probabities can directly be used as measures of uncertainty. However, methods that provide uncertainty estimates are not necessarily calibrated. 
Most existing work on uncertainty estimation starts with a Bayesian formulation~\cite{lakshminarayanan2016simple,jena2019bayesian,maronas2020calibration}, whereby a prior distribution is specified, and the posterior distribution over the parameters is optimized over the training data. These Bayesian models should result in better calibrated probability measures if their prior assumptions are valid. However, when some of the underlying assumptions are violated, the results may not be calibrated:~\cite{kendall2017uncertainties} is a good example for a Bayesian model improving calibration, but not achieving it. Other uncertainty estimation approaches include ensembles~\cite{lakshminarayanan2016simple} and  Monte Carlo dropout~\cite{gal2016dropout}, which help probability calibration but do not directly cope or achieve it. Gaussian Process (GP) approaches~\cite{wenger2020non} can inherently provide good uncertainty estimates, but may suffer from lower accuracy and higher computational complexity on high-dimensional classification tasks. In particular, a GP will only provide calibrated measures of uncertainty if the Gaussian assumption is valid. In practice, this may not be the case when combined with a deep network~\cite{tran2019calibrating}. Further, GP models are costly for classification and GP regression formulations require calibration~\cite{milios2018dirichlet,wenger2020non}. Our formulation is entirely different and directly predicts calibration parameters for softmax layers. Our model does not depend on any assumption and is a completely poct-hoc approach for any pre-trained segmentation model with probability outputs.\\

\section{Networks for LTS and IBTS}
\label{app:LTS_networks}
\noindent
To obtain $T^*$ in Eq.~\eqref{eq:TS_opt}, we directly optimize the parameter $T$ with respect to the NLL loss on the hold-out validation dataset. \\

\noindent
To obtain $T_i(x)^*$, we borrow the idea of soft decision trees~\cite{irsoy2016autoencoder} and propose to use a tree-like convolutional neural network~\cite{lee2017generalizing} to predict $T_i (x)$, which has fewer parameters than a standard convolutional neural network while achieving comparable state-of-the-art performance~\cite{lee2017generalizing}. We resort to such a simpler tree-like model, because one of the datasets that we use for evaluation is relatively small, though more complex models could be further explored.

\begin{figure}[!th]
  \includegraphics[width=\linewidth]{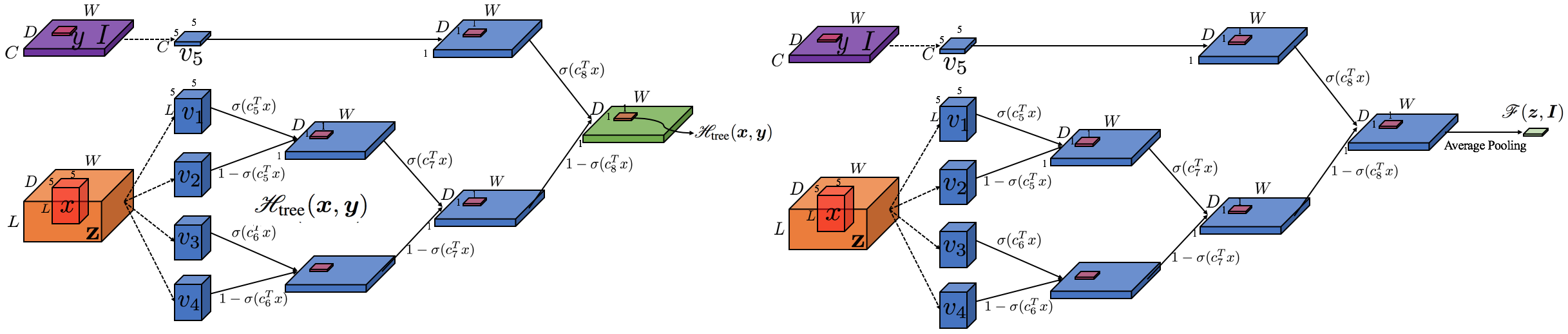}
  \caption{LTS (left) and IBTS (right) hierarchical tree-like architectures demonstrated in 2-D. $W$ is the image width, $D$ is the image length, $L$ is the number of classes, $C$ is the number of channels. $\bm{x}$ is the patch centered at location $x$ of size $L \times 5 \times 5$. Its corresponding patch inside image $I$ is denoted by $\bm{y}$, which is of size $C \times 5 \times 5$. $\sigma$ is the sigmoid function. Input to the model are the logits of size $L \times W \times D$. Output is the spatially varying temperature value of the image ($1 \times W \times D$) for LTS or an image-dependent temperature scalar value ($1\times 1 \times 1$) for IBTS. $\bm{v}_i$ and $\bm{c}_j$ are convolutional filters of size $L \times 5 \times 5$ (except $\bm{v}_5$ is of size $C \times 5 \times 5$ to be compatible with the size of image). Note that the dilation is 2 for all convolutional filters, thus resulting in a 9$\times$9 receptive field.}
  \label{fig:LTS_model}
\end{figure}

\noindent
The proposed framework is constructed as a pre-specified hierarchical binary tree in which each leaf is a convolutional filter learned during training. Denote the leaf node with index $m$ as $\bm{v}_m$, the patch in logits
$\textbf{z}$ to be convolved as $\bm{x}$ and its corresponding patch in image $\textbf{I}$ to be convolved as $\bm{y}$. Since a convolutional layer can be transformed into a fully-connected layer, which is essentially a matrix multiplication plus a bias offset, we use $\bm{v}^{T}_m\bm{x}$ to represent the convolution operation in the framework for ease of notation (omit bias offset for simplicity).
For internal nodes of the tree, each parent node value is a mixture (i.e. weighted average) of children nodes' values and the mixture parameter is also learned during training. Specifically, we use a convolution operation $\bm{c}_m$ plus a sigmoid function $\sigma$\
to determine the mixture parameter $\sigma(\bm{c}^T_m\bm{x})$. The root node is the final output. For IBTS, the output is a single temperature value for the logits, while, for LTS, the output is a temperature map which has the same size as the input logits, except that the number of feature channels is 1. Thus, the nodes of the tree can be represented as follows:

\begin{equation}
\label{eq:tree_node_definition} 
\mathscr{H}_m(\bm{x}, \bm{y}) =
  \begin{cases}
    \bm{v}^{T}_m\bm{y} + 1       & \quad \text{if leaf node in image}\\
    \bm{v}^{T}_m\bm{x} + 1       & \quad \text{if leaf node in logits}\\
    \sigma(\bm{c}^T_m\bm{x}) \mathscr{H}_{m, \text{logits}, \text{left}}(\bm{x}) + (1 - \sigma(\bm{c}^T_m\bm{x})) \mathscr{H}_{m, \text{logits}, \text{right}}(\bm{x})                         & \quad \text{if internal node in logits}\\
    \text{ReLU}\big(\sigma(\bm{c}^T_m\bm{x}) \mathscr{H}_{m, \text{logits}}(\bm{x}) + (1 - \sigma(\bm{c}^T_m\bm{x})) \mathscr{H}_{m, \text{image}}(\bm{y})\big)+\varepsilon          & \quad \text{if root node}
  \end{cases},
\end{equation}

\noindent
where ReLU is the Rectified Linear Unit, $\mathscr{H}_m(\bm{x}, \bm{y})$ is the root node value, $\mathscr{H}_{m, \text{logits}, \text{left}}(\bm{x})$ and $\mathscr{H}_{m, \text{logits}, \text{right}}(\bm{x})$ are the left child node value and right child node value for internal nodes in logits, respectively. $\mathscr{H}_{m, \text{logits}}(\bm{x})$ is the top node containing information only from the logits and $\mathscr{H}_{m, \text{image}}(\bm{y})$ is the top node containing information only from the image. $\varepsilon$ is a very small positive real number to guarantee the positivity for the output temperature value. The $+1$ value for the leaf node is for model initialization and stabilization. With this trick, the learning process is more stable and the performance is much better. If there are only leaf nodes, then the convolution filters are trying to learn the residual of the temperature scalar value with respect to the standard uncalibrated temperature value 1. Fig.~\ref{fig:LTS_model}(left) illustrates the proposed tree-like learning framework for LTS. For simplicity, let us assume the output is positive, then the specific representation becomes 
 
\begin{equation}
\begin{split}
\label{eq:proposed_tree_model}
\mathscr{H}_{\text{tree}}(\bm{x}, \bm{y}) = \; & \sigma(\bm{c}^T_8\bm{x})(\bm{v}^T_5\bm{y}+1) \\
& + (1-\sigma(\bm{c}^T_8\bm{x}))\big\{\sigma(\bm{c}^T_7\bm{x}) \big[\sigma(\bm{c}^T_5\bm{x})(\bm{v}^{T}_1\bm{x} + 1) + (1-\sigma(\bm{c}^T_5\bm{x}))(\bm{v}^{T}_2\bm{x} + 1)\big] \\
& + (1-\sigma(\bm{c}^T_7\bm{x})) \big[\sigma(\bm{c}^T_6\bm{x})(\bm{v}^{T}_3\bm{x} + 1) + (1-\sigma(\bm{c}^T_6\bm{x}))(\bm{v}^{T}_4\bm{x} + 1)\big]\big\}.
\end{split}
\end{equation}
To connect back to the definition in \cref{sec:local_temp_scaling}, $\mathscr{H}_{\text{tree}}$ is the network $\mathscr{H}$, $\bm{v}_i$ and $\bm{c}_j$ are parameters $\alpha$, $\bm{x}$ is the patch centered at location $x$ in logits $\textbf{z}$, $\bm{y}$ is the corresponding patch of image $\textbf{I}$.\\ 

\noindent
To obtain $T_i^*$, we modify the above-mentioned network $\mathscr{H}_{\text{tree}}$ to predict one temperature value $T_i$ for each image. We add an average pooling layer after $\mathscr{H}_{\text{tree}}$ to get the image-based temperature value. Specifically, using $\mathscr{F}$ to represent the network of IBTS as in Eq.~\eqref{eq:IBTS_opt}, we have $\mathscr{F} = \frac{1}{|\Omega|} \sum_{x \in \Omega} \mathscr{H}_{\text{tree}}(\bm{x}, \bm{y})$, where $\bm{x}$ is the patch centered at location $x$ in logits
$\textbf{z}$, $\bm{y}$ is the corresponding batch of $\bm{x}$ in image $\textbf{I}$, and $\Omega$ is the logits space. Fig.~\ref{fig:LTS_model}(right) illustrates the proposed tree-like learning framework for IBTS. 

\section{Dataset Description and Implementation Details}
\label{app:implementation} 
\noindent
We use the following image segmentation datasets in our experiments:
\begin{enumerate}
  \item \textbf{COCO}~\cite{lin2014microsoft}: The Common Object in Context (COCO)~\cite{lin2014microsoft} dataset is a large-scale dataset of complex images. It provides pixel-level labels for 118K training images (COCO train2017) and 5K validation images (COCO val2017). Further, the COCO-stuff~\cite{caesar2018cvpr} dataset augments COCO with dense pixel-level annotations for 80 thing classes and 91 stuff classes. For simplicity, we focus on the 20 categories that are present in the Pascal VOC~\cite{everingham2015pascal} dataset for our experiments, considering the remaining classes as background. 
  \item \textbf{CamVid}~\cite{brostow2008segmentation,brostow2009semantic}:
    The Cambridge-driving Labeled Video Database (CamVid)~\cite{brostow2008segmentation,brostow2009semantic} is a collection of videos with object class semantic labels. We use the split and image resolution as in \cite{jegou2017one}, which consists of 367 frames for training, 101 frames for validation and 233 frames for testing. Each frame has a size of 360$\times$480 and its pixels are labeled with 11 semantic classes excluding background.  
  \item \textbf{LPBA40}~\cite{shattuck2008construction}: The LONI Probabilistic Brain Atlas (LPBA40)~\cite{shattuck2008construction} dataset contains 40 T1-weighted 3D brain MR images from healthy patients. Each image has labels for 56 manually segmented structures. For preprocessing, all images are first affinely registered to the ICBM MNI152 nonlinear atlas~\cite{grabner2006symmetric} using NiftyReg~\cite{modat2014global,modat2010fast,rueckert1999nonrigid} and intensity normalized via histogram equalization.
\end{enumerate}

\noindent
For the Fully-Convolutional Network (FCN) experiment in \cref{sec:COCO_exp}, we use the COCO val2017 dataset for our calibration experiment in which the training/validation/testing images are partitioned in sets of size 3.5K/0.5K/1K, respectively. We use the PyTorch pre-trained model\footnote{\url{https://pytorch.org/docs/stable/torchvision/models.html\#semantic-segmentation}} for semantic segmentation on the COCO dataset. This is an FCN~\cite{long2015fully} with a ResNet-101~\cite{he2016deep} backbone. The pre-trained model has been trained on a subset of COCO train2017, i.e., for the 20 categories that are present in the Pascal VOC~\cite{everingham2015pascal} dataset. For details, please resort to the Pytorch official webpage (footnote) mentioned above.\\

\noindent
For the Tiramisu experiment in \cref{sec:Tiramisu_exp}, we use the hold-out validation dataset for our calibration experiment in which the training/validation images are 90/11. Finally the calibration performance is tested on the testing dataset which includes 233 images. We use the PyTorch Tiramisu\footnote{The implementation follows this GitHub repository: \url{https://github.com/bfortuner/pytorch_tiramisu}} segmentation model~\cite{jegou2017one} on the CamVid dataset. Training details are included in the GitHub repository.\\

\noindent
For the U-Net experiment in \cref{sec:UNet_exp}, we use a 2-fold cross-validation setup to cover all the 40 images in the dataset. Training/validation/testing images are partitioned as 17/3/20. This is consistent with the setting in~\cite{ding2019votenet}. We use 4-fold cross-validation for our calibration experiment to cover all 40 images. Training/validation/testing images are partitioned as 10/3/10 for each fold. The U-Net takes patches of $72 \times 72 \times 72$ of the training images, where the $40 \times 40 \times 40$ patch center is used to tile the volume. The output is the voxel-wise probability of each label at each position. Training patches are randomly cropped assuring at least 5\% correct labels in the patch volume. We use \texttt{Adam}~\cite{kingma2014adam} with 300 epochs and a multi-step learning rate. The initial learning rate is 1e-3, and then reduced by 90\% at the 150-th epoch and the 250-th epoch, respectively. Cross-entropy loss is used as the loss function. When calibrating, within each fold of the U-Net 2-fold cross validation, we perform another 2-fold cross validation. Specifically, 23 images (3 from validation and 20 from testing) are split into 10/3/10 for train/validation/test. 2-fold cross-validation will cover all 20 testing images of U-Net testing. This design results in a 4-fold cross validation experiment to cover all 40 images.\\

\noindent
For the Downstream MAS label fusion experiment in \cref{sec:VoteNet_exp}, we use 2-fold cross-validation to cover all the images. In each fold, 17 atlases are chosen. Training/validation/testing images are partitioned as 272/51/340. This is consistent with the setting in~\cite{ding2020votenet+}. We use 4-fold cross-validation for the calibration experiments to cover all images. Training/validation/testing are partitioned as 170/51/170 for each fold. Training data for VoteNet+ is acquired by deformable image registrations. Specifically, the same 17 images as for the U-Net training are chosen as atlas images. First, all 17 atlases are registered to each other, which results in $17 \times 16 = 272$ pairs of training data. Then all 17 atlases are registered to the 3 validation images for the U-Net, which results in $17 \times 3 = 51$ pairs of validation data. Finally, all 17 atlases are registered to the 20 testing images for the U-Net, which results in $17 \times 20 = 340$ pairs of testing data. The same 2-fold cross-validation strategy still applies to VoteNet+, but with the data split as 272/51/340 for train/validation/test. VoteNet+ takes patches of $72 \times 72 \times 72$ from the target image and a warped atlas image at the same position, where the $40 \times 40 \times 40$ patch center is used to tile the volume. The output is the voxel-wise probability, indicating whether the warped atlas label is equal to the target image label. We use \texttt{Adam}~\cite{kingma2014adam} with 500 epochs with a multi-step learning rate. The initial learning rate is 1e-3 and then reduced by half at the 200-th epoch, 350-th epoch, and 450-th epoch respectively. Same as for the U-Net, training patches are randomly cropped assuring at least 5\% correct labels in the patch volume. Binary cross-entropy is used as the loss function. When calibrating, within each fold of the VoteNet+ 2-fold cross validation, we perform a 2-fold cross validation. Specifically, 391 pairs (51 from validation and 340 from testing) are split into 170/51/170 for train/validation/test. 2-fold cross-validation will cover all 340 testing pairs of VoteNet+ testing. This design results in a 4-fold cross validation experiment to cover all 680 pairs.  Furthermore, we use joint label fusion (JLF)~\cite{wang2012multi} to obtain the final segmentation for each image. See {\supp}~\ref{app:JLF_deduction} for more information on MAS and label fusion, as well as experimental details.\\

\noindent
To train IBTS and LTS, we use \texttt{Adam}~\cite{kingma2014adam} with 100 epochs and a multi-step learning rate. The initial learning rate for the LPBA40 dataset is 1e-4 and is reduced to 1e-5 after 50 epochs, while for the COCO and the CamVid dataset, it is 1e-5 and is reduced to 1e-6 after 50 epochs. We use the cross-entropy loss. The loss is evaluated over the \textit{All} region to ignore the majority of the background.\\

\noindent
The FL and MMCE losses are from the GitHub repository\footnote{\url{https://github.com/torrvision/focal_calibration/tree/main/Losses}} of \cite{mukhoti2020calibrating}. Isotonic regression (IsoReg)~\cite{zadrozny2002transforming} and ensemble temperature scaling (ETS)~\cite{zhang2020mix} are from the GitHub repository\footnote{\url{https://github.com/zhang64-llnl/Mix-n-Match-Calibration}} of ~\cite{zhang2020mix}. Vector scaling (VS)~\cite{guo2017calibration} and Dirichlet calibration with off-diagonal regularization (DirODIR)~\cite{kull2019beyond} are from the GitHub repository\footnote{\url{https://github.com/dirichletcal/experiments_neurips}} of \cite{kull2019beyond}. Training with FL and MMCE follows the same recipe as training with the multi-class entropy loss except that the training loss term is changed. The GitHub implementation repository\footnote{\url{https://github.com/bfortuner/pytorch_tiramisu}} provides all details about the hyper-parameters of training of the deep Tiramisu network; we thus omit them here to avoid duplication. For DirODIR, the hyper-parameters for off-diagonal regularization and bias regularization are both set to 0.01. We use \texttt{Adam} for a maximum of 100 epochs with early stop patience set to 10 epochs, i.e. training stops early if 10 consecutively worse epochs are observed. The model is trained with an initial learning rate of 1e-3 and fine-tuned with a learning rate of 1e-4.

\begin{figure*}[!t]
  \centering
  \includegraphics[width=\textwidth]{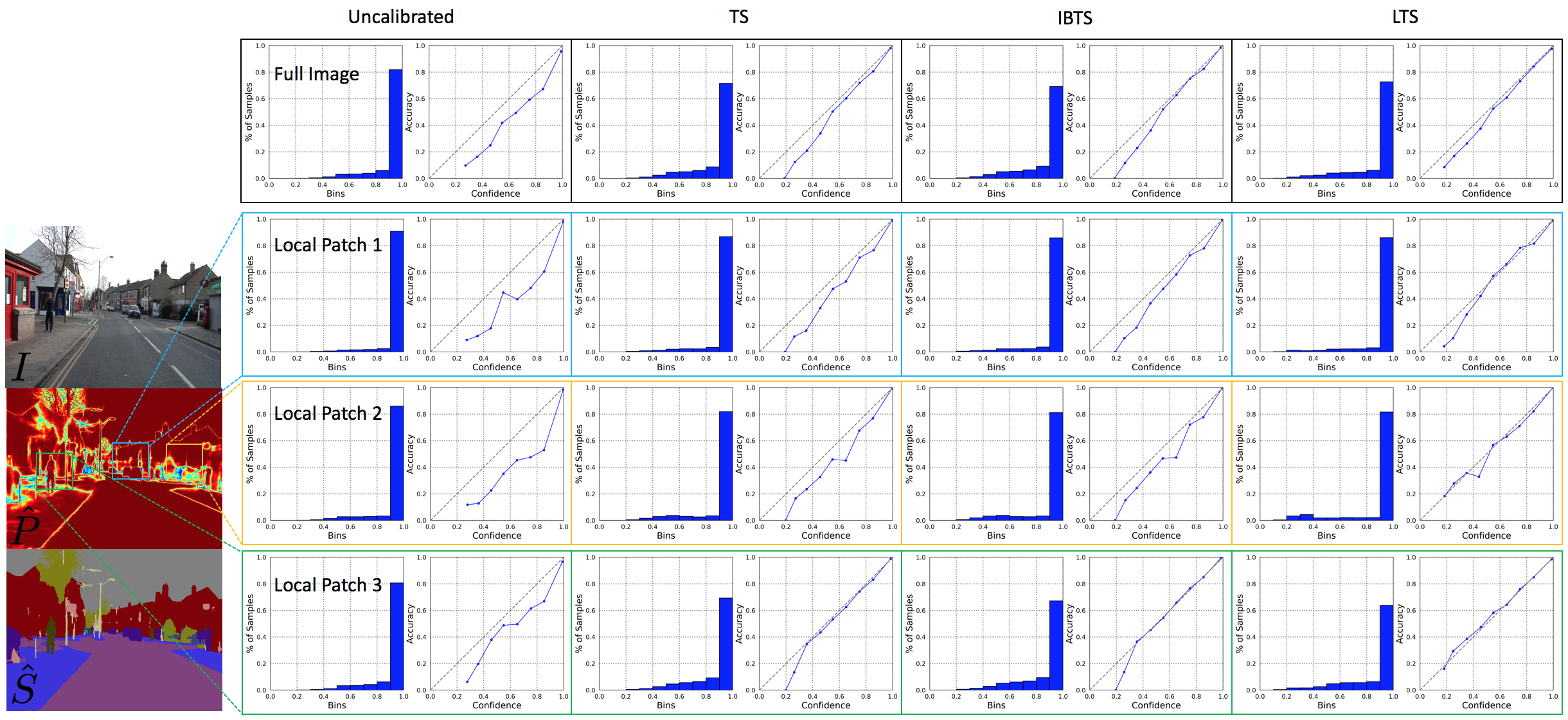}
  \caption{An example of global and local reliability diagrams for different methods for the Tiramisu semantic segmentation experiment (\cref{sec:Tiramisu_exp}). $I$ is the image, $\hat{P}$ is the predicted uncalibrated probability, and $\hat{S}$ is the predicted segmentation. Figures are displayed in couples, where the left figure is the probability distribution of pixels/voxels while the right figure is the reliability diagram (See {\supp}~\ref{app:metrics} for definitions). The top row shows the global reliability diagrams for different methods for the entire image. The three rows underneath correspond to local reliability diagrams for the different methods for different local patches. LTS not only calibrates probabilities well for the entire image but also calibrates probabilities better than TS and IBTS in local pacthes.}
  \label{fig:camvid_local_rd}
\end{figure*}

\section{Local Reliability Diagrams}
\label{app:reliability_diagram}
\noindent
To visualize the spatially-varying property of LTS, we show the local reliability diagram of Tiramisu for the CamVid experiment in Fig.~\ref{fig:camvid_local_rd}. Similar to the conclusion from Fig.~\ref{fig:local_rd}, Fig.~\ref{fig:camvid_local_rd} also suggests that LTS performs better than TS and IBTS for the entire image as well as for the local image patches. This observation is consistent with results in Tab.~\ref{tab:All_metrics}.

\section{Temperature Scaling from Entropy Point of View}
\label{app:temp_entropy}

\noindent
Temperature scaling can also be connected to entropy~\cite{guo2017calibration}. In this section, we establish the relation between entropy and temperature scaling by showing that different temperature scaling models are indeed the solutions for entropy maximization or minimization subject to different constraints. Note that a related insight has been proposed in \cite{guo2017calibration} for classification. We extend it to semantic segmentation for our different temperature scaling settings and provide detailed discussions. Specifically, we show the solutions of TS, IBTS and LTS when minimizing NLL in {\supp}~\ref{app:minimize_NLL}; we define overconfidence and underconfidence in {\supp}~\ref{app:overfitting_underfitting}; we show the entropy maximization and minimization solutions without constraints in {\supp}~\ref{app:entropy_extrmes_without_constraints}; we deduct the solutions for entropy maximization under the condition of overconfidence as well as for entropy minimization under the condition of underconfidence in {\supp}~\ref{app:entropy_entrmes_with_constraints}; finally, we show that the solutions for minimizing NLL w.r.t. TS, IBTS, LTS are also the solutions for entropy maximization in the case of overconfidence or the solutions for entropy minimization in the case of underconfidence in {\supp}~\ref{app:NLL_meet_with_entropy}. Overall, TS, IBTS and LTS determined based on a given dataset results in NLL (cross entropy) and entropy reaching an equilibrium which empirically corresponds to a well-calibration state.\\

\subsection{Minimize NLL with (Local) Temperature Scaling}
\label{app:minimize_NLL}

\noindent
\begin{lemma}
\label{lamma_bound}
Given a logit vector map $\textbf{z}(x)$ at position $x$ and its corresponding probability map obtained via softmax function ($\sigma_{SM}$) the weighted averaged logits with temperature scaling (TS) are (1) monotonic with respect to temperature value and (2) yield the following bounds:
\begin{equation}
    \frac{1}{L}\sum_{l=1}^L \textbf{z}(x)^{(l)} \leq \sum_{l=1}^L \textbf{z}(x)^{(l)}\sigma_{SM}\big(\textbf{z}(x)/T\big)^{(l)}
    \leq \max_l \{\textbf{z}(x)^{(l)}\}.
\end{equation}
\end{lemma}

\begin{proof}
Let $\lambda = \frac{1}{T}$ and denote $\mathcal{F}(\lambda) = \sum_{l=1}^L \textbf{z}(x)^{(l)}\sigma_{SM}\big(\lambda \textbf{z}(x)\big)^{(l)} 
= \sum_{l=1}^L \textbf{z}(x)^{(l)}\frac{\exp\big(\lambda \textbf{z}(x)^{(l)}\big)}{\sum_{j=1}^L\exp\big(\lambda \textbf{z}(x)^{(j)}\big)}$. Then we take the derivative with respect to $\lambda$,
\begin{equation}
    \frac{\partial \mathcal{F}(\lambda)}{\partial \lambda} = \frac{\Big( \sum_{l=1}^L (\textbf{z}(x)^{(l)})^2 \exp\big(\lambda \textbf{z}(x)^{(l)}\big) \Big) \Big( \sum_{l=1}^L\exp\big(\lambda \textbf{z}(x)^{(l)}\big) \Big) - \Big( \sum_{l=1}^L \textbf{z}(x)^{(l)} \exp\big(\lambda \textbf{z}(x)^{(l)}\big) \Big)^2}{\Big( \sum_{j=1}^L\exp\big(\lambda \textbf{z}(x)^{(j)}\big) \Big)^2}.
\end{equation}
By the Cauchy–Schwarz inequality, we have 
\begin{equation*}
    \Big( \sum_{l=1}^L (\textbf{z}(x)^{(l)})^2 \exp\big(\lambda \textbf{z}(x)^{(l)}\big) \Big) \Big( \sum_{l=1}^L\exp\big(\lambda \textbf{z}(x)^{(l)}\big) \Big) \geq \Big( \sum_{l=1}^L \textbf{z}(x)^{(l)} \exp\big(\lambda \textbf{z}(x)^{(l)}\big) \Big)^2.
\end{equation*}
Thus, $\frac{\partial \mathcal{F}(\lambda)}{\partial \lambda} \geq 0$. This indicates that the function $\mathcal{F}(\lambda)$ is monotonicly increasing with respect to $\lambda$. Since the temperature scaling value $T$ is non-negative, i.e., $T \in \mathbb{R}^+$, we have $\lambda \in \mathbb{R}^+$. Furthermore,
\begin{equation}
\begin{split}
    \lambda \rightarrow 0, \quad  & \sigma_{SM}\big(\lambda \textbf{z}(x)\big)^{(l)} = \frac{1}{L}, \quad \forall l=1, ..., L;\\
    \lambda \rightarrow +\infty, \quad & \sigma_{SM}\big(\lambda \textbf{z}(x)\big)^{(l)} = 
    \begin{cases}
            1, \quad \max_j \{\textbf{z}(x)^{(j)}\} = \textbf{z}(x)^{(l)},\\
            0, \quad \text{otherwise}.
    \end{cases}
\end{split}
\end{equation}
Therefore, we have $\frac{1}{L}\sum_{l=1}^L \textbf{z}(x)^{(l)} \leq \mathcal{F}(\lambda) \leq \max_l \{\textbf{z}(x)^{(l)}\}$.\\
\end{proof}

\noindent
\textbf{Remark.} If $T$ is allowed to be negative, i.e. $T \in \mathbb{R}$, then the following bounds hold:
\begin{equation}
    \min_l \{\textbf{z}(x)^{(l)}\} \leq \sum_{l=1}^L \textbf{z}(x)^{(l)}\sigma_{SM}\big(\textbf{z}(x)/T\big)^{(l)}
    \leq \max_l \{\textbf{z}(x)^{(l)}\}.
\end{equation}

\begin{theorem}
\label{thm:NLL_opt}
 Given $n$ logit vector maps $\textbf{z}_1, ..., \textbf{z}_n$ and label maps $S_1, ..., S_n$, the optimal 
 temperature values of temperature scaling (TS), image-based temperature scaling (IBTS) and local temperature scaling (LTS) to the following NLL minimization problem
\begin{equation}
\begin{split}
\min_{\alpha_i(x)} - &\sum_{i=1}^n \sum_{x \in \Omega} \log \Big(\sigma_{SM}\big(\alpha_i(x) \textbf{z}_i(x) \big)^{(S_i(x))}\Big)\\
&subject \; to \quad \alpha_i(x) \geq 0
\end{split}
\end{equation}
are
\begin{equation}
\begin{split}
&\begin{cases}
\;\; \alpha^*=0, \quad\quad\quad\quad \text{if}\quad \sum_{i=1}^{n} \sum_{x \in \Omega} \textbf{z}_i(x)^{(S_i(x))} \leq \frac{1}{L}\sum_{i=1}^n\sum_{x \in \Omega}\sum_{l=1}^L \textbf{z}_i(x)^{(l)}\\ 
\Big\{\alpha^*>0  \mid \sum_{i=1}^n\sum_{x \in \Omega}\sum_{l=1}^L \textbf{z}_i(x)^{(l)}\sigma_{SM}\big( \alpha^* \textbf{z}_i(x) \big)^{(l)} = \sum_{i=1}^{n} \sum_{x \in \Omega} \textbf{z}_i(x)^{(S_i(x))} \Big\}, \text{otherwise}
\end{cases}\\
&\begin{cases} 
\;\; \alpha_i^*=0, \quad\quad\quad\quad \text{if}\quad \sum_{x \in \Omega} \textbf{z}_i(x)^{(S_i(x))} \leq \frac{1}{L}\sum_{x \in \Omega}\sum_{l=1}^L \textbf{z}_i(x)^{(l)}\\ 
\Big\{\alpha_i^*>0  \mid \sum_{x \in \Omega}\sum_{l=1}^L \textbf{z}_i(x)^{(l)}\sigma_{SM}\big( \alpha_i^* \textbf{z}_i(x) \big)^{(l)} = \sum_{x \in \Omega} \textbf{z}_i(x)^{(S_i(x))} \Big\}, \text{otherwise}
\end{cases}\\
&\begin{cases}
\;\; \alpha_i(x)^*=0, \quad\quad\quad\quad \text{if}\quad \textbf{z}_i(x)^{(S_i(x))} \leq \frac{1}{L}\sum_{l=1}^L \textbf{z}_i(x)^{(l)}\\ 
\Big\{\alpha_i(x)^*>0  \mid \sum_{l=1}^L \textbf{z}_i(x)^{(l)}\sigma_{SM}\big( \alpha_i(x)^* \textbf{z}_i(x) \big)^{(l)} =  \textbf{z}_i(x)^{(S_i(x))} \Big\}, \text{otherwise}
\end{cases},
\end{split} 
\end{equation}
where
\begin{equation}
\begin{split}
\text{(TS):}\quad \alpha_i(x) &\coloneqq \alpha, \forall i, x, \quad \text{and} \quad T \coloneqq \frac{1}{\alpha}, T \in \mathbb{R}^+\\
\text{(IBTS):}\quad \alpha_i(x) &\coloneqq \alpha_i, \forall x, \quad \text{and} \quad T_i \coloneqq \frac{1}{\alpha_i}, T_i \in \mathbb{R}^+\\
\text{(LTS):}\quad \alpha_i(x) &\coloneqq \alpha_i(x), \quad \text{and} \quad T_i(x) \coloneqq \frac{1}{\alpha_i(x)}, T_i(x) \in \mathbb{R}^+ .
\end{split}
\end{equation}\\
\end{theorem}

\begin{proof}

For TS, Let
\begin{equation}
\mathcal{F}(\alpha) = - \sum_{i=1}^n \sum_{x \in \Omega} \log \Big(\sigma_{SM}\big(\alpha \textbf{z}_i(x) \big)^{(S_i(x))}\Big).
\end{equation}
Taking the derivative w.r.t. $\alpha$ we obtain
\begin{equation}
\frac{\partial \mathcal{F}(\alpha)}{\partial \alpha} = -\sum_{i=1}^n \sum_{x \in \Omega} \Big( \textbf{z}_i(x)^{(S_i(x))} - \sum_{l=1}^L \textbf{z}_i(x)^{(l)}\sigma_{SM} \big( \alpha \textbf{z}_i(x) \big)^{(l)}  \Big) 
\end{equation}

\noindent
\textbf{Case 1:} If $\sum_{i=1}^{n} \sum_{x \in \Omega} \textbf{z}_i(x)^{(S_i(x))} \leq \frac{1}{L}\sum_{i=1}^n\sum_{x \in \Omega}\sum_{l=1}^L \textbf{z}_i(x)^{(l)}$, we have $\frac{\partial \mathcal{F}(\alpha)}{\partial \alpha}\mid_{\alpha=0} \geq 0$. With Lemma~\ref{lamma_bound}, $\mathcal{F}(\alpha)$ is a monotonic increasing function. This indicates the minimum value is achieved at $\alpha=0$.\\
\textbf{Case 2:} If $\sum_{i=1}^{n} \sum_{x \in \Omega} \textbf{z}_i(x)^{(S_i(x))} > \frac{1}{L}\sum_{i=1}^n\sum_{x \in \Omega}\sum_{l=1}^L \textbf{z}_i(x)^{(l)}$. With Lemma~\ref{lamma_bound} we have $\frac{\partial \mathcal{F}(\alpha)}{\partial \alpha}\mid_{\alpha=0} < 0$ and $\frac{\partial \mathcal{F}(\alpha)}{\partial \alpha}\mid_{\alpha \rightarrow +\infty} \geq 0$. From the intermediate value theorem and Lemma~\ref{lamma_bound}, we know there exists a unique $\alpha^*$ ($\{\alpha^*>0  \mid \sum_{i=1}^n\sum_{x \in \Omega}\sum_{l=1}^L \textbf{z}_i(x)^{(l)}\sigma_{SM}\big( \alpha^* \textbf{z}_i(x) \big)^{(j)} = \sum_{i=1}^{n} \sum_{x \in \Omega} \textbf{z}_i(x)^{(S_i(x))} \}$) such that $\frac{\partial \mathcal{F}(\alpha)}{\partial \alpha}\mid_{\alpha = \alpha^*} = 0$. This $\alpha^*$ is the point where $\mathcal{F}(\alpha)$ reaches the minimum value.\\

\noindent
For IBTS, let
\begin{equation}
\mathcal{F}(\alpha_i) = - \sum_{i=1}^n \sum_{x \in \Omega} \log \Big(\sigma_{SM}\big(\alpha_i \textbf{z}_i(x) \big)^{(S_i(x))}\Big).
\end{equation}
Taking the derivative w.r.t. $\alpha_i$, we obtain
\begin{equation}
\frac{\partial \mathcal{F}(\alpha_i)}{\partial \alpha_i} = - \sum_{x \in \Omega} \Big( \textbf{z}_i(x)^{(S_i(x))} - \sum_{l=1}^L \textbf{z}_i(x)^{(l)}\sigma_{SM} \big( \alpha_i \textbf{z}_i(x) \big)^{(l)}  \Big), \quad \forall i.  
\end{equation}

\noindent
\textbf{Case 1:} If $\sum_{x \in \Omega} \textbf{z}_i(x)^{(S_i(x))} \leq \frac{1}{L}\sum_{x \in \Omega}\sum_{l=1}^L \textbf{z}_i(x)^{(l)}$, we have $\frac{\partial \mathcal{F}(\alpha_i)}{\partial \alpha_i}\mid_{\alpha_i=0} \geq 0$. With Lemma~\ref{lamma_bound}, $\mathcal{F}(\alpha_i)$ is a monotonic increasing function. This indicates the minimum value is achieved at $\alpha_i=0$.\\
\textbf{Case 2:} If $\sum_{x \in \Omega} \textbf{z}_i(x)^{(S_i(x))} > \frac{1}{L}\sum_{x \in \Omega}\sum_{l=1}^L \textbf{z}_i(x)^{(l)}$. With Lemma~\ref{lamma_bound} we have $\frac{\partial \mathcal{F}(\alpha_i)}{\partial \alpha_i}\mid_{\alpha_i=0} < 0$ and $\frac{\partial \mathcal{F}(\alpha_i)}{\partial \alpha_i}\mid_{\alpha_i \rightarrow +\infty} \geq 0$. From the intermediate value theorem and Lemma~\ref{lamma_bound}, we know there exists a unique $\alpha_i^*$ ($\{\alpha_i^*>0  \mid \sum_{x \in \Omega}\sum_{l=1}^L \textbf{z}_i(x)^{(l)}\sigma_{SM}\big( \alpha_i^* \textbf{z}_i(x) \big)^{(j)} = \sum_{x \in \Omega} \textbf{z}_i(x)^{(S_i(x))} \}$) such that $\frac{\partial \mathcal{F}(\alpha_i)}{\partial \alpha_i}\mid_{\alpha_i = \alpha_i^*} = 0$. This $\alpha_i^*$ is the point where $\mathcal{F}(\alpha_i)$ reaches the minimum value.\\

\noindent
For LTS, let
\begin{equation}
\mathcal{F}(\alpha_i(x)) = - \sum_{i=1}^n \sum_{x \in \Omega} \log \Big(\sigma_{SM}\big(\alpha_i(x) \textbf{z}_i(x) \big)^{(S_i(x))}\Big).
\end{equation}
Taking the derivative w.r.t. $\alpha_i(x)$, we obtain
\begin{equation}
\frac{\partial \mathcal{F}(\alpha_i(x))}{\partial \alpha_i(x)} = - \Big( \textbf{z}_i(x)^{(S_i(x))} - \sum_{l=1}^L \textbf{z}_i(x)^{(l)}\sigma_{SM} \big( \alpha_i(x) \textbf{z}_i(x) \big)^{(l)}  \Big), \quad \forall i, x.
\end{equation}

\noindent
\textbf{Case 1:} If $\textbf{z}_i(x)^{(S_i(x))} \leq \frac{1}{L}\sum_{l=1}^L \textbf{z}_i(x)^{(l)}$, we have $\frac{\partial \mathcal{F}(\alpha_i(x))}{\partial \alpha_i(x)}\mid_{\alpha_i(x)=0} \geq 0$. With Lemma~\ref{lamma_bound}, $\mathcal{F}(\alpha_i(x))$ is a monotonic increasing function. This indicates the minimum value is achieved at $\alpha_i(x)=0$.\\
\textbf{Case 2:} If $\textbf{z}_i(x)^{(S_i(x))} > \frac{1}{L}\sum_{l=1}^L \textbf{z}_i(x)^{(l)}$. With Lemma~\ref{lamma_bound} we have $\frac{\partial \mathcal{F}(\alpha_i(x))}{\partial \alpha_i(x)}\mid_{\alpha_i(x)=0} < 0$ and $\frac{\partial \mathcal{F}(\alpha_i(x))}{\partial \alpha_i(x)}\mid_{\alpha_i(x) \rightarrow +\infty} \geq 0$. From the intermediate value theorem and Lemma~\ref{lamma_bound}, we know there exists a unique $\alpha_i(x)^*$ ($\{\alpha_i(x)^*>0  \mid \sum_{x \in \Omega}\sum_{l=1}^L \textbf{z}_i(x)^{(l)}\sigma_{SM}\big( \alpha_i(x)^* \textbf{z}_i(x) \big)^{(j)} = \sum_{x \in \Omega} \textbf{z}_i(x)^{(S_i(x))} \}$) such that $\frac{\partial \mathcal{F}(\alpha_i(x))}{\partial \alpha_i(x)}\mid_{\alpha_i(x) = \alpha_i(x)^*} = 0$. This $\alpha_i(x)^*$ is the point where $\mathcal{F}(\alpha_i(x))$ reaches the minimum value.

\end{proof}

\noindent
\textbf{Remark.} The original temperature scaling method defines $T$ instead of $\alpha$ in Theorem~\ref{thm:NLL_opt}. $T$ and $\alpha$ are exchangeable via $T = \frac{1}{\alpha}$. Here we use $\alpha$ to make the proof readable and easy to follow. Furthermore, the definition of temperature scaling requires the temperature value $T>0$. By using $\alpha$, we require $\alpha \geq 0$ with $\alpha \rightarrow 0$ when $T \rightarrow +\infty$.\\

\subsection{Overconfidence and Underconfidence}
\label{app:overfitting_underfitting}
\noindent
One indication of overconfidence for semantic segmentation is that the NLL is greater than or equal to the entropy on the testing dataset (and also the validation dataset) (see \cref{sec:why_cross_entropy_and_temperature_scaling} for a detailed explanation). As demonstrated by~\cite{mukhoti2020calibrating}, this greater-than relationship is mainly because the network gradually becomes more and more confident on its incorrect predictions. Mathematically, before calibration, we have the following relationship on the validation (or testing) dataset:
\begin{equation}
\label{eq:TS_overfitting}
    - \sum_{i=1}^n \sum_{x \in \Omega} \log \Big(\sigma_{SM}\big(\textbf{z}_i(x) \big)^{(S_i(x))}\Big) \geq -\sum_{i=1}^n \sum_{x \in \Omega} \sum_{l=1}^L \sigma_{SM}\big(\textbf{z}_i(x)\big)^{(l)} \log \Big(\sigma_{SM}\big(\textbf{z}_i(x)\big)^{(l)}\Big).
\end{equation}
Furthermore, Eq.~\eqref{eq:TS_overfitting} leads to
\begin{align}
    - \sum_{i=1}^n \sum_{x \in \Omega}\Big[ \textbf{z}_i(x)^{(S_i(x))} + \log \Big( \sum_{l=1}^L \exp(\textbf{z}_i(x)^{(l)}) \Big) \Big] &\geq -\sum_{i=1}^n \sum_{x \in \Omega} \Big[ \sum_{l=1}^L \textbf{z}_i(x)^{(l)} \sigma_{SM}\big(\textbf{z}_i(x)\big)^{(l)}\\
    &\quad+ \underbrace{\sum_{l=1}^L \sigma_{SM}\big(\textbf{z}_i(x)\big)^{(l)}}_{\textbf{=1}} \log \Big( \sum_{l=1}^L \exp(\textbf{z}_i(x)^{(l)}) \Big) \Big] \notag\\
    - \sum_{i=1}^n \sum_{x \in \Omega}\Big[ \textbf{z}_i(x)^{(S_i(x))} + \bcancel{\log \Big( \sum_{l=1}^L \exp(\textbf{z}_i(x)^{(l)}) \Big)} \Big] &\geq -\sum_{i=1}^n \sum_{x \in \Omega} \Big[ \sum_{l=1}^L \textbf{z}_i(x)^{(l)} \sigma_{SM}\big(\textbf{z}_i(x)\big)^{(l)}\\
    &\quad+ \bcancel{\log \Big( \sum_{l=1}^L \exp(\textbf{z}_i(x)^{(l)}) \Big)} \Big] \notag\\
    \sum_{i=1}^n \sum_{x \in \Omega} \textbf{z}_i(x)^{(S_i(x))} &\leq \sum_{i=1}^n \sum_{x \in \Omega} \sum_{l=1}^L \textbf{z}_i(x)^{(l)} \sigma_{SM}\big(\textbf{z}_i(x)\big)^{(l)}. \label{eq:TS_constraint_from}
\end{align}
Eq.~\eqref{eq:TS_constraint_from} is where the idea of the TS constraint in Eq.~\eqref{thm:1} is coming from. Similarly, if we assume
\begin{align}
    - \sum_{x \in \Omega} \log \Big(\sigma_{SM}\big(\textbf{z}_i(x) \big)^{(S_i(x))}\Big) \geq - \sum_{x \in \Omega} \sum_{l=1}^L \sigma_{SM}\big(\textbf{z}_i(x)\big)^{(l)} \log \Big(\sigma_{SM}\big(\textbf{z}_i(x)\big)^{(l)}\Big) \quad \forall i\\
    - \log \Big(\sigma_{SM}\big(\textbf{z}_i(x) \big)^{(S_i(x))}\Big) \geq - \sum_{l=1}^L \sigma_{SM}\big(\textbf{z}_i(x)\big)^{(l)} \log \Big(\sigma_{SM}\big(\textbf{z}_i(x)\big)^{(l)}\Big) \quad \forall i, x,
\end{align}
we get
\begin{align}
    \sum_{x \in \Omega} \textbf{z}_i(x)^{(S_i(x))} &\leq \sum_{x \in \Omega} \sum_{l=1}^L \textbf{z}_i(x)^{(l)} \sigma_{SM}\big(\textbf{z}_i(x)\big)^{(l)}, \quad \forall i \label{eq:IBTS_constraint_from}\\
    \textbf{z}_i(x)^{(S_i(x))} &\leq \sum_{l=1}^L \textbf{z}_i(x)^{(l)} \sigma_{SM}\big(\textbf{z}_i(x)\big)^{(l)}, \quad \forall i, x. \label{eq:LTS_constraint_from}
\end{align}
Hence, Eq.~\eqref{eq:IBTS_constraint_from} is where the idea of the IBTS constraint in Eq.~\eqref{thm:1} is coming from and Eq.~\eqref{eq:LTS_constraint_from} is where the idea of the LTS constraint in Eq.~\eqref{thm:1} is coming from.\\

\begin{definition}
\label{def:overfitting_def}
For semantic segmentation, a model is \textbf{overconfident} for the predicted probabilities in $n$ validation images if
\begin{equation}
\begin{split}
    -\sum_{i=1}^n \sum_{x \in \Omega} \sum_{l=1}^L \sigma_{SM}\big(\textbf{z}_i(x)\big)^{(l)} \log \Big(\sigma_{SM}\big(\textbf{z}_i(x)\big)^{(l)}\Big) &\leq
    - \sum_{i=1}^n \sum_{x \in \Omega} \log \Big(\sigma_{SM}\big(\textbf{z}_i(x) \big)^{(S_i(x))}\Big) \\
    &\text{or}\\
    \sum_{i=1}^n \sum_{x \in \Omega} \textbf{z}_i(x)^{(S_i(x))} &\leq \sum_{i=1}^n \sum_{x \in \Omega} \sum_{l=1}^L \textbf{z}_i(x)^{(l)} \sigma_{SM}\big(\textbf{z}_i(x)\big)^{(l)} \,;
\end{split}
\end{equation}
a model is \textbf{overconfident} for the predicted probabilities in a validation image $I_i$ if
\begin{equation}
\begin{split}
    - \sum_{x \in \Omega} \sum_{l=1}^L \sigma_{SM}\big(\textbf{z}_i(x)\big)^{(l)} \log \Big(\sigma_{SM}\big(\textbf{z}_i(x)\big)^{(l)}\Big) &\leq - \sum_{x \in \Omega} \log \Big(\sigma_{SM}\big(\textbf{z}_i(x) \big)^{(S_i(x))}\Big) \\
    &\text{or}\\
    \sum_{x \in \Omega} \textbf{z}_i(x)^{(S_i(x))} &\leq \sum_{x \in \Omega} \sum_{l=1}^L \textbf{z}_i(x)^{(l)} \sigma_{SM}\big(\textbf{z}_i(x)\big)^{(l)} \,;
\end{split}
\end{equation}
a model is \textbf{overconfident} for the predicted probabilities at position $x$ of a validation image $I_i$ if
\begin{equation}
\begin{split}
    - \sum_{l=1}^L \sigma_{SM}\big(\textbf{z}_i(x)\big)^{(l)} \log \Big(\sigma_{SM}\big(\textbf{z}_i(x)\big)^{(l)}\Big) &\leq - \log \Big(\sigma_{SM}\big(\textbf{z}_i(x) \big)^{(S_i(x))}\Big)\\
    &\text{or}\\
    \textbf{z}_i(x)^{(S_i(x))} &\leq \sum_{l=1}^L \textbf{z}_i(x)^{(l)} \sigma_{SM}\big(\textbf{z}_i(x)\big)^{(l)}\,.
\end{split}
\end{equation}
\end{definition}

\noindent
Furthermore, for underconfidence of semantic segmentation, the NLL is generally less than or equal to the entropy. This is because, when training is insufficient, for correct predictions we have NLL less than or equal to the entropy while for incorrect predictions there is no guaranteed relationship between NLL and entropy. Besides, the majority of the pixel/voxel label predictions for a semantic segmentation are correct after the network has been trained a certain period of time (before overconfidence). Hence, NLL will is expected to be less than or equal to the entropy on average during the underconfident stage.
Thus we have the following constraints during underconfidence,
\begin{align}
    \sum_{i=1}^n \sum_{x \in \Omega} \textbf{z}_i(x)^{(S_i(x))} &\geq \sum_{i=1}^n \sum_{x \in \Omega} \sum_{l=1}^L \textbf{z}_i(x)^{(l)} \sigma_{SM}\big(\textbf{z}_i(x)\big)^{(l)} \label{eq:Underfitting_TS_constraint_from}\\
    \sum_{x \in \Omega} \textbf{z}_i(x)^{(S_i(x))} &\geq \sum_{x \in \Omega} \sum_{l=1}^L \textbf{z}_i(x)^{(l)} \sigma_{SM}\big(\textbf{z}_i(x)\big)^{(l)}, \quad \forall i \label{eq:Underfitting_IBTS_constraint_from}\\
    \textbf{z}_i(x)^{(S_i(x))} &\geq \sum_{l=1}^L \textbf{z}_i(x)^{(l)} \sigma_{SM}\big(\textbf{z}_i(x)\big)^{(l)}, \quad \forall i, x. \label{eq:Underfitting_LTS_constraint_from}
\end{align}
Eq.~\eqref{eq:Underfitting_TS_constraint_from}, Eq.~\eqref{eq:Underfitting_IBTS_constraint_from}, and Eq.~\eqref{eq:Underfitting_LTS_constraint_from} are the prototypes of the constraints for TS, IBTS, LTS in Theorem~\ref{thm:underfitting_entropy}.\\

\begin{definition}
\label{def:underfitting_def}
For semantic segmentation, a model is \textbf{underconfident} for the predicted probabilities in $n$ validation images if
\begin{equation}
\begin{split}
    -\sum_{i=1}^n \sum_{x \in \Omega} \sum_{l=1}^L \sigma_{SM}\big(\textbf{z}_i(x)\big)^{(l)} \log \Big(\sigma_{SM}\big(\textbf{z}_i(x)\big)^{(l)}\Big) &\geq
    - \sum_{i=1}^n \sum_{x \in \Omega} \log \Big(\sigma_{SM}\big(\textbf{z}_i(x) \big)^{(S_i(x))}\Big) \\
    &\text{or}\\
    \sum_{i=1}^n \sum_{x \in \Omega} \textbf{z}_i(x)^{(S_i(x))} &\geq \sum_{i=1}^n \sum_{x \in \Omega} \sum_{l=1}^L \textbf{z}_i(x)^{(l)} \sigma_{SM}\big(\textbf{z}_i(x)\big)^{(l)} \,;
\end{split}
\end{equation}
a model is \textbf{underconfident} for the predicted probabilities in a validation image $I_i$ if
\begin{equation}
\begin{split}
    - \sum_{x \in \Omega} \sum_{l=1}^L \sigma_{SM}\big(\textbf{z}_i(x)\big)^{(l)} \log \Big(\sigma_{SM}\big(\textbf{z}_i(x)\big)^{(l)}\Big) &\geq - \sum_{x \in \Omega} \log \Big(\sigma_{SM}\big(\textbf{z}_i(x) \big)^{(S_i(x))}\Big) \\
    &\text{or}\\
    \sum_{x \in \Omega} \textbf{z}_i(x)^{(S_i(x))} &\geq \sum_{x \in \Omega} \sum_{l=1}^L \textbf{z}_i(x)^{(l)} \sigma_{SM}\big(\textbf{z}_i(x)\big)^{(l)} \,;
\end{split}
\end{equation}
a model is \textbf{underconfident} for the predicted probabilities at position $x$ of a validation image $I_i$ if
\begin{equation}
\begin{split}
    - \sum_{l=1}^L \sigma_{SM}\big(\textbf{z}_i(x)\big)^{(l)} \log \Big(\sigma_{SM}\big(\textbf{z}_i(x)\big)^{(l)}\Big) &\geq - \log \Big(\sigma_{SM}\big(\textbf{z}_i(x) \big)^{(S_i(x))}\Big)\\
    &\text{or}\\
    \textbf{z}_i(x)^{(S_i(x))} &\geq \sum_{l=1}^L \textbf{z}_i(x)^{(l)} \sigma_{SM}\big(\textbf{z}_i(x)\big)^{(l)}\,.
\end{split}
\end{equation}
\end{definition}

\begin{definition}
\label{def:balanced_def}
For semantic segmentation, a model is \textbf{balanced} for the predicted probabilities in $n$ validation images if
\begin{equation}
\begin{split}
    -\sum_{i=1}^n \sum_{x \in \Omega} \sum_{l=1}^L \sigma_{SM}\big(\textbf{z}_i(x)\big)^{(l)} \log \Big(\sigma_{SM}\big(\textbf{z}_i(x)\big)^{(l)}\Big) &=
    - \sum_{i=1}^n \sum_{x \in \Omega} \log \Big(\sigma_{SM}\big(\textbf{z}_i(x) \big)^{(S_i(x))}\Big) \\
    &\text{or}\\
    \sum_{i=1}^n \sum_{x \in \Omega} \textbf{z}_i(x)^{(S_i(x))} &= \sum_{i=1}^n \sum_{x \in \Omega} \sum_{l=1}^L \textbf{z}_i(x)^{(l)} \sigma_{SM}\big(\textbf{z}_i(x)\big)^{(l)} \,;
\end{split}
\end{equation}
a model is \textbf{balanced} for the predicted probabilities in a validation image $I_i$ if
\begin{equation}
\begin{split}
    - \sum_{x \in \Omega} \sum_{l=1}^L \sigma_{SM}\big(\textbf{z}_i(x)\big)^{(l)} \log \Big(\sigma_{SM}\big(\textbf{z}_i(x)\big)^{(l)}\Big) &= - \sum_{x \in \Omega} \log \Big(\sigma_{SM}\big(\textbf{z}_i(x) \big)^{(S_i(x))}\Big) \\
    &\text{or}\\
    \sum_{x \in \Omega} \textbf{z}_i(x)^{(S_i(x))} &= \sum_{x \in \Omega} \sum_{l=1}^L \textbf{z}_i(x)^{(l)} \sigma_{SM}\big(\textbf{z}_i(x)\big)^{(l)} \,;
\end{split}
\end{equation}
a model is \textbf{balanced} for the predicted probabilities at position $x$ of a validation image $I_i$ if
\begin{equation}
\begin{split}
    - \sum_{l=1}^L \sigma_{SM}\big(\textbf{z}_i(x)\big)^{(l)} \log \Big(\sigma_{SM}\big(\textbf{z}_i(x)\big)^{(l)}\Big) &= - \log \Big(\sigma_{SM}\big(\textbf{z}_i(x) \big)^{(S_i(x))}\Big)\\
    &\text{or}\\
    \textbf{z}_i(x)^{(S_i(x))} &= \sum_{l=1}^L \textbf{z}_i(x)^{(l)} \sigma_{SM}\big(\textbf{z}_i(x)\big)^{(l)}\,.
\end{split}
\end{equation}
\end{definition}

\subsection{Weighted Averaged Logits and Entropy Extremes}
\label{app:entropy_extrmes_without_constraints}

\noindent
\begin{lemma}
\label{lamma_unconstraint}
Given $n$ logit vector maps $\textbf{z}_1, ..., \textbf{z}_n$, equal probability for all labels is the unique solution $q$ (probability distribution) to the following entropy maximization problem:
\begin{equation}
\begin{split}
\label{lemma2}
\max_{q} \quad& -\sum_{i=1}^n \sum_{x \in \Omega} \sum_{l=1}^L q\big(\textbf{z}_i(x)\big)^{(l)} \log \Big(q\big(\textbf{z}_i(x)\big)^{(l)}\Big) \\
subject \; to \quad& q\big(\textbf{z}_i(x)\big)^{(l)} \geq 0 \quad \forall i, x, l\\
& \sum_{l=1}^L q\big(\textbf{z}_i(x)\big)^{(l)} = 1 \quad \forall i, x\\
\end{split}
\end{equation}
\end{lemma}

\begin{proof}
We use Lagrangian multipliers to solve the optimization problem. $q\big(\textbf{z}_i(x)\big)^{(l)} \geq 0$ is ignored in the Lagrangian but the deducted solution satisfies this constraint automatically.

\noindent
Let $\beta_i(x)$ be the multipliers. The Lagrangian is
\begin{equation}
\begin{split}
\mathcal{L} = & -\sum_{i=1}^n \sum_{x \in \Omega} \sum_{l=1}^L q\big(\textbf{z}_i(x)\big)^{(l)} \log \Big(q\big(\textbf{z}_i(x)\big)^{(l)}\Big)
+ \sum_{i=1}^n \sum_{x \in \Omega} \beta_i (x) \bigg(\sum_{l=1}^L q\big(\textbf{z}_i(x)\big)^{(l)} - 1 \bigg). 
\end{split}
\end{equation} 
We take the derivative with respect to $q\big(\textbf{z}_i(x)\big)^{(l)}$ and set it to 0
\begin{equation}
\frac{\partial \mathcal{L}}{\partial q\big(\textbf{z}_i(x)\big)^{(l)}} = -1 -\log \Big(q\big(\textbf{z}_i(x)\big)^{(l)}\Big) + \beta_i(x) = 0.
\end{equation}
Thus, we obtain the expression of $q\big(\textbf{z}_i(x)\big)^{(l)}$ as
\begin{equation}
q\big(\textbf{z}_i(x)\big)^{(l)} = e^{\beta_i(x)-1}.
\end{equation}
Hence, $q\big(\textbf{z}_i(x)\big)^{(l)} \geq 0$. Since $\sum_{l=1}^L q\big(\textbf{z}_i(x)\big)^{(l)} = 1$ for all $i$ and $x$, it must satisfy
\begin{equation}
q\big(\textbf{z}_i(x)\big)^{(l)} = \frac{1}{L}.
\end{equation}
Hence the equal probability distribution over all labels is the entropy maximization solution.\\
\end{proof}

\noindent
\textbf{Remark.} For a classification or semantic segmentation task, equal probability for each label will yield the maximum entropy.\\

\noindent
\textbf{Remark.} The minimum entropy lies at extreme points, i.e.
\begin{equation}
\left.\begin{aligned}
\argmin_{q} \quad& -\sum_{i=1}^n \sum_{x \in \Omega} \sum_{l=1}^L q\big(\textbf{z}_i(x)\big)^{(l)} \log \Big(q\big(\textbf{z}_i(x)\big)^{(l)}\Big) \\
subject \; to \quad& q\big(\textbf{z}_i(x)\big)^{(l)} \geq 0 \quad \forall i, x, l\\
& \sum_{l=1}^L q\big(\textbf{z}_i(x)\big)^{(l)} = 1 \quad \forall i, x\\
\end{aligned}\right\} \Big\{q\big(\textbf{z}_i(x)\big)^{(l)}=1, q\big(\textbf{z}_i(x)\big)^{(j)}=0, (\forall j\neq i) \Big\}, \forall i 
\end{equation}
\\

\subsection{Entropy Extremes Under Constraints}
\label{app:entropy_entrmes_with_constraints}
\noindent
\begin{theorem}
\label{thm:entropy_temperature}
Given $n$ logit vector maps $\textbf{z}_1, ..., \textbf{z}_n$ and label maps $S_1, ..., S_n$, temperature scaling (TS), image-based temperature scaling (IBTS) and local temperature scaling (LTS) are the unique solutions $q$ (probability distribution)
to the following entropy maximization problem with different constraints (A, B or C):

\begin{equation}
\begin{split}
\label{thm:1}
\max_{q} \quad& -\sum_{i=1}^n \sum_{x \in \Omega} \sum_{l=1}^L q\big(\textbf{z}_i(x)\big)^{(l)} \log \Big(q\big(\textbf{z}_i(x)\big)^{(l)}\Big) \\
subject \; to \quad& q\big(\textbf{z}_i(x)\big)^{(l)} \geq 0 \quad \forall i, x, l\\
& \sum_{l=1}^L q\big(\textbf{z}_i(x)\big)^{(l)} = 1 \quad \forall i, x\\
& \begin{cases}
		\sum_{i=1}^n\sum_{x \in \Omega}\sum_{l=1}^L \textbf{z}_i(x)^{(l)}q\big(\textbf{z}_i(x)\big)^{(l)} \geq \varepsilon^A & \text{(A: TS constraint)} \\
		\sum_{x \in \Omega} \sum_{l=1}^L \textbf{z}_i(x)^{(l)}q\big(\textbf{z}_i(x)\big)^{(l)} \geq \varepsilon_i^B \quad \forall i & \text{(B: IBTS constraint)} \\
		\sum_{l=1}^L \textbf{z}_i(x)^{(l)}q\big(\textbf{z}_i(x)\big)^{(l)} \geq \varepsilon_i^C (x) \quad \forall i, x & \text{(C: LTS constraint)}
\end{cases}\\
\end{split}
\end{equation}
where $\varepsilon^A$, $\varepsilon^B_i$ and $\varepsilon^C_i (x)$ are the following constants:
\begin{equation}
\begin{split}
 \varepsilon^A &= \sum_{i=1}^{n} \sum_{x \in \Omega} \textbf{z}_i(x)^{(S_i(x))}\,, \\
 \varepsilon_i^B &= \sum_{x \in \Omega} \textbf{z}_i(x)^{(S_i(x))}\,, \\
 \varepsilon_i^C (x) &= \textbf{z}_i(x)^{(S_i(x))}\,. 
\end{split}
\end{equation}
And the corresponding optimal inverse temperature values for TS, IBTS and LTS are
\begin{equation}
\begin{split}
&\begin{cases}
\;\; \alpha^*=0, \quad\quad\quad\quad \text{if}\quad \sum_{i=1}^{n} \sum_{x \in \Omega} \textbf{z}_i(x)^{(S_i(x))} \leq \frac{1}{L}\sum_{i=1}^n\sum_{x \in \Omega}\sum_{l=1}^L \textbf{z}_i(x)^{(l)}\\ 
\Big\{\alpha^*>0  \mid \sum_{i=1}^n\sum_{x \in \Omega}\sum_{l=1}^L \textbf{z}_i(x)^{(l)}\sigma_{SM}\big( \alpha^* \textbf{z}_i(x) \big)^{(j)} = \sum_{i=1}^{n} \sum_{x \in \Omega} \textbf{z}_i(x)^{(S_i(x))} \Big\}, \text{otherwise}
\end{cases}\\
&\begin{cases} 
\;\; \alpha_i^*=0, \quad\quad\quad\quad \text{if}\quad \sum_{x \in \Omega} \textbf{z}_i(x)^{(S_i(x))} \leq \frac{1}{L}\sum_{x \in \Omega}\sum_{l=1}^L \textbf{z}_i(x)^{(l)}\\ 
\Big\{\alpha_i^*>0  \mid \sum_{x \in \Omega}\sum_{l=1}^L \textbf{z}_i(x)^{(l)}\sigma_{SM}\big( \alpha_i^* \textbf{z}_i(x) \big)^{(j)} = \sum_{x \in \Omega} \textbf{z}_i(x)^{(S_i(x))} \Big\}, \text{otherwise}
\end{cases}\\
&\begin{cases}
\;\; \alpha_i(x)^*=0, \quad\quad\quad\quad \text{if}\quad \textbf{z}_i(x)^{(S_i(x))} \leq \frac{1}{L}\sum_{l=1}^L \textbf{z}_i(x)^{(l)}\\ 
\Big\{\alpha_i(x)^*>0  \mid \sum_{l=1}^L \textbf{z}_i(x)^{(l)}\sigma_{SM}\big( \alpha_i(x)^* \textbf{z}_i(x) \big)^{(j)} =  \textbf{z}_i(x)^{(S_i(x))} \Big\}, \text{otherwise}
\end{cases}.
\end{split} 
\end{equation}\\
\end{theorem}

\begin{proof}
We use the Karush–Kuhn–Tucker (KKT) conditions to solve the optimization problems. $q\big(\textbf{z}_i(x)\big)^{(l)} \geq 0$ is ignored for the KKT conditions as the deducted solution satisfies this constraint automatically (i.e., it is inactive).

\noindent 
For constraint A, let $\alpha$, $\beta_i(x)$ be the multipliers. The Lagrangian is
\begin{equation}
\begin{split}
\mathcal{L} = & -\sum_{i=1}^n \sum_{x \in \Omega} \sum_{l=1}^L q\big(\textbf{z}_i(x)\big)^{(l)} \log \Big(q\big(\textbf{z}_i(x)\big)^{(l)}\Big)
- \sum_{i=1}^n \sum_{x \in \Omega} \beta_i (x) \bigg(\sum_{l=1}^L q\big(\textbf{z}_i(x)\big)^{(l)} - 1 \bigg)\\
& - \alpha \Big( \varepsilon^A - \sum_{i=1}^n\sum_{x \in \Omega}\sum_{l=1}^L \textbf{z}_i(x)^{(l)}q\big(\textbf{z}_i(x)\big)^{(l)} \Big). 
\end{split}
\end{equation} 
Thus, the KKT conditions are
\begin{align}
\frac{\partial \mathcal{L}}{\partial q\big(\textbf{z}_i(x)\big)^{(l)}} = -1 -\log \Big(q\big(\textbf{z}_i(x)\big)^{(l)}\Big) + \alpha \textbf{z}_i(x)^{(l)} - \beta_i(x) &= 0 \quad \forall i, l, x, \label{eq:KKT_A_1}\\
\sum_{l=1}^L q\big(\textbf{z}_i(x)\big)^{(l)} - 1 &= 0 \quad \forall i, x, \label{eq:KKT_A_2}\\
\varepsilon^A - \sum_{i=1}^n\sum_{x \in \Omega}\sum_{l=1}^L \textbf{z}_i(x)^{(l)}q\big(\textbf{z}_i(x)\big)^{(l)} &\leq 0, \label{eq:KKT_A_3}\\
\alpha &\geq 0, \label{eq:KKT_A_4}\\
\alpha \Big(\varepsilon^A - \sum_{i=1}^n\sum_{x \in \Omega}\sum_{l=1}^L \textbf{z}_i(x)^{(l)}q\big(\textbf{z}_i(x)\big)^{(l)} \Big) &= 0. \label{eq:KKT_A_5}
\end{align}
From Eq.~\eqref{eq:KKT_A_1}, we obtain the expression of $q\big(\textbf{z}_i(x)\big)^{(l)}$ as
\begin{equation}
q\big(\textbf{z}_i(x)\big)^{(l)} = e^{\alpha \textbf{z}_i(x)^{(l)} - \beta_i(x)-1}.
\end{equation}
Hence, $q\big(\textbf{z}_i(x)\big)^{(l)} \geq 0$. Since $\sum_{l=1}^L q\big(\textbf{z}_i(x)\big)^{(l)} = 1$ (Eq.~\eqref{eq:KKT_A_2}) for all $i$ and $x$, it must satisfy
\begin{equation}
\label{eq:ts_sol}
q\big(\textbf{z}_i(x)\big)^{(l)} = \frac{e^{\alpha \textbf{z}_i(x)^{(l)}}}{\sum_{j=1}^L e^{\alpha \textbf{z}_i(x)^{(j)}}}.
\end{equation}
From Eq.~\eqref{eq:KKT_A_3}, we have
\begin{equation}
\begin{split}
\sum_{i=1}^n\sum_{x \in \Omega}\sum_{l=1}^L \textbf{z}_i(x)^{(l)}q\big(\textbf{z}_i(x)\big)^{(l)} &= \sum_{i=1}^n\sum_{x \in \Omega}\sum_{l=1}^L \textbf{z}_i(x)^{(l)}\frac{e^{\alpha \textbf{z}_i(x)^{(l)}}}{\sum_{j=1}^L e^{\alpha \textbf{z}_i(x)^{(j)}}} \\
&\geq \varepsilon^A \\
&= \sum_{i=1}^{n} \sum_{x \in \Omega} \textbf{z}_i(x)^{(S_i(x))}.    
\end{split}
\end{equation}
\textbf{Case 1:} If $\sum_{i=1}^{n} \sum_{x \in \Omega} \textbf{z}_i(x)^{(S_i(x))} > \frac{1}{L}\sum_{i=1}^n\sum_{x \in \Omega}\sum_{l=1}^L \textbf{z}_i(x)^{(l)}$, then we have
\begin{equation}
\begin{split}
\sum_{i=1}^n\sum_{x \in \Omega}\sum_{l=1}^L \textbf{z}_i(x)^{(l)}q\big(\textbf{z}_i(x)\big)^{(l)} &= \sum_{i=1}^n\sum_{x \in \Omega}\sum_{l=1}^L \textbf{z}_i(x)^{(l)}\frac{e^{\alpha \textbf{z}_i(x)^{(l)}}}{\sum_{j=1}^L e^{\alpha \textbf{z}_i(x)^{(j)}}} \\
&\geq \sum_{i=1}^{n} \sum_{x \in \Omega} \textbf{z}_i(x)^{(S_i(x))} \\
&> \frac{1}{L}\sum_{i=1}^n\sum_{x \in \Omega}\sum_{l=1}^L \textbf{z}_i(x)^{(l)}.    
\end{split}    
\end{equation}

\noindent
If $\alpha=0$, then $q\big(\textbf{z}_i(x)\big)^{(l)}=1/L$ for all $i$, $l$ and $x$. Thus, Eq.~\eqref{eq:KKT_A_3} becomes $\varepsilon^A - \sum_{i=1}^n\sum_{x \in \Omega}\sum_{l=1}^L \textbf{z}_i(x)^{(l)}\frac{1}{L} \leq 0$, which violates the $\sum_{i=1}^{n} \sum_{x \in \Omega} \textbf{z}_i(x)^{(S_i(x))} > \frac{1}{L}\sum_{i=1}^n\sum_{x \in \Omega}\sum_{l=1}^L \textbf{z}_i(x)^{(l)}$ assumption. Hence, $\alpha \neq 0$.

\noindent
Furthermore, we have 
\begin{equation}
\frac{1}{L}\sum_{i=1}^n\sum_{x \in \Omega}\sum_{l=1}^L \textbf{z}_i(x)^{(l)} < \sum_{i=1}^{n} \sum_{x \in \Omega} \textbf{z}_i(x)^{(S_i(x))} \leq \sum_{i=1}^n\sum_{x \in \Omega} \max_l \{ \textbf{z}_i(x)^{(l)} \},
\end{equation}
with Lemma~\ref{lamma_bound} and the intermediate value theorem, there must be a unique strictly positive solution $\alpha^*$ for $\alpha$ such that $\sum_{i=1}^n\sum_{x \in \Omega}\sum_{l=1}^L \textbf{z}_i(x)^{(l)}q\big(\textbf{z}_i(x)\big)^{(l)} = \varepsilon^A = \sum_{i=1}^{n} \sum_{x \in \Omega} \textbf{z}_i(x)^{(S_i(x))}$. Thus Eq.~\eqref{eq:KKT_A_4} and Eq.~\eqref{eq:KKT_A_5} both hold.\\

\noindent
\textbf{Case 2:} If $\sum_{i=1}^{n} \sum_{x \in \Omega} \textbf{z}_i(x)^{(S_i(x))} \leq \frac{1}{L}\sum_{i=1}^n\sum_{x \in \Omega}\sum_{l=1}^L \textbf{z}_i(x)^{(l)}$. \\
If $\alpha \neq 0$, Eq.~\eqref{eq:KKT_A_5} yields $\sum_{i=1}^n\sum_{x \in \Omega}\sum_{l=1}^L \textbf{z}_i(x)^{(l)}q\big(\textbf{z}_i(x)\big)^{(l)} = \varepsilon^A = \sum_{i=1}^{n} \sum_{x \in \Omega} \textbf{z}_i(x)^{(S_i(x))}$. With Lemma~\ref{lamma_bound} and intermediate value theorem, there exists a unique non-positive $\alpha$. This violates Eq.~\eqref{eq:KKT_A_4} and the $\alpha \neq 0$ assumption. Thus, $\alpha=0$.\\
Furthermore, when $\alpha=0$, it yields $q\big(\textbf{z}_i(x)\big)^{(l)}=1/L$ for all $i$, $l$ and $x$. Take $q\big(\textbf{z}_i(x)\big)^{(l)}=1/L$ into Eq.~\eqref{eq:KKT_A_3}, the inequality holds. Eq.~\eqref{eq:KKT_A_4} and Eq.~\eqref{eq:KKT_A_5} also hold. From Lemma~\ref{lamma_unconstraint}, we know that $q\big(\textbf{z}_i(x)\big)^{(l)}=1/L$ is the solution for entropy maximization of Eq.~\eqref{lemma2}. Since Eq.~\eqref{thm:1} is the subproblem of Eq.~\eqref{lemma2}, $q\big(\textbf{z}_i(x)\big)^{(l)}=1/L$ also reaches the entropy maximization of Eq.~\eqref{thm:1}. \\

\noindent
Overall, the optimal solution is 
\begin{equation}
q\big(\textbf{z}_i(x)\big)^{(l)} = \frac{e^{\alpha^* \textbf{z}_i(x)^{(l)}}}{\sum_{j=1}^L e^{\alpha^* \textbf{z}_i(x)^{(j)}}},
\end{equation}
with
\begin{equation}
\begin{cases}
\alpha^* = 0, \quad\quad\quad\quad \text{if}\quad \sum_{i=1}^{n} \sum_{x \in \Omega} \textbf{z}_i(x)^{(S_i(x))} \leq \frac{1}{L}\sum_{i=1}^n\sum_{x \in \Omega}\sum_{l=1}^L \textbf{z}_i(x)^{(l)}\\
\{\alpha^*>0  \mid \sum_{i=1}^n\sum_{x \in \Omega}\sum_{l=1}^L \textbf{z}_i(x)^{(l)}\frac{e^{\alpha^* \textbf{z}_i(x)^{(l)}}}{\sum_{j=1}^L e^{\alpha^* \textbf{z}_i(x)^{(j)}}} = \sum_{i=1}^{n} \sum_{x \in \Omega} \textbf{z}_i(x)^{(S_i(x))} \}, \quad \text{otherwise}
\end{cases}    
\end{equation}
Let $T = \frac{1}{\alpha^*}$ ($\alpha^* \rightarrow 0$ as $T \rightarrow +\infty$), then this is the TS solution. Note that $T$ does not depend on $i$ and $x$, which is the same as the temperature value in Eq.~\eqref{eq:TS_opt}.\\

\noindent
For constraint B, let $\alpha_i$, $\beta_i(x)$ be the multipliers. Then the Lagrangian is
\begin{equation}
\begin{split}
\mathcal{L} = & -\sum_{i=1}^n \sum_{x \in \Omega} \sum_{l=1}^L q\big(\textbf{z}_i(x)\big)^{(l)} \log \Big(q\big(\textbf{z}_i(x)\big)^{(l)}\Big) - \sum_{i=1}^n \sum_{x \in \Omega} \beta_i (x) \bigg(\sum_{l=1}^L q\big(\textbf{z}_i(x)\big)^{(l)} - 1 \bigg)\\
& - \sum_{i=1}^{n} \alpha_i \Big( \varepsilon^B_i - \sum_{x \in \Omega}\sum_{l=1}^L \textbf{z}_i(x)^{(l)}q\big(\textbf{z}_i(x)\big)^{(l)} \Big).
\end{split}
\end{equation} 
Thus, the KKT conditions are
\begin{align}
\frac{\partial \mathcal{L}}{\partial q\big(\textbf{z}_i(x)\big)^{(l)}} = -1 -\log \Big(q\big(\textbf{z}_i(x)\big)^{(l)} \Big) + \alpha_i \textbf{z}_i(x)^{(l)} - \beta_i(x) &= 0 \quad \forall i, l, x, \label{eq:KKT_B_1}\\
\sum_{l=1}^L q\big(\textbf{z}_i(x)\big)^{(l)} - 1 &= 0 \quad \forall i, x, \label{eq:KKT_B_2}\\
\varepsilon^B_i - \sum_{x \in \Omega}\sum_{l=1}^L \textbf{z}_i(x)^{(l)}q\big(\textbf{z}_i(x)\big)^{(l)} &\leq 0 \quad \forall i, \label{eq:KKT_B_3}\\
\alpha_i &\geq 0 \quad \forall i, \label{eq:KKT_B_4}\\
\alpha_i \Big( \varepsilon^B_i - \sum_{x \in \Omega}\sum_{l=1}^L \textbf{z}_i(x)^{(l)}q\big(\textbf{z}_i(x)\big)^{(l)} \Big) &= 0 \quad \forall i. \label{eq:KKT_B_5}
\end{align}
From Eq.~\eqref{eq:KKT_B_1}, we obtain the expression of $q\big(\textbf{z}_i(x)\big)^{(l)}$ as
\begin{equation}
q\big(\textbf{z}_i(x)\big)^{(l)} = e^{\alpha_i \textbf{z}_i(x)^{(l)} - \beta_i(x)-1}.
\end{equation}
Hence, $q\big(\textbf{z}_i(x)\big)^{(l)} \geq 0$. Since $\sum_{l=1}^L q\big(\textbf{z}_i(x)\big)^{(l)} = 1$ (Eq.~\eqref{eq:KKT_B_2}) for all $i$ and $x$, it must have
\begin{equation}
\label{eq:ibts_sol}
q\big(\textbf{z}_i(x)\big)^{(l)} = \frac{e^{\alpha_i \textbf{z}_i(x)^{(l)}}}{\sum_{j=1}^L e^{\alpha_i \textbf{z}_i(x)^{(j)}}},
\end{equation}
From Eq.~\eqref{eq:KKT_B_3}, we have
\begin{equation}
\begin{split}
\sum_{x \in \Omega}\sum_{l=1}^L \textbf{z}_i(x)^{(l)}q\big(\textbf{z}_i(x)\big)^{(l)} &= \sum_{x \in \Omega}\sum_{l=1}^L \textbf{z}_i(x)^{(l)}\frac{e^{\alpha_i \textbf{z}_i(x)^{(l)}}}{\sum_{j=1}^L e^{\alpha_i \textbf{z}_i(x)^{(j)}}} \\
&\geq \varepsilon^B_i \\
&= \sum_{x \in \Omega} \textbf{z}_i(x)^{(S_i(x))}.    
\end{split}
\end{equation}
\textbf{Case 1:} If $\sum_{x \in \Omega} \textbf{z}_i(x)^{(S_i(x))} > \frac{1}{L}\sum_{x \in \Omega}\sum_{l=1}^L \textbf{z}_i(x)^{(l)}$, then we have
\begin{equation}
\begin{split}
\sum_{x \in \Omega}\sum_{l=1}^L \textbf{z}_i(x)^{(l)}q\big(\textbf{z}_i(x)\big)^{(l)} &= \sum_{x \in \Omega}\sum_{l=1}^L \textbf{z}_i(x)^{(l)}\frac{e^{\alpha_i \textbf{z}_i(x)^{(l)}}}{\sum_{j=1}^L e^{\alpha_i \textbf{z}_i(x)^{(j)}}} \\
&\geq \sum_{x \in \Omega} \textbf{z}_i(x)^{(S_i(x))} \\
&> \frac{1}{L}\sum_{x \in \Omega}\sum_{l=1}^L \textbf{z}_i(x)^{(l)}.    
\end{split}    
\end{equation}

\noindent
If $\alpha_i=0$, then $q\big(\textbf{z}_i(x)\big)^{(l)}=1/L$ for all $i$, $l$ and $x$. Thus, Eq.~\eqref{eq:KKT_B_3} becomes $\varepsilon^B_i - \sum_{x \in \Omega}\sum_{l=1}^L \textbf{z}_i(x)^{(l)}\frac{1}{L} \leq 0$, which violates the $\sum_{x \in \Omega} \textbf{z}_i(x)^{(S_i(x))} > \frac{1}{L}\sum_{x \in \Omega}\sum_{l=1}^L \textbf{z}_i(x)^{(l)}$ assumption. Hence, $\alpha_i \neq 0$.

\noindent
Furthermore, we have 
\begin{equation}
\frac{1}{L}\sum_{x \in \Omega}\sum_{l=1}^L \textbf{z}_i(x)^{(l)} < \sum_{x \in \Omega} \textbf{z}_i(x)^{(S_i(x))} \leq \sum_{x \in \Omega} \max_l \{ \textbf{z}_i(x)^{(l)} \},
\end{equation}
with Lemma~\ref{lamma_bound} and the intermediate value theorem, there must be a unique strictly positive solution $\alpha_i^*$ for $\alpha_i$ such that $\sum_{x \in \Omega}\sum_{l=1}^L \textbf{z}_i(x)^{(l)}q\big(\textbf{z}_i(x)\big)^{(l)} = \varepsilon^B_i = \sum_{x \in \Omega} \textbf{z}_i(x)^{(S_i(x))}$. Thus Eq.~\eqref{eq:KKT_B_4} and Eq.~\eqref{eq:KKT_B_5} both hold.\\

\noindent
\textbf{Case 2:} If $\sum_{x \in \Omega} \textbf{z}_i(x)^{(S_i(x))} < \frac{1}{L}\sum_{x \in \Omega}\sum_{l=1}^L \textbf{z}_i(x)^{(l)}$. \\
If $\alpha_i \neq 0$, Eq.~\eqref{eq:KKT_B_5} yields $\sum_{x \in \Omega}\sum_{l=1}^L \textbf{z}_i(x)^{(l)}q\big(\textbf{z}_i(x)\big)^{(l)} = \varepsilon^B_i = \sum_{x \in \Omega} \textbf{z}_i(x)^{(S_i(x))}$. With Lemma~\ref{lamma_bound} and the intermediate value theorem, there exists a unique non-positive $\alpha_i$. This violates Eq.~\eqref{eq:KKT_B_4} and the $\alpha_i \neq 0$ assumption. Thus, $\alpha_i=0$.\\
Furthermore, when $\alpha_i=0$, it yields $q\big(\textbf{z}_i(x)\big)^{(l)}=1/L$ for all $i$, $l$ and $x$. Take $q\big(\textbf{z}_i(x)\big)^{(l)}=1/L$ into Eq.~\eqref{eq:KKT_B_3}, the inequality holds. Eq.~\eqref{eq:KKT_B_4} and Eq.~\eqref{eq:KKT_B_5} also hold. From Lemma~\ref{lamma_unconstraint}, we know that $q\big(\textbf{z}_i(x)\big)^{(l)}=1/L$ is the solution for entropy maximization of Eq.~\eqref{lemma2}. Since Eq.~\eqref{thm:1} is the subproblem of Eq.~\eqref{lemma2}, $q\big(\textbf{z}_i(x)\big)^{(l)}=1/L$ also reaches the entropy maximization of Eq.~\eqref{thm:1}. \\

\noindent
Overall, the optimal solution is 
\begin{equation}
q\big(\textbf{z}_i(x)\big)^{(l)} = \frac{e^{\alpha_i^* \textbf{z}_i(x)^{(l)}}}{\sum_{j=1}^L e^{\alpha_i^* \textbf{z}_i(x)^{(j)}}},
\end{equation}
with
\begin{equation}
\begin{cases}
\alpha_i^* = 0, \quad\quad\quad\quad \text{if}\quad \sum_{x \in \Omega} \textbf{z}_i(x)^{(S_i(x))} \leq \frac{1}{L}\sum_{x \in \Omega}\sum_{l=1}^L \textbf{z}_i(x)^{(l)}\\
\{ \alpha_i^* > 0 \mid \sum_{x \in \Omega}\sum_{l=1}^L \textbf{z}_i(x)^{(l)}\frac{e^{\alpha_i^* \textbf{z}_i(x)^{(l)}}}{\sum_{j=1}^L e^{\alpha_i^* \textbf{z}_i(x)^{(j)}}} = \sum_{x \in \Omega} \textbf{z}_i(x)^{(S_i(x))} \}, \quad \text{otherwise}
\end{cases}    
\end{equation}
Let $T_i = \frac{1}{\alpha_i^*}$ ($\alpha_i^* \rightarrow 0$ as $T_i \rightarrow +\infty$), then this is the IBTS solution. Note that $T_i$ does not depend on $x$, which is the same as the temperature value in Eq.~\eqref{eq:image_temperature_scaling}.\\

\noindent
For constraint C, let $\alpha_i(x)$, $\beta_i(x)$ be the multipliers. Then the Lagrangian is
\begin{equation}
\begin{split}
\mathcal{L} = & -\sum_{i=1}^n \sum_{x \in \Omega} \sum_{l=1}^L q\big(\textbf{z}_i(x)\big)^{(l)} \log \Big(q\big(\textbf{z}_i(x)\big)^{(l)}\Big) 
- \sum_{i=1}^n \sum_{x \in \Omega} \beta_i (x) \bigg(\sum_{l=1}^L q\big(\textbf{z}_i(x)\big)^{(l)} - 1 \bigg)\\
& - \sum_{i=1}^{n} \sum_{x \in \Omega} \alpha_i(x) \Big( \varepsilon^C_i (x) - \sum_{l=1}^L \textbf{z}_i(x)^{(l)}q\big(\textbf{z}_i(x)\big)^{(l)} \Big).
\end{split}
\end{equation} 
Thus, the KKT conditions are
\begin{align}
\frac{\partial \mathcal{L}}{\partial q\big(\textbf{z}_i(x)\big)^{(l)}} = -1 -\log \Big(q\big(\textbf{z}_i(x)\big)^{(l)}\Big) + \alpha_i(x) \textbf{z}_i(x)^{(l)} - \beta_i(x) &= 0 \quad \forall i, x, l, \label{eq:KKT_C_1}\\
\sum_{l=1}^L q\big(\textbf{z}_i(x)\big)^{(l)} - 1 &= 0 \quad \forall i, x, \label{eq:KKT_C_2}\\
\varepsilon^C_i(x) - \sum_{l=1}^L \textbf{z}_i(x)^{(l)}q\big(\textbf{z}_i(x)\big)^{(l)} &\leq 0 \quad \forall i,x, \label{eq:KKT_C_3}\\
\alpha_i(x) &\geq 0 \quad \forall i, x, \label{eq:KKT_C_4}\\
\alpha_i(x) \Big( \varepsilon^C_i(x) - \sum_{l=1}^L \textbf{z}_i(x)^{(l)}q\big(\textbf{z}_i(x)\big)^{(l)} \Big) &= 0 \quad \forall i, x. \label{eq:KKT_C_5}
\end{align}
From Eq.~\eqref{eq:KKT_C_1}, we obtain the expression of $q\big(\textbf{z}_i(x)\big)^{(l)}$ as
\begin{equation}
q\big(\textbf{z}_i(x)\big)^{(l)} = e^{\alpha_i(x) \textbf{z}_i(x)^{(l)} - \beta_i(x)-1}.
\end{equation}
Hence, $q\big(\textbf{z}_i(x)\big)^{(l)} \geq 0$. Since $\sum_{l=1}^L q(\textbf{z}_i(x))^{(l)} = 1$ (Eq.~\eqref{eq:KKT_C_2}) for all $i$ and $x$, it must have
\begin{equation}
q\big(\textbf{z}_i(x)\big)^{(l)} = \frac{e^{\alpha_i(x) \textbf{z}_i(x)^{(l)}}}{\sum_{j=1}^L e^{\alpha_i(x) \textbf{z}_i(x)^{(j)}}},
\end{equation}
From Eq.~\eqref{eq:KKT_C_3}, we have
\begin{equation}
\begin{split}
\sum_{l=1}^L \textbf{z}_i(x)^{(l)}q\big(\textbf{z}_i(x)\big)^{(l)} &= \sum_{l=1}^L \textbf{z}_i(x)^{(l)}\frac{e^{\alpha_i(x) \textbf{z}_i(x)^{(l)}}}{\sum_{j=1}^L e^{\alpha_i(x) \textbf{z}_i(x)^{(j)}}} \\
&\geq \varepsilon^C_i(x) \\
&= \textbf{z}_i(x)^{(S_i(x))}.    
\end{split}
\end{equation}
\textbf{Case 1:} If $\textbf{z}_i(x)^{(S_i(x))} > \frac{1}{L}\sum_{l=1}^L \textbf{z}_i(x)^{(l)}$, then we have
\begin{equation}
\begin{split}
\sum_{l=1}^L \textbf{z}_i(x)^{(l)}q\big(\textbf{z}_i(x)\big)^{(l)} &= \sum_{l=1}^L \textbf{z}_i(x)^{(l)}\frac{e^{\alpha_i(x) \textbf{z}_i(x)^{(l)}}}{\sum_{j=1}^L e^{\alpha_i(x) \textbf{z}_i(x)^{(j)}}} \\
&\geq \textbf{z}_i(x)^{(S_i(x))} \\
&> \frac{1}{L}\sum_{l=1}^L \textbf{z}_i(x)^{(l)}.    
\end{split}    
\end{equation}

\noindent
If $\alpha_i(x)=0$, then $q\big(\textbf{z}_i(x)\big)^{(l)}=1/L$ for all $i$, $l$ and $x$. Thus, Eq.~\eqref{eq:KKT_C_3} becomes $\varepsilon^C_i(x) - \sum_{l=1}^L \textbf{z}_i(x)^{(l)}\frac{1}{L} \leq 0$, which violates the $\textbf{z}_i(x)^{(S_i(x))} > \frac{1}{L}\sum_{l=1}^L \textbf{z}_i(x)^{(l)}$ assumption. Hence, $\alpha_i(x) \neq 0$.

\noindent
Furthermore, we have 
\begin{equation}
\frac{1}{L}\sum_{l=1}^L \textbf{z}_i(x)^{(l)} < \textbf{z}_i(x)^{(S_i(x))} \leq \max_l \{ \textbf{z}_i(x)^{(l)} \},
\end{equation}
with Lemma~\ref{lamma_bound} and the intermediate value theorem, there must be a unique strictly positive solution $\alpha_i(x)^*$ for $\alpha_i(x)$ such that $\sum_{l=1}^L \textbf{z}_i(x)^{(l)}q\big(\textbf{z}_i(x)\big)^{(l)} = \varepsilon^C_i(x) = \textbf{z}_i(x)^{(S_i(x))}$. Thus Eq.~\eqref{eq:KKT_C_4} and Eq.~\eqref{eq:KKT_C_5} both hold.\\

\noindent
\textbf{Case 2:} If $\textbf{z}_i(x)^{(S_i(x))} < \frac{1}{L}\sum_{l=1}^L \textbf{z}_i(x)^{(l)}$. \\
If $\alpha_i(x) \neq 0$, Eq.~\eqref{eq:KKT_C_5} yields $\sum_{l=1}^L \textbf{z}_i(x)^{(l)}q\big(\textbf{z}_i(x)\big)^{(l)} = \varepsilon^C_i(x) = \textbf{z}_i(x)^{(S_i(x))}$. With Lemma~\ref{lamma_bound} and the intermediate value theorem, there exists a unique non-positive $\alpha_i$. This violates Eq.~\eqref{eq:KKT_C_4} and $\alpha_i(x) \neq 0$ assumption. Thus, $\alpha_i(x)=0$.\\
Furthermore, when $\alpha_i(x)=0$, it yields $q\big(\textbf{z}_i(x)\big)^{(l)}=1/L$ for all $i$, $l$ and $x$. Take $q\big(\textbf{z}_i(x)\big)^{(l)}=1/L$ into Eq.~\eqref{eq:KKT_C_3}, the inequality holds. Eq.~\eqref{eq:KKT_C_4} and Eq.~\eqref{eq:KKT_C_5} also hold. From Lemma~\ref{lamma_unconstraint}, we know that $q\big(\textbf{z}_i(x)\big)^{(l)}=1/L$ is the solution for entropy maximization of Eq.~\eqref{lemma2}. Since Eq.~\eqref{thm:1} is the subproblem of Eq.~\eqref{lemma2}, $q\big(\textbf{z}_i(x)\big)^{(l)}=1/L$ also reaches the entropy maximization of Eq.~\eqref{thm:1}. \\

\noindent
Overall, the optimal solution is 
\begin{equation}
q\big(\textbf{z}_i(x)\big)^{(l)} = \frac{e^{\alpha_i(x)^* \textbf{z}_i(x)^{(l)}}}{\sum_{j=1}^L e^{\alpha_i(x)^* \textbf{z}_i(x)^{(j)}}},
\end{equation}
with
\begin{equation}
\begin{cases}
\alpha_i(x)^* = 0, \quad\quad\quad\quad \text{if}\quad \textbf{z}_i(x)^{(S_i(x))} \leq \frac{1}{L}\sum_{l=1}^L \textbf{z}_i(x)^{(l)}\\
\{ \alpha_i(x)^* > 0 \mid \sum_{l=1}^L \textbf{z}_i(x)^{(l)}\frac{e^{\alpha_i(x)^* \textbf{z}_i(x)^{(l)}}}{\sum_{j=1}^L e^{\alpha_i(x)^* \textbf{z}_i(x)^{(j)}}} = \textbf{z}_i(x)^{(S_i(x))} \}, \quad \text{otherwise}
\end{cases}    
\end{equation}
Let $T_i(x) = \frac{1}{\alpha_i(x)^*}$ ($\alpha_i(x)^* \rightarrow 0$ as $T_i(x) \rightarrow +\infty$), then this is the LTS solution. Note that this $T_i(x)$ depends on $i$ and $x$, which is the same as the temperature value in Eq.~\eqref{eq:local_temperature_scaling}.\\
\end{proof}

\noindent
\textbf{Remark.} Note that the first two constraints on $q(\text{z}_i(x))$ are shared by all three models, while the last constraint varies across the three models, i.e. A for TS, B for IBTS, and C for LTS. The first two constraints guarantee that $q$ is a probability distribution while the last constraint makes assumptions on the distributions of the corresponding models. Constraint A assumes that the average true class logit is less than or equal to the weighted average logit over the entire image space and all samples. Constraint B requires that the avearge true class logit is less than or equal to the weighted average logit over the image space. Constraint C specifies that the true class logit is less than or equal to the weighted average logit at each location of each image. Note that the three constrains are designed under the overconfidence scenario. The order of the restrictiveness of the constraints is C $>$ B $>$ A, which indicates the model complexity order LTS $>$ IBTS $>$ TS.\\

\noindent
\textbf{Remark.} Theorem~\ref{thm:entropy_temperature} gives a more general proof. However, when it comes to TS, IBTS and LTS, we do not necessarily need such strong conditions. Instead we can use the following simplified theorem~\ref{thm:overderfitting_entropy_second_proof}. 

\begin{customthm}{2-b}
\label{thm:overderfitting_entropy_second_proof}
Given $n$ logit vector maps $\textbf{z}_1, ..., \textbf{z}_n$ and label maps $S_1, ..., S_n$, the optimal temperature values of temperature scaling (TS), image-based temperature scaling (IBTS) and local temperature scaling (LTS) to the following entropy maximization problem with different constraints (A, B or C)

\begin{equation}
\begin{split}
\label{eq:thm2b_extropy_maximum}
\max_{\alpha_i(x)} \quad& -\sum_{i=1}^n \sum_{x \in \Omega} \sum_{l=1}^L \sigma_{SM}\big(\alpha_i(x) \textbf{z}_i(x)\big)^{(l)} \log \Big(\sigma_{SM} \big(\alpha_i(x) \textbf{z}_i(x)\big)^{(l)}\Big) \\
subject \; to \quad& \alpha_i(x) \geq 0 \quad \forall i, x, l\\
& \begin{cases}
		\sum_{i=1}^n\sum_{x \in \Omega}\sum_{l=1}^L \textbf{z}_i(x)^{(l)}\sigma_{SM}\big(\alpha_i(x) \textbf{z}_i(x)\big)^{(l)} \geq \varepsilon^A & \text{(A: TS constraint)} \\
		\sum_{x \in \Omega} \sum_{l=1}^L \textbf{z}_i(x)^{(l)}\sigma_{SM}\big(\alpha_i(x) \textbf{z}_i(x)\big)^{(l)} \geq \varepsilon_i^B \quad \forall i & \text{(B: IBTS constraint)} \\
		\sum_{l=1}^L \textbf{z}_i(x)^{(l)}\sigma_{SM}\big(\alpha_i(x) \textbf{z}_i(x)\big)^{(l)} \geq \varepsilon_i^C (x) \quad \forall i, x & \text{(C: LTS constraint)}
\end{cases}\\
\end{split}
\end{equation}
where $\varepsilon^A$, $\varepsilon^B_i$ and $\varepsilon^C_i (x)$ are the following constants:
\begin{equation}
\begin{split}
 \varepsilon^A &= \sum_{i=1}^{n} \sum_{x \in \Omega} \textbf{z}_i(x)^{(S_i(x))} \,, \\
 \varepsilon_i^B &= \sum_{x \in \Omega} \textbf{z}_i(x)^{(S_i(x))} \,, \\
 \varepsilon_i^C (x) &= \textbf{z}_i(x)^{(S_i(x))} \,. 
\end{split}
\end{equation}
are
\begin{equation}
\begin{split}
&\begin{cases}
\;\; \alpha^*=0, \quad\quad\quad\quad \text{if}\quad \sum_{i=1}^{n} \sum_{x \in \Omega} \textbf{z}_i(x)^{(S_i(x))} \leq \frac{1}{L}\sum_{i=1}^n\sum_{x \in \Omega}\sum_{l=1}^L \textbf{z}_i(x)^{(l)}\\ 
\Big\{\alpha^*>0  \mid \sum_{i=1}^n\sum_{x \in \Omega}\sum_{l=1}^L \textbf{z}_i(x)^{(l)}\sigma_{SM}\big( \alpha^* \textbf{z}_i(x) \big)^{(j)} = \sum_{i=1}^{n} \sum_{x \in \Omega} \textbf{z}_i(x)^{(S_i(x))} \Big\}, \text{otherwise}
\end{cases}\\
&\begin{cases} 
\;\; \alpha_i^*=0, \quad\quad\quad\quad \text{if}\quad \sum_{x \in \Omega} \textbf{z}_i(x)^{(S_i(x))} \leq \frac{1}{L}\sum_{x \in \Omega}\sum_{l=1}^L \textbf{z}_i(x)^{(l)}\\ 
\Big\{\alpha_i^*>0  \mid \sum_{x \in \Omega}\sum_{l=1}^L \textbf{z}_i(x)^{(l)}\sigma_{SM}\big( \alpha_i^* \textbf{z}_i(x) \big)^{(j)} = \sum_{x \in \Omega} \textbf{z}_i(x)^{(S_i(x))} \Big\}, \text{otherwise}
\end{cases}\\
&\begin{cases}
\;\; \alpha_i(x)^*=0, \quad\quad\quad\quad \text{if}\quad \textbf{z}_i(x)^{(S_i(x))} \leq \frac{1}{L}\sum_{l=1}^L \textbf{z}_i(x)^{(l)}\\ 
\Big\{\alpha_i(x)^*>0  \mid \sum_{l=1}^L \textbf{z}_i(x)^{(l)}\sigma_{SM}\big( \alpha_i(x)^* \textbf{z}_i(x) \big)^{(j)} =  \textbf{z}_i(x)^{(S_i(x))} \Big\}, \text{otherwise}
\end{cases}.
\end{split} 
\end{equation}\\
where
\begin{equation}
\begin{split}
\text{(TS):}\quad \alpha_i(x) &\coloneqq \alpha, \forall i, x, \quad \text{and} \quad T \coloneqq \frac{1}{\alpha}, T \in \mathbb{R}^+\\
\text{(IBTS):}\quad \alpha_i(x) &\coloneqq \alpha_i, \forall x, \quad \text{and} \quad T_i \coloneqq \frac{1}{\alpha_i}, T_i \in \mathbb{R}^+\\
\text{(LTS):}\quad \alpha_i(x) &\coloneqq \alpha_i(x), \quad \text{and} \quad T_i(x) \coloneqq \frac{1}{\alpha_i(x)}, T_i(x) \in \mathbb{R}^+ .
\end{split}
\end{equation}
\end{customthm}

\begin{proof}
We use the Karush-Kuhn-Tucker (KKT) conditions to solve the optimization problems. $\alpha \geq 0$ is ignored in the Lagrangian and later be validated w.r.t. the deducted solution.
For TS, Let $\lambda$ be the multiplier, the Lagrangian is 
\begin{equation}
\mathcal{L} = -\sum_{i=1}^n \sum_{x \in \Omega} \sum_{l=1}^L \sigma_{SM}\big(\alpha \textbf{z}_i(x)\big)^{(l)} \log \Big(\sigma_{SM} \big(\alpha \textbf{z}_i(x)\big)^{(l)}\Big) - \lambda \Big( \sum_{i=1}^{n} \sum_{x \in \Omega} \textbf{z}_i(x)^{(S_i(x))} - \sum_{i=1}^n\sum_{x \in \Omega}\sum_{l=1}^L \textbf{z}_i(x)^{(l)}\sigma_{SM}\big(\alpha \textbf{z}_i(x)\big)^{(l)} \Big).
\end{equation}

\noindent
Taking the derivative w.r.t. $\alpha$, we have
\begin{align}
    \frac{\partial \mathcal{L}}{\partial \alpha} &= -\sum_{i=1}^n \sum_{x \in \Omega} \sum_{l=1}^L \sigma_{SM}\big(\alpha \textbf{z}_i(x)\big)^{(l)}\Big( \textbf{z}_i(x)^{(l)} - \sum_{j=1}^L \textbf{z}_i(x)^{(j)}\sigma_{SM}\big( \alpha \textbf{z}_i(x) \big)^{(j)} \Big) \log \Big(\sigma_{SM} \big(\alpha \textbf{z}_i(x)\big)^{(l)}\Big) \notag\\
    &\quad -\sum_{i=1}^n \sum_{x \in \Omega} \underbrace{\Big[\sum_{l=1}^L \sigma_{SM}\big(\alpha \textbf{z}_i(x)\big)^{(l)} \Big( \textbf{z}_i(x)^{(l)} - \sum_{j=1}^L \textbf{z}_i(x)^{(j)}\sigma_{SM}\big( \alpha \textbf{z}_i(x) \big)^{(j)} \Big) \Big]}_{=\sum_{l=1}^L \sigma_{SM}\big(\alpha \textbf{z}_i(x)\big)^{(l)} \textbf{z}_i(x)^{(l)} - \underbrace{\sum_{l=1}^L \sigma_{SM}\big(\alpha \textbf{z}_i(x)\big)^{(l)}}_{\textbf{= 1}} \sum_{j=1}^L \textbf{z}_i(x)^{(j)}\sigma_{SM}\big( \alpha \textbf{z}_i(x) \big)^{(j)} \textbf{= 0} } \notag\\
    &\quad + \lambda \sum_{i=1}^n \sum_{x \in \Omega} \sum_{l=1}^L \sigma_{SM}\big(\alpha \textbf{z}_i(x)\big)^{(l)} \textbf{z}_i(x)^{(l)} \Big( \textbf{z}_i(x)^{(l)} - \sum_{j=1}^L \textbf{z}_i(x)^{(j)}\sigma_{SM}\big( \alpha \textbf{z}_i(x) \big)^{(j)} \Big)\\
    &= -\sum_{i=1}^n \sum_{x \in \Omega} \sum_{l=1}^L \sigma_{SM}\big(\alpha \textbf{z}_i(x)\big)^{(l)}\Big( \textbf{z}_i(x)^{(l)} - \sum_{j=1}^L \textbf{z}_i(x)^{(j)}\sigma_{SM}\big( \alpha \textbf{z}_i(x) \big)^{(j)} \Big) \Big(\alpha \textbf{z}_i(x)^{(l)} - \log \big(\sum_{j=1}^L \exp (\alpha \textbf{z}_i(x)^{(j)} ) \big)\Big) \notag\\
    &\quad + \lambda \sum_{i=1}^n \sum_{x \in \Omega} \sum_{l=1}^L \sigma_{SM}\big(\alpha \textbf{z}_i(x)\big)^{(l)} \textbf{z}_i(x)^{(l)} \Big( \textbf{z}_i(x)^{(l)} - \sum_{j=1}^L \textbf{z}_i(x)^{(j)}\sigma_{SM}\big( \alpha \textbf{z}_i(x) \big)^{(j)} \Big)\\
    &= -\alpha \sum_{i=1}^n \sum_{x \in \Omega} \sum_{l=1}^L \textbf{z}_i(x)^{(l)} \sigma_{SM}\big(\alpha \textbf{z}_i(x)\big)^{(l)}\Big( \textbf{z}_i(x)^{(l)} - \sum_{j=1}^L \textbf{z}_i(x)^{(j)}\sigma_{SM}\big( \alpha \textbf{z}_i(x) \big)^{(j)} \Big) \notag\\
    &\quad + \sum_{i=1}^n \sum_{x \in \Omega} \underbrace{\Big[ \sum_{l=1}^L \sigma_{SM}\big(\alpha \textbf{z}_i(x)\big)^{(l)}\Big( \textbf{z}_i(x)^{(l)} - \sum_{j=1}^L \textbf{z}_i(x)^{(j)}\sigma_{SM}\big( \alpha \textbf{z}_i(x) \big)^{(j)} \Big) \Big]}_{\textbf{= 0}} \log \big(\sum_{j=1}^L \exp (\alpha \textbf{z}_i(x)^{(j)} ) \big) \notag\\
    &\quad + \lambda \sum_{i=1}^n \sum_{x \in \Omega} \sum_{l=1}^L \sigma_{SM}\big(\alpha \textbf{z}_i(x)\big)^{(l)} \textbf{z}_i(x)^{(l)} \Big( \textbf{z}_i(x)^{(l)} - \sum_{j=1}^L \textbf{z}_i(x)^{(j)}\sigma_{SM}\big( \alpha \textbf{z}_i(x) \big)^{(j)} \Big)\\
    &= (\lambda -\alpha) \sum_{i=1}^n \sum_{x \in \Omega} \Big( \sum_{l=1}^L \big( \textbf{z}_i(x)^{(l)} \big)^2 \sigma_{SM}\big(\alpha \textbf{z}_i(x)\big)^{(l)} - \big (\sum_{l=1}^L \textbf{z}_i(x)^{(l)} \sigma_{SM}\big(\alpha \textbf{z}_i(x)\big)^{(l)} \big)^2 \Big).
\end{align}

\noindent
Thus, the KKT conditions are
\begin{align}
\frac{\partial \mathcal{L}}{\partial \alpha} = (\lambda -\alpha) \sum_{i=1}^n \sum_{x \in \Omega} \Big( \sum_{l=1}^L \big( \textbf{z}_i(x)^{(l)} \big)^2 \sigma_{SM}\big(\alpha \textbf{z}_i(x)\big)^{(l)} - \big (\sum_{l=1}^L \textbf{z}_i(x)^{(l)} \sigma_{SM}\big(\alpha \textbf{z}_i(x)\big)^{(l)} \big)^2 \Big) &= 0 \quad \forall i, x, \label{eq:KKT_second_A_2}\\
\sum_{i=1}^{n} \sum_{x \in \Omega} \textbf{z}_i(x)^{(S_i(x))} - \sum_{i=1}^n\sum_{x \in \Omega}\sum_{l=1}^L \textbf{z}_i(x)^{(l)}\sigma_{SM}\big(\alpha \textbf{z}_i(x)\big)^{(l)} &\leq 0, \label{eq:KKT_second_A_3}\\
\lambda &\geq 0, \label{eq:KKT_second_A_4}\\
\lambda \Big(\sum_{i=1}^{n} \sum_{x \in \Omega} \textbf{z}_i(x)^{(S_i(x))} - \sum_{i=1}^n\sum_{x \in \Omega}\sum_{l=1}^L \textbf{z}_i(x)^{(l)}\sigma_{SM}\big(\alpha \textbf{z}_i(x)\big)^{(l)} \Big) &= 0. \label{eq:KKT_second_A_5}
\end{align}

\noindent
By the Cauchy-Schwarz inequality, we have
\begin{align}
    &\quad \sum_{l=1}^L \big( \textbf{z}_i(x)^{(l)} \big)^2 \sigma_{SM}\big(\alpha \textbf{z}_i(x)\big)^{(l)} - \big (\sum_{l=1}^L \textbf{z}_i(x)^{(l)} \sigma_{SM}\big(\alpha \textbf{z}_i(x)\big)^{(l)} \big)^2 \notag \\
    &= \Big( \sum_{l=1}^L \big( \textbf{z}_i(x)^{(l)} \big)^2 \sigma_{SM}\big(\alpha \textbf{z}_i(x)\big)^{(l)} \Big) \underbrace{\Big( \sum_{l=1}^L \sigma_{SM}\big(\alpha \textbf{z}_i(x)\big)^{(l)} \Big)}_{\textbf{= 1}} - \Big (\sum_{l=1}^L \textbf{z}_i(x)^{(l)} \sigma_{SM}\big(\alpha \textbf{z}_i(x)\big)^{(l)} \Big)^2\\
    &\geq  \Big (\sum_{l=1}^L |\textbf{z}_i(x)^{(l)}| \sigma_{SM}\big(\alpha \textbf{z}_i(x)\big)^{(l)} \Big)^2 - \Big (\sum_{l=1}^L \textbf{z}_i(x)^{(l)} \sigma_{SM}\big(\alpha \textbf{z}_i(x)\big)^{(l)} \Big)^2\\
    &\geq 0
\end{align}

\noindent
Hence, we have $\lambda=\alpha$ in Eq.~\eqref{eq:KKT_second_A_2}.

\noindent
\textbf{Case 1:} If $\sum_{i=1}^{n} \sum_{x \in \Omega} \textbf{z}_i(x)^{(S_i(x))} > \frac{1}{L}\sum_{i=1}^n\sum_{x \in \Omega}\sum_{l=1}^L \textbf{z}_i(x)^{(l)}$, then we have
\begin{equation}
\begin{split}
\sum_{i=1}^n\sum_{x \in \Omega}\sum_{l=1}^L \textbf{z}_i(x)^{(l)}\sigma_{SM}\big(\alpha \textbf{z}_i(x)\big)^{(l)} &= \sum_{i=1}^n\sum_{x \in \Omega}\sum_{l=1}^L \textbf{z}_i(x)^{(l)}\frac{e^{\alpha \textbf{z}_i(x)^{(l)}}}{\sum_{j=1}^L e^{\alpha \textbf{z}_i(x)^{(j)}}} \\
&\geq \sum_{i=1}^{n} \sum_{x \in \Omega} \textbf{z}_i(x)^{(S_i(x))} \\
&> \frac{1}{L}\sum_{i=1}^n\sum_{x \in \Omega}\sum_{l=1}^L \textbf{z}_i(x)^{(l)}.    
\end{split}    
\end{equation}

\noindent
If $\alpha=0$, then $\sigma_{SM}\big(\alpha \textbf{z}_i(x)\big)^{(l)}=1/L$ for all $i$, $l$ and $x$. Thus, Eq.~\eqref{eq:KKT_second_A_3} becomes $\sum_{i=1}^{n} \sum_{x \in \Omega} \textbf{z}_i(x)^{(S_i(x))} - \sum_{i=1}^n\sum_{x \in \Omega}\sum_{l=1}^L \textbf{z}_i(x)^{(l)}\frac{1}{L} \leq 0$, which violates the $\sum_{i=1}^{n} \sum_{x \in \Omega} \textbf{z}_i(x)^{(S_i(x))} > \frac{1}{L}\sum_{i=1}^n\sum_{x \in \Omega}\sum_{l=1}^L \textbf{z}_i(x)^{(l)}$ assumption. Hence, $\alpha \neq 0$.

\noindent
Furthermore, we have 
\begin{equation}
\frac{1}{L}\sum_{i=1}^n\sum_{x \in \Omega}\sum_{l=1}^L \textbf{z}_i(x)^{(l)} < \sum_{i=1}^{n} \sum_{x \in \Omega} \textbf{z}_i(x)^{(S_i(x))} \leq \sum_{i=1}^n\sum_{x \in \Omega} \max_l \{ \textbf{z}_i(x)^{(l)} \},
\end{equation}
with Lemma~\ref{lamma_bound} and the intermediate value theorem, there must be a unique strictly positive solution $\alpha^*$ for $\alpha$ such that $\sum_{i=1}^n\sum_{x \in \Omega}\sum_{l=1}^L \textbf{z}_i(x)^{(l)}\sigma_{SM}\big(\alpha \textbf{z}_i(x)\big)^{(l)} = \sum_{i=1}^{n} \sum_{x \in \Omega} \textbf{z}_i(x)^{(S_i(x))}$. Thus Eq.~\eqref{eq:KKT_second_A_4} and Eq.~\eqref{eq:KKT_second_A_5} both hold.\\

\noindent
\textbf{Case 2:} If $\sum_{i=1}^{n} \sum_{x \in \Omega} \textbf{z}_i(x)^{(S_i(x))} \leq \frac{1}{L}\sum_{i=1}^n\sum_{x \in \Omega}\sum_{l=1}^L \textbf{z}_i(x)^{(l)}$. \\
If $\alpha \neq 0$, Eq.~\eqref{eq:KKT_second_A_5} and $\lambda=\alpha$ yields $\sum_{i=1}^n\sum_{x \in \Omega}\sum_{l=1}^L \textbf{z}_i(x)^{(l)}q\big(\textbf{z}_i(x)\big)^{(l)} = \sum_{i=1}^{n} \sum_{x \in \Omega} \textbf{z}_i(x)^{(S_i(x))}$. With Lemma~\ref{lamma_bound} and the intermediate value theorem, there exists a unique non-positive $\alpha$. This violates Eq.~\eqref{eq:KKT_second_A_4} and the $\alpha \neq 0$ assumption. Thus, $\alpha=0$.\\
Furthermore, when $\alpha=0$, it yields $\sigma_{SM}\big(\alpha \textbf{z}_i(x)\big)^{(l)}=1/L$ for all $i$, $l$ and $x$. Take $\sigma_{SM}\big(\alpha \textbf{z}_i(x)\big)^{(l)}=1/L$ into Eq.~\eqref{eq:KKT_second_A_3}, the inequality holds. Eq.~\eqref{eq:KKT_second_A_4} and Eq.~\eqref{eq:KKT_second_A_5} also hold. From Lemma~\ref{lamma_unconstraint}, we know that $\sigma_{SM}\big(\alpha \textbf{z}_i(x)\big)^{(l)}=1/L$ is the solution for entropy maximization of Eq.~\eqref{lemma2}. Since Eq.~\eqref{eq:thm2b_extropy_maximum} is the subproblem of Eq.~\eqref{lemma2}, $\sigma_{SM}\big(\alpha \textbf{z}_i(x)\big)^{(l)}=1/L$ also reaches the entropy maximization of Eq.~\eqref{eq:thm2b_extropy_maximum}. \\

\noindent
Overall, the optimal solution is 
\begin{equation}
\begin{cases}
\alpha^* = 0, \quad\quad\quad\quad \text{if}\quad \sum_{i=1}^{n} \sum_{x \in \Omega} \textbf{z}_i(x)^{(S_i(x))} \leq \frac{1}{L}\sum_{i=1}^n\sum_{x \in \Omega}\sum_{l=1}^L \textbf{z}_i(x)^{(l)}\\
\{\alpha^*>0  \mid \sum_{i=1}^n\sum_{x \in \Omega}\sum_{l=1}^L \textbf{z}_i(x)^{(l)}\frac{e^{\alpha^* \textbf{z}_i(x)^{(l)}}}{\sum_{j=1}^L e^{\alpha^* \textbf{z}_i(x)^{(j)}}} = \sum_{i=1}^{n} \sum_{x \in \Omega} \textbf{z}_i(x)^{(S_i(x))} \}, \quad \text{otherwise}
\end{cases}    
\end{equation}
Let $T = \frac{1}{\alpha^*}$ ($\alpha^* \rightarrow 0$ as $T \rightarrow +\infty$), then this is the TS solution. Note that $T$ does not depend on $i$ and $x$, which is the same as the temperature value in Eq.~\eqref{eq:TS_opt}.\\

\noindent
Similarly, for IBTS and LTS, we can get
\begin{align}
    &\argmax_{\alpha_i} -\sum_{i=1}^n \sum_{x \in \Omega} \sum_{l=1}^L \sigma_{SM}\big(\alpha_i \textbf{z}_i(x)\big)^{(l)} \log \Big(\sigma_{SM} \big(\alpha_i \textbf{z}_i(x)\big)^{(l)}\Big) \notag\\
    = &\begin{cases} 
        \;\; \alpha_i^*=0, \quad\quad\quad\quad \text{if}\quad \sum_{x \in \Omega} \textbf{z}_i(x)^{(S_i(x))} \leq \frac{1}{L}\sum_{x \in \Omega}\sum_{l=1}^L \textbf{z}_i(x)^{(l)}\\ 
        \Big\{\alpha_i^*>0  \mid \sum_{x \in \Omega}\sum_{l=1}^L \textbf{z}_i(x)^{(l)}\sigma_{SM}\big( \alpha_i^* \textbf{z}_i(x) \big)^{(j)} = \sum_{x \in \Omega} \textbf{z}_i(x)^{(S_i(x))} \Big\}, \text{otherwise}
        \end{cases}\\
    &\argmax_{\alpha_i(x)} -\sum_{i=1}^n \sum_{x \in \Omega} \sum_{l=1}^L \sigma_{SM}\big(\alpha_i(x) \textbf{z}_i(x)\big)^{(l)} \log \Big(\sigma_{SM} \big(\alpha_i(x) \textbf{z}_i(x)\big)^{(l)}\Big) \notag\\
    = &\begin{cases}
        \;\; \alpha_i(x)^*=0, \quad\quad\quad\quad \text{if}\quad \textbf{z}_i(x)^{(S_i(x))} \leq \frac{1}{L}\sum_{l=1}^L \textbf{z}_i(x)^{(l)}\\ 
        \Big\{\alpha_i(x)^*>0  \mid \sum_{l=1}^L \textbf{z}_i(x)^{(l)}\sigma_{SM}\big( \alpha_i(x)^* \textbf{z}_i(x) \big)^{(j)} =  \textbf{z}_i(x)^{(S_i(x))} \Big\}, \text{otherwise}
        \end{cases}
\end{align}
\end{proof}

\begin{theorem}
\label{thm:underfitting_entropy}
Given $n$ logit vector maps $\textbf{z}_1, ..., \textbf{z}_n$ and label maps $S_1, ..., S_n$, the optimal temperature values of temperature scaling (TS), image-based temperature scaling (IBTS) and local temperature scaling (LTS) to the following entropy minimization problem with different constraints (A, B or C)

\begin{equation}
\begin{split}
\min_{\alpha_i(x)} \quad& -\sum_{i=1}^n \sum_{x \in \Omega} \sum_{l=1}^L \sigma_{SM}\big(\alpha_i(x) \textbf{z}_i(x)\big)^{(l)} \log \Big(\sigma_{SM} \big(\alpha_i(x) \textbf{z}_i(x)\big)^{(l)}\Big) \\
subject \; to \quad& \alpha_i(x) \geq 0 \quad \forall i, x, l\\
& \begin{cases}
		\sum_{i=1}^n\sum_{x \in \Omega}\sum_{l=1}^L \textbf{z}_i(x)^{(l)}\sigma_{SM}\big(\alpha_i(x) \textbf{z}_i(x)\big)^{(l)} \leq \varepsilon^A & \text{(A: TS constraint)} \\
		\sum_{x \in \Omega} \sum_{l=1}^L \textbf{z}_i(x)^{(l)}\sigma_{SM}\big(\alpha_i(x) \textbf{z}_i(x)\big)^{(l)} \leq \varepsilon_i^B \quad \forall i & \text{(B: IBTS constraint)} \\
		\sum_{l=1}^L \textbf{z}_i(x)^{(l)}\sigma_{SM}\big(\alpha_i(x) \textbf{z}_i(x)\big)^{(l)} \leq \varepsilon_i^C (x) \quad \forall i, x & \text{(C: LTS constraint)}
\end{cases}\\
\end{split}
\end{equation}
where $\varepsilon^A$, $\varepsilon^B_i$ and $\varepsilon^C_i (x)$ are the following constants:
\begin{equation}
\begin{split}
 \varepsilon^A &= \sum_{i=1}^{n} \sum_{x \in \Omega} \textbf{z}_i(x)^{(S_i(x))} \geq \frac{1}{L}\sum_{i=1}^{n} \sum_{x \in \Omega} \sum_{l=1}^L \textbf{z}_i(x)^{(l)} \,, \\
 \varepsilon_i^B &= \sum_{x \in \Omega} \textbf{z}_i(x)^{(S_i(x))} \geq \frac{1}{L} \sum_{x \in \Omega} \sum_{l=1}^L \textbf{z}_i(x)^{(l)} \,, \\
 \varepsilon_i^C (x) &= \textbf{z}_i(x)^{(S_i(x))} \geq \frac{1}{L} \sum_{l=1}^L \textbf{z}_i(x)^{(l)} \,. 
\end{split}
\end{equation}
are
\begin{equation}
\begin{split}
&\Big\{\alpha^*\geq0  \mid \sum_{i=1}^n\sum_{x \in \Omega}\sum_{l=1}^L \textbf{z}_i(x)^{(l)}\sigma_{SM}\big( \alpha^* \textbf{z}_i(x) \big)^{(j)} = \sum_{i=1}^{n} \sum_{x \in \Omega} \textbf{z}_i(x)^{(S_i(x))} \Big\},\\
&\Big\{\alpha_i^*\geq0  \mid \sum_{x \in \Omega}\sum_{l=1}^L \textbf{z}_i(x)^{(l)}\sigma_{SM}\big( \alpha_i^* \textbf{z}_i(x) \big)^{(j)} = \sum_{x \in \Omega} \textbf{z}_i(x)^{(S_i(x))} \Big\},\\
&\Big\{\alpha_i(x)^*\geq0  \mid \sum_{l=1}^L \textbf{z}_i(x)^{(l)}\sigma_{SM}\big( \alpha_i(x)^* \textbf{z}_i(x) \big)^{(j)} =  \textbf{z}_i(x)^{(S_i(x))} \Big\}.
\end{split} 
\end{equation}
where
\begin{equation}
\begin{split}
\text{(TS):}\quad \alpha_i(x) &\coloneqq \alpha, \forall i, x, \quad \text{and} \quad T \coloneqq \frac{1}{\alpha}, T \in \mathbb{R}^+\\
\text{(IBTS):}\quad \alpha_i(x) &\coloneqq \alpha_i, \forall x, \quad \text{and} \quad T_i \coloneqq \frac{1}{\alpha_i}, T_i \in \mathbb{R}^+\\
\text{(LTS):}\quad \alpha_i(x) &\coloneqq \alpha_i(x), \quad \text{and} \quad T_i(x) \coloneqq \frac{1}{\alpha_i(x)}, T_i(x) \in \mathbb{R}^+ .
\end{split}
\end{equation}
\end{theorem}

\begin{proof}
For TS, Let
\begin{equation}
\mathcal{F}(\alpha) = -\sum_{i=1}^n \sum_{x \in \Omega} \sum_{l=1}^L \sigma_{SM}\big(\alpha \textbf{z}_i(x)\big)^{(l)} \log \Big(\sigma_{SM} \big(\alpha \textbf{z}_i(x)\big)^{(l)}\Big).
\end{equation}
Taking the derivative w.r.t. $\alpha$, we have
\begin{align}
    \frac{\partial \mathcal{F}(\alpha)}{\partial \alpha} &= -\sum_{i=1}^n \sum_{x \in \Omega} \sum_{l=1}^L \sigma_{SM}\big(\alpha \textbf{z}_i(x)\big)^{(l)}\Big( \textbf{z}_i(x)^{(l)} - \sum_{j=1}^L \textbf{z}_i(x)^{(j)}\sigma_{SM}\big( \alpha \textbf{z}_i(x) \big)^{(j)} \Big) \log \Big(\sigma_{SM} \big(\alpha \textbf{z}_i(x)\big)^{(l)}\Big) \notag\\
    &\quad -\sum_{i=1}^n \sum_{x \in \Omega} \underbrace{\Big[\sum_{l=1}^L \sigma_{SM}\big(\alpha \textbf{z}_i(x)\big)^{(l)} \Big( \textbf{z}_i(x)^{(l)} - \sum_{j=1}^L \textbf{z}_i(x)^{(j)}\sigma_{SM}\big( \alpha \textbf{z}_i(x) \big)^{(j)} \Big) \Big]}_{=\sum_{l=1}^L \sigma_{SM}\big(\alpha \textbf{z}_i(x)\big)^{(l)} \textbf{z}_i(x)^{(l)} - \underbrace{\sum_{l=1}^L \sigma_{SM}\big(\alpha \textbf{z}_i(x)\big)^{(l)}}_{\textbf{= 1}} \sum_{j=1}^L \textbf{z}_i(x)^{(j)}\sigma_{SM}\big( \alpha \textbf{z}_i(x) \big)^{(j)} \textbf{= 0} }\\
    &= -\sum_{i=1}^n \sum_{x \in \Omega} \sum_{l=1}^L \sigma_{SM}\big(\alpha \textbf{z}_i(x)\big)^{(l)}\Big( \textbf{z}_i(x)^{(l)} - \sum_{j=1}^L \textbf{z}_i(x)^{(j)}\sigma_{SM}\big( \alpha \textbf{z}_i(x) \big)^{(j)} \Big) \Big(\alpha \textbf{z}_i(x)^{(l)} - \log \big(\sum_{j=1}^L \exp (\alpha \textbf{z}_i(x)^{(j)} ) \big)\Big)\\
    &= -\alpha \sum_{i=1}^n \sum_{x \in \Omega} \sum_{l=1}^L \textbf{z}_i(x)^{(l)} \sigma_{SM}\big(\alpha \textbf{z}_i(x)\big)^{(l)}\Big( \textbf{z}_i(x)^{(l)} - \sum_{j=1}^L \textbf{z}_i(x)^{(j)}\sigma_{SM}\big( \alpha \textbf{z}_i(x) \big)^{(j)} \Big) \notag\\
    &\quad + \sum_{i=1}^n \sum_{x \in \Omega} \underbrace{\Big[ \sum_{l=1}^L \sigma_{SM}\big(\alpha \textbf{z}_i(x)\big)^{(l)}\Big( \textbf{z}_i(x)^{(l)} - \sum_{j=1}^L \textbf{z}_i(x)^{(j)}\sigma_{SM}\big( \alpha \textbf{z}_i(x) \big)^{(j)} \Big) \Big]}_{\textbf{= 0}} \log \big(\sum_{j=1}^L \exp (\alpha \textbf{z}_i(x)^{(j)} ) \big)\\
    &= -\alpha \sum_{i=1}^n \sum_{x \in \Omega} \Big( \sum_{l=1}^L \big( \textbf{z}_i(x)^{(l)} \big)^2 \sigma_{SM}\big(\alpha \textbf{z}_i(x)\big)^{(l)} - \big (\sum_{l=1}^L \textbf{z}_i(x)^{(l)} \sigma_{SM}\big(\alpha \textbf{z}_i(x)\big)^{(l)} \big)^2 \Big).
\end{align}
By the Cauchy-Schwarz inequality, we have
\begin{align}
    &\quad \sum_{l=1}^L \big( \textbf{z}_i(x)^{(l)} \big)^2 \sigma_{SM}\big(\alpha \textbf{z}_i(x)\big)^{(l)} - \big (\sum_{l=1}^L \textbf{z}_i(x)^{(l)} \sigma_{SM}\big(\alpha \textbf{z}_i(x)\big)^{(l)} \big)^2 \notag \\
    &= \Big( \sum_{l=1}^L \big( \textbf{z}_i(x)^{(l)} \big)^2 \sigma_{SM}\big(\alpha \textbf{z}_i(x)\big)^{(l)} \Big) \underbrace{\Big( \sum_{l=1}^L \sigma_{SM}\big(\alpha \textbf{z}_i(x)\big)^{(l)} \Big)}_{\textbf{= 1}} - \Big (\sum_{l=1}^L \textbf{z}_i(x)^{(l)} \sigma_{SM}\big(\alpha \textbf{z}_i(x)\big)^{(l)} \Big)^2\\
    &\geq  \Big (\sum_{l=1}^L |\textbf{z}_i(x)^{(l)}| \sigma_{SM}\big(\alpha \textbf{z}_i(x)\big)^{(l)} \Big)^2 - \Big (\sum_{l=1}^L \textbf{z}_i(x)^{(l)} \sigma_{SM}\big(\alpha \textbf{z}_i(x)\big)^{(l)} \Big)^2\\
    &\geq 0
\end{align}
Since $\alpha \geq 0$, finally we get
\begin{equation}
    \frac{\partial \mathcal{F}(\alpha)}{\partial \alpha} \leq 0.
\end{equation}
Thus $\mathcal{F}(\alpha)$ is monotonicly decreasing w.r.t. $\alpha$.\\ 

\noindent
Furthermore, we have the following relations by definition 
\begin{align}
\sum_{i=1}^n\sum_{x \in \Omega}\sum_{l=1}^L \textbf{z}_i(x)^{(l)}\sigma_{SM}\big(\alpha \textbf{z}_i(x)\big)^{(l)} &\leq \sum_{i=1}^{n} \sum_{x \in \Omega} \textbf{z}_i(x)^{(S_i(x))}\\
\frac{1}{L}\sum_{i=1}^{n} \sum_{x \in \Omega} \sum_{l=1}^L \textbf{z}_i(x)^{(l)} &\leq \sum_{i=1}^{n} \sum_{x \in \Omega} \textbf{z}_i(x)^{(S_i(x))} \leq \sum_{i=1}^{n} \sum_{x \in \Omega} \max_{l} \{\textbf{z}_i(x)^{(l)} \} \, .
\end{align}
With Lemma~\ref{lamma_bound} and the intermediate value theorem, there must be a unique non-negative solution $\alpha^*$ for $\alpha$ such that $\sum_{i=1}^n\sum_{x \in \Omega}\sum_{l=1}^L \textbf{z}_i(x)^{(l)}\sigma_{SM}\big(\alpha \textbf{z}_i(x)\big)^{(l)} = \sum_{i=1}^{n} \sum_{x \in \Omega} \textbf{z}_i(x)^{(S_i(x))}$. This $\alpha^*$ is also the maximum $\alpha$ that we can get without violating the constraints. Because $\mathcal{F}(\alpha)$ is monotonicly decreasing, thus $\alpha^*$ is the optimal point that minimizes the entropy, i.e.
\begin{align}
    &\argmin_{\alpha} -\sum_{i=1}^n \sum_{x \in \Omega} \sum_{l=1}^L \sigma_{SM}\big(\alpha \textbf{z}_i(x)\big)^{(l)} \log \Big(\sigma_{SM} \big(\alpha \textbf{z}_i(x)\big)^{(l)}\Big) \notag\\
    = &\Big\{\alpha^*\geq0  \mid \sum_{i=1}^n\sum_{x \in \Omega}\sum_{l=1}^L \textbf{z}_i(x)^{(l)}\sigma_{SM}\big( \alpha^* \textbf{z}_i(x) \big)^{(l)} = \sum_{i=1}^{n} \sum_{x \in \Omega} \textbf{z}_i(x)^{(S_i(x))} \Big\}
\end{align}

\noindent
Similarly, for IBTS and LTS, we can get
\begin{align}
    &\argmin_{\alpha_i} -\sum_{i=1}^n \sum_{x \in \Omega} \sum_{l=1}^L \sigma_{SM}\big(\alpha_i \textbf{z}_i(x)\big)^{(l)} \log \Big(\sigma_{SM} \big(\alpha_i \textbf{z}_i(x)\big)^{(l)}\Big) \notag\\
    = &\Big\{\alpha_i^*\geq0  \mid \sum_{i=1}^n\sum_{x \in \Omega}\sum_{l=1}^L \textbf{z}_i(x)^{(l)}\sigma_{SM}\big( \alpha_i^* \textbf{z}_i(x) \big)^{(l)} = \sum_{i=1}^{n} \sum_{x \in \Omega} \textbf{z}_i(x)^{(S_i(x))} \Big\}\\
    &\argmin_{\alpha_i(x)} -\sum_{i=1}^n \sum_{x \in \Omega} \sum_{l=1}^L \sigma_{SM}\big(\alpha_i(x) \textbf{z}_i(x)\big)^{(l)} \log \Big(\sigma_{SM} \big(\alpha_i(x) \textbf{z}_i(x)\big)^{(l)}\Big) \notag\\
    = &\Big\{\alpha_i(x)^*\geq0  \mid \sum_{i=1}^n\sum_{x \in \Omega}\sum_{l=1}^L \textbf{z}_i(x)^{(l)}\sigma_{SM}\big( \alpha_i(x)^* \textbf{z}_i(x) \big)^{(l)} = \sum_{i=1}^{n} \sum_{x \in \Omega} \textbf{z}_i(x)^{(S_i(x))} \Big\}
\end{align}

\end{proof}

\noindent
\textbf{Remark.} Different from the proof in Theorem~\ref{thm:entropy_temperature} where we used KKT conditions, we only used the gradient here and gave a specific expression for the probability (i.e. softmax of logits) to prove Theorem~\ref{thm:underfitting_entropy}. This kind of proof choice is
because (1) the objective function in Theorem~\ref{thm:entropy_temperature} is concave and we want to obtain
the maximum; (2) the constraints in Theorem~\ref{thm:entropy_temperature} are strong enough (self-contained)
to derive the solution.

\subsection{(Local) Temperature Scaling Drives NLL and Entropy to an Equilibrium}
\label{app:NLL_meet_with_entropy}
\begin{theorem}
\label{thm:NLL_Entropy}
    (1) When the to-be-calibrated semantic segmentation network is overconfident, minimizing NLL w.r.t. TS, IBTS, and LTS results in solutions that are also the solutions of maximizing entropy of the calibrated probability w.r.t. TS, IBTS and LTS under the condition of overconfidence. 
    (2) When the to-be-calibrated semantic segmentation network is underconfident, minimizing NLL w.r.t. TS, IBTS, and LTS results in solutions that are also the solutions of minimizing entropy of the calibrated probability w.r.t. TS, IBTS and LTS under the condition of underconfidence. 
    (3) The post-hoc probability calibration of semantic segmentation with TS, IBTS and LTS approaches reach an equilibrium between Negative Log Likelihood (NLL) and entropy for both underconfidence and overconfidence.\\ 
\end{theorem}

\begin{proof}
For TS, if overconfident, we have the following relationship from definition~\ref{def:overfitting_def}:
\begin{equation}
\label{eq:TS_overfitting_exp}
    \sum_{i=1}^n \sum_{x \in \Omega} \textbf{z}_i(x)^{(S_i(x))} \leq \sum_{i=1}^n \sum_{x \in \Omega} \sum_{l=1}^L \textbf{z}_i(x)^{(l)} \sigma_{SM}\big(\textbf{z}_i(x)\big)^{(l)} \,.
\end{equation}
To eliminate overconfidence, we need to decrease NLL and increase entropy to probabilistically describe empirically observable segmentation errors (see \cref{sec:why_cross_entropy_and_temperature_scaling} for detailed explanations).
From Eq.~\eqref{eq:TS_overfitting_exp}, Theorem~\ref{thm:entropy_temperature} (or theorem~\ref{thm:overderfitting_entropy_second_proof}) and Theorem~\ref{thm:NLL_opt} we know there is a unique optimal $\alpha^*$
\begin{equation}
\begin{cases}
\;\; \alpha^*=0, \quad\quad\quad\quad \text{if}\quad \sum_{i=1}^{n} \sum_{x \in \Omega} \textbf{z}_i(x)^{(S_i(x))} \leq \frac{1}{L}\sum_{i=1}^n\sum_{x \in \Omega}\sum_{l=1}^L \textbf{z}_i(x)^{(l)}\\ 
\Big\{0<\alpha^*\leq 1  \mid \sum_{i=1}^n\sum_{x \in \Omega}\sum_{l=1}^L \textbf{z}_i(x)^{(l)}\sigma_{SM}\big( \alpha^* \textbf{z}_i(x) \big)^{(l)} = \sum_{i=1}^{n} \sum_{x \in \Omega} \textbf{z}_i(x)^{(S_i(x))} \Big\}, \text{otherwise}
\end{cases}
\end{equation}
that drives
the NLL to minimum point and the entropy to maximum point simultaneously. Besides, at the optimal point, NLL equals to entropy, thus reaching an equilibrium. And the overconfidence state is transferred to a balanced state
\begin{equation}
\begin{cases}
-\sum_{i=1}^n \sum_{x \in \Omega} \sum_{l=1}^L \frac{1}{L} \log \Big(\frac{1}{L}\Big) =
    - \sum_{i=1}^n \sum_{x \in \Omega} \log \Big(\frac{1}{L}\Big), \text{if} \sum_{i=1}^{n} \sum_{x \in \Omega} \textbf{z}_i(x)^{(S_i(x))} \leq \frac{1}{L}\sum_{i=1}^n\sum_{x \in \Omega}\sum_{l=1}^L \textbf{z}_i(x)^{(l)}\\
\sum_{i=1}^n\sum_{x \in \Omega}\sum_{l=1}^L \textbf{z}_i(x)^{(l)}\sigma_{SM}\big( \alpha^* \textbf{z}_i(x) \big)^{(l)} = \sum_{i=1}^{n} \sum_{x \in \Omega} \textbf{z}_i(x)^{(S_i(x))}, \text{otherwise}.
\end{cases}
\end{equation}

\noindent
If underconfident, we have the following relationship from definition~\ref{def:underfitting_def}:
\begin{equation}
\label{eq:TS_underfitting_exp}
    \sum_{i=1}^n \sum_{x \in \Omega} \textbf{z}_i(x)^{(S_i(x))} \geq \sum_{i=1}^n \sum_{x \in \Omega} \sum_{l=1}^L \textbf{z}_i(x)^{(l)} \sigma_{SM}\big(\textbf{z}_i(x)\big)^{(l)} \,.
\end{equation}
To eliminate underconfidence, we need to decrease NLL and decrease entropy to probabilistically describe empirically observable segmentation errors.
From Eq.~\eqref{eq:TS_underfitting_exp}, Theorem~\ref{thm:underfitting_entropy} and Theorem~\ref{thm:NLL_opt} we know there is a unique optimal $\alpha^*$
\begin{equation}
    \Big\{\alpha^*\geq 1  \mid \sum_{i=1}^n\sum_{x \in \Omega}\sum_{l=1}^L \textbf{z}_i(x)^{(l)}\sigma_{SM}\big( \alpha^* \textbf{z}_i(x) \big)^{(l)} = \sum_{i=1}^{n} \sum_{x \in \Omega} \textbf{z}_i(x)^{(S_i(x))} \Big\}
\end{equation}
that drives
the NLL to minimum point and the entropy to minimum point simultanously. Besides, at the optimal point, NLL equals to entropy, thus reaching an equilibrium. And the underconfidence state is transferred to a balanced state
\begin{equation}
    \sum_{i=1}^n\sum_{x \in \Omega}\sum_{l=1}^L \textbf{z}_i(x)^{(l)}\sigma_{SM}\big( \alpha^* \textbf{z}_i(x) \big)^{(l)} = \sum_{i=1}^{n} \sum_{x \in \Omega} \textbf{z}_i(x)^{(S_i(x))}
\end{equation}
Overall, TS post-hoc probability calibration makes NLL and entropy reach an equilibrium for the validation dataset under both the underconfidence and overconfidence scenarios.\\

\noindent
Similarly, IBTS and LTS post-hoc probability calibrations also make NLL and entropy reach an equilibrium for each image and for each location respectively under both the underconfident and overconfident scenarios.\\
\end{proof}

\section{Evaluation Metrics for Semantic Segmentation}
\label{app:metrics}
\noindent
This section introduces evaluation metrics for calibration and segmentation.\\

\noindent
\textbf{Reliability Diagram.}
Reliability diagrams are commonly used as visual representations of calibration performance~\cite{degroot1983comparison,murphy1977reliability,niculescu2005predicting}. A reliability diagram is derived from the definition of perfect calibration where the accuracy and the confidence are presented separately. If a model is perfectly calibrated, then the diagram should indicate an identity relationship between the confidence and the accuracy. Otherwise, there is miscalibration in the model. See Fig.~\ref{fig:local_rd} and Fig.~\ref{fig:camvid_local_rd} for examples.\\

\noindent
To visually illustrate the relationship of the confidence and the accuracy in Eq.~\eqref{perfect_calibration}, one can estimate both the confidence and the accuracy from finite samples. Specifically, semantic segmentation results can be grouped into $N$ equal-sized probability intervals (each of size $1/N$) to calculate the accuracy of each bin. Let $\Omega_j$ be the set of pixels/voxels whose predicted probabilities fall into the interval $\Delta_j = (\frac{j-1}{N}, \frac{j}{N}]$. Thus, the \textit{accuracy}~\cite{guo2017calibration} of $\Omega_j$ can be estimated as
\begin{equation}
acc(\Omega_j) = \frac{1}{|\Omega_j|} \sum_{x \in \Omega_j} \mathds{1} (\hat{S}(x) = S(x)),
\end{equation}
where $\hat{S}(x)$ and $S(x)$ are the predicted and true labels for pixel/voxel $x$, $\mathds{1}$ is the indicator function. Note that $acc(\Omega_j)$ is an unbiased and consistent estimator of $\mathbb{P} (\hat{S} = S | \hat{P} \in \Delta_j)$~\cite{guo2017calibration} where $\hat{P}(x)$ is the probability associated with $\hat{S}(x)$ for pixel/voxel at location $x$. The \textit{average confidence}~\cite{guo2017calibration} over bin $\Omega_j$ can be defined as
\begin{equation}
conf(\Omega_j) = \frac{1}{|\Omega_j|} \sum_{x \in \Omega_j} \hat{P}(x),
\end{equation}
Thus, $acc(\Omega_j)$ and $conf(\Omega_j)$ approximate the left-hand side and right-hand side of Eq.~\eqref{perfect_calibration} for bin $\Omega_j$.\\

\noindent
Based on the definition of perfect calibration, a reliability diagram checks whether $acc(\Omega_j) = conf(\Omega_j)$ for all $j \in {1, 2, ..., N}$ and plots the quantitative relation in a bar chart.\\

\noindent
\textbf{Expected Calibration Error (ECE).}
A reliability diagram is only a visual cue to indicate the performance of model calibration: it does not reflect the number of pixels/voxels in each bin. Thus, to account for such variations of the number of samples in a bin, it has been suggested~\cite{naeini2015obtaining} to use a scalar value to summarize the overall calibration performance. The expected calibration error~\cite{naeini2015obtaining} uses the expectation between confidence and the accuracy to indicate the magnitude of the miscalibration. More precisely,
\begin{equation}
ECE = \sum_{j=1}^N \frac{|\Omega_j|}{\Omega_*} | acc(\Omega_j) - conf(\Omega_j) |,
\end{equation}
where $\Omega_* = \sum_j^N |\Omega_j|$ is the total number of pixels/voxels. The difference between $acc$ and $conf$ for a given bin represents the calibration gap.\\

\noindent
\textbf{Maximum Calibration Error (MCE).}
The maximum calibration error~\cite{naeini2015obtaining} measures the worst-case deviation between the confidence and the accuracy. This is extremely important in high-risk applications where reliable confidence prediction is crucial for decision making. Specifically,
\begin{equation}
MCE = \max_{j \in \{1,...,N\}} |acc(\Omega_j) - conf(\Omega_j)|\,.
\end{equation}
Note that both the ECE and the MCE are closely related to the reliability diagram. The ECE is a weighted average of all gaps across all bins while the MCE is the largest gap.\\

\noindent
\textbf{Static Calibration Error (SCE).}
The ECE is computed by only using the predicted label's probability, which does not consider information obtained for other labels. The static calibration error (SCE)~\cite{nixon2019measuring} has therefore been proposed for the multi-label setting, which extends ECE by separately computing the calibration error within a bin for each label followed by averaging across all bins.
More precisely, the SCE is defined as
\begin{equation}
SCE = \sum_{l \in L}  \sum_{j=1}^N  \frac{|\Omega_{j, l}|}{|L| \Omega_*} |acc(\Omega_{j, l}) - conf(\Omega_{j, l})|,
\end{equation}
where $L$ is the set of labels, $\Omega_{j, l}$ is the subset of pixels/voxels for label $l$ in bin $\Omega_j$.\\

\noindent
\textbf{Adaptive Calibration Error (ACE).}
Another weakness of ECE is that the number of pixels/voxels in each bin varies a lot among different bins, posing a bias-variance tradeoff for choosing the number of bins~\cite{nixon2019measuring}.
This motivates the introduction of the adaptive calibration error (ACE)~\cite{nixon2019measuring}. Specifically, ACE uses an adaptive scheme which separates the bin intervals so that each bin contains an equal number of pixels/voxels. Specifically, 
\begin{equation}
ACE = \sum_{l \in L} \sum_{r=1}^R \frac{1}{|L|R} |acc(\Omega_{r, l}) - conf(\Omega_{r, l})|,
\end{equation}
where $R$ is the number of equal-frequency bins, $\Omega_r$ is the $r$-th sorted bin which contains $\Omega_* / R$ pixels/voxels. $\Omega_{r, l}$ is the subset of pixels/voxels for label $l$ in the $r$-th bin $\Omega_r$.\\

\noindent
\textbf{Avgerage Surface Distance (ASD).} ASD is the symmetric average surface distance (usually in millimeter (mm)) between each predicted segmentation label and the true segmentation label. The distance between a point $p$ on a gold-standard or ground-truth surface $\partial S^{(l)}$ and the predicted surface $\partial \hat{S}^{(l)}$ with respect to label $l$ is given by the minimum of the Euclidean norm, i.e. $d(p, \partial \hat{S}^{(l)}) = \min_{\hat{p} \in \partial \hat{S}^{(l)}} ||p - \hat{p}||_2$, where $\hat{p}$ is a point on surface $\partial \hat{S}^{(l)}$. Hence symmetric average surface distance is defined as
\begin{equation}
    ASD = \frac{1}{|L|} \sum_{l \in L} \Bigg( \frac{1}{|\partial S^{(l)}| + |\partial \hat{S}^{(l)}|} \bigg ( \sum_{p \in \partial S^{(l)}} d(p, \partial \hat{S}^{(l)}) + \sum_{\hat{p} \in \partial \hat{S}^{(l)}} d(\hat{p}, \partial S^{(l)}) \bigg ) \Bigg ).
\end{equation}

\noindent
\textbf{Surface Dice (SD).} SD is the averaged Dice score between the segmented label surface and the true label surface at a given tolerance (we use 1 mm). This tolerance captures that a point $p$ may still be counted as being on the surface $\partial \hat{S}^{(l)}$ if the distance is at or below the tolerance, i.e. $d(p, \partial \hat{S}^{(l)}) \leq \text{tolerance}$. Formally, the averaged surface Dice score is defined as

\begin{equation}
    SD = \frac{1}{|L|} \sum_{l \in L}\frac{2|\{p | d(p, \partial S^{(l)}) \leq \epsilon, d(p, \partial \hat{S}^{(l)}) \leq \epsilon\}|}{|\{p | d(p, \partial S^{(l)}) \leq \epsilon\}| + |\{p | d(p, \partial \hat{S}^{(l)}) \leq \epsilon\}|}\,,
\end{equation}
where $\epsilon$ is the tolerance threshold, and $|\cdot|$ is the Cardinality of the set.

\noindent
\textbf{95\% Maximum Distance ($\textbf{95MD}$).} $\text{95MD}$ is the 95th percentile of the symmetric distance between the segmented label volume and the true label volume. The definition is
\begin{equation}
    95MD = \frac{1}{|L|} \sum_{l \in L} \bigg (95\% \text{Percentile} \Big\{...,d(p, \hat{S}^{(l)}), ..., d(\hat{p}, S^{(l)}), ...\Big\} \quad \forall p \in S^{(l)}, \hat{p} \in \hat{S}^{(l)} \bigg).
\end{equation}

\noindent
\textbf{Volume Dice (VD).} VD is the average Dice score over segmented labels (excluding the background). This is a commonly used metric to determine the success of segmentation in the field of medical image analysis. It is defined as
\begin{equation}
    VD = \frac{1}{|L|} \sum_{l \in L}\frac{2|S^{(l)} \cap \hat{S}^{(l)}|}{|S^{(l)}| + |\hat{S}^{(l)}|}.
\end{equation}

\section{Example of \textit{Boundary} Region and \textit{All} Region}
\label{app:region}
\noindent
Fig.~\ref{fig:boundary} shows an example of the \textit{Boundary} region and the \textit{All} region for a 2D slice of a 3D MR brain image. The \textit{Boundary} region is created with boundaries of labels and voxels that are up to 2 voxels away from boundary voxels. The \textit{All} region contains label regions excluding the background and the \textit{Boundary} region. Note that in the multi-atlas segmentation label fusion experiment, the boundary region of the VoteNet+ ground-truth labels is very sparse and thin. Thus, we use the \textit{Boundary} region and the \textit{All} region of the original segmentation labels of the magnetic resonance (MR) images instead. This is the same evaluation approach as for the U-Net segmentation experiment.

\begin{figure}[ht]
  \includegraphics[width=\linewidth]{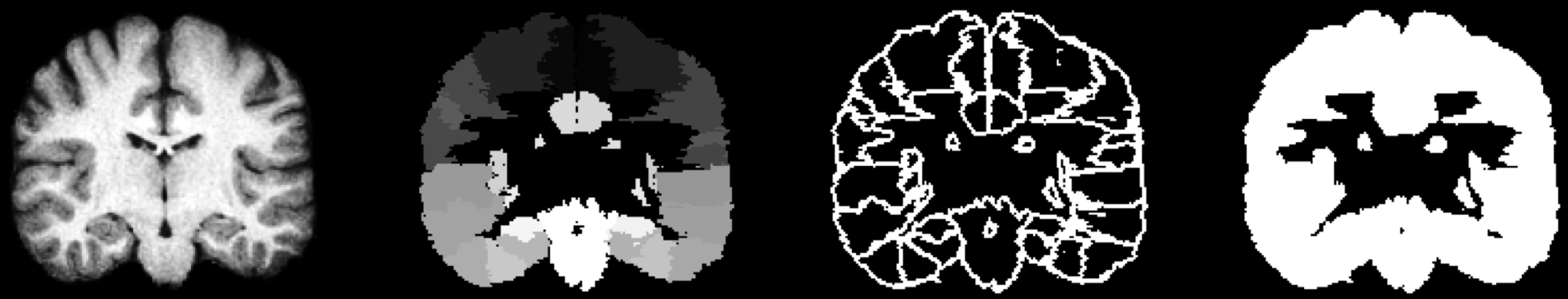}
  \caption{Illustration of \textit{Boundary} region and \textit{All} region of an MR brain image from the LPBA40 dataset in 2D. Left two columns: image and corresponding label map. Right two columns: \textit{Boundary} region and \textit{All} region. The \textit{Boundary} region is usually where mis-segmentations and mis-calibrations occur. The \textit{All} region enlarges the label region to include the \textit{Boundray} region, it thus captures an evaluation region which excludes almost all background of an image.}
  \label{fig:boundary}
\end{figure}

\section{Patch Size vs Metrics Results}
\label{app:patch_size_metrcis}
\noindent
Fig.~\ref{fig:patch_size_difference} shows the results of \textit{Local-Avg} for different metrics with different patch sizes. Note that the \textit{Local-Avg} and \textit{Local-Max} results reported in Tab.~\ref{tab:All_metrics} are for a patch size of 72$\times$72 (or 72$\times$72$\times$72 in 3D). We observe that the probability calibration performance tends to be worse for smaller patch sizes. This is expected as patch variations (also the differences of patch-based multi-class probability distributions) are very significant across patches when patch sizes are small. LTS can improve the calibration performance over TS and IBTS, because it can capture spatially varying effects.

\begin{figure*}[ht]
\centering
  \includegraphics[width=\textwidth]{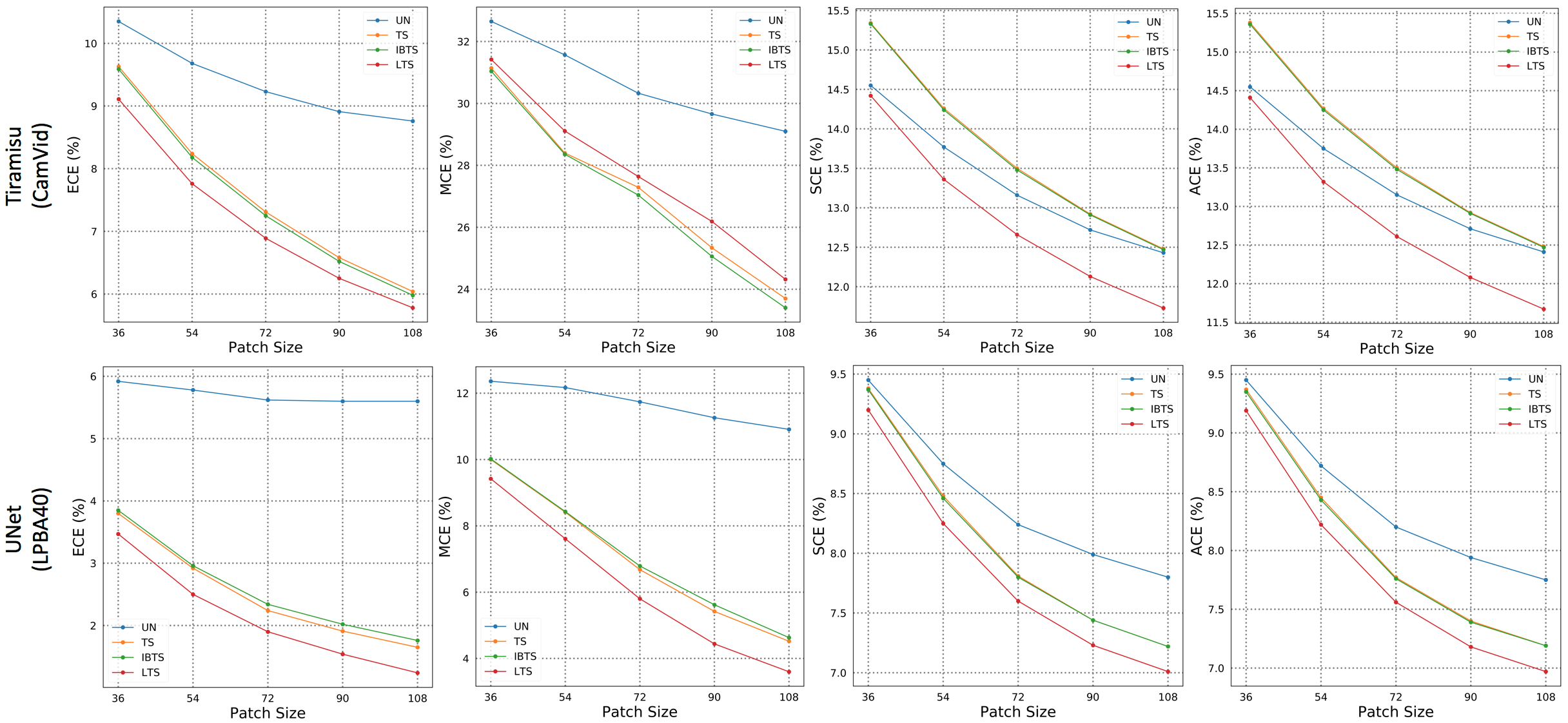}
  \caption{\textit{Local-Avg} results LPBA40 and CamVid experiments for different patch sizes.
  UN denotes uncalibrated results. In general, the smaller the patch size the worse the performance. Besides, LTS works best for most metrics.}
  \label{fig:patch_size_difference}
\end{figure*}

\section{Dataset Variations}
\label{app:dataset_variation}

\noindent
Image variations are different for different datasets. Fig.~\ref{fig:dataset_variation} illustrates such variations. COCO using an FCN is the most complex dataset, followed by CamVid using Tiramisu, LPBA40 using a UNet and finally LPBA40 combined with VoteNet+. The quantitative results of the metrics in Tab.~\ref{tab:All_metrics} follows the same pattern: with the results for COCO using an FCN the weakest and the results for LPBA40 using VoteNet+ the best.

\begin{figure*}[ht]
\centering
  \includegraphics[width=0.6\textwidth]{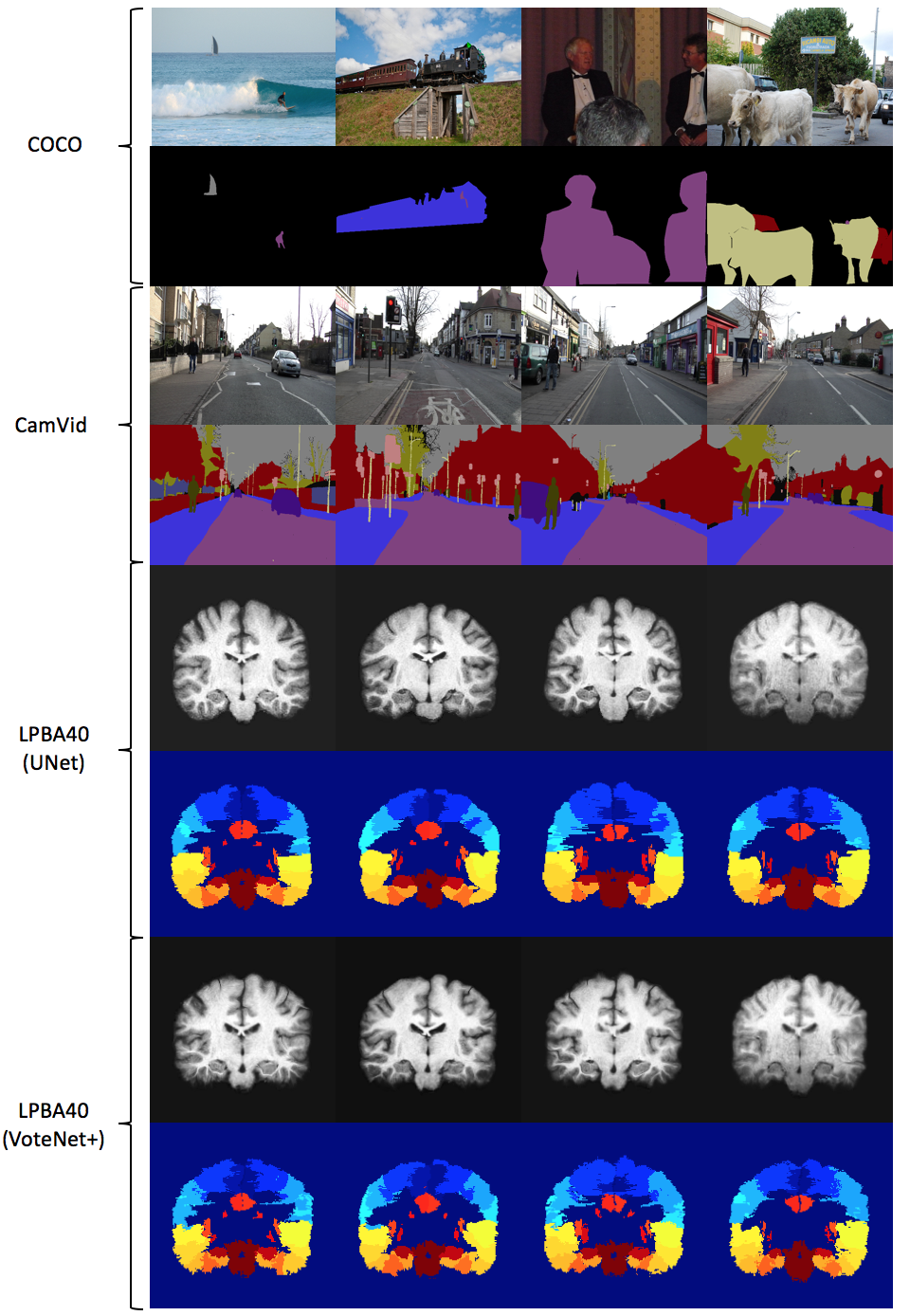}
  \caption{An example of images and labels in different datasets for different experiments. COCO is the most complex dataset and contains different kinds of natural images. CamVid is mainly focused on street scenes. LPBA40 is a dataset of 3D brain MR images. Note that images for UNet are affine pre-registered to a common atlas space while images for VoteNet+ are registered to a target image via a deformable registration. Thus image variations of VoteNet+ experiment are less than that for the UNet experiment.}
  \label{fig:dataset_variation}
\end{figure*}

\section{Additional Quantitative Results}
\label{app:additional_results}
\noindent
Additional quantitative results are provided in Tab.~\ref{tab:additional_metrics}. The results are in line with the conclusions we obtain in \cref{sec:experiments}, i.e. LTS works significantly better than TS~\cite{guo2017calibration}, isotonic Regression (IsoReg)~\cite{zadrozny2002transforming}, ensemble temperature scaling (ETS)~\cite{zhang2020mix}, vector scaling (VS)~\cite{guo2017calibration}, and Dirichlet calibration with off-diagonal regularization (DirODIR)~\cite{kull2019beyond}.

\begin{table*}[!ht] 
\centering
\begin{adjustbox}{max width=\textwidth}
\begin{tabular}{cccccccccccccc}
\specialrule{.15em}{.05em}{.05em}
\multicolumn{1}{c}{\multirow{3}{*}{Dataset}} &
\multicolumn{1}{c}{\multirow{3}{*}{Method}} & \multicolumn{3}{c}{ECE(\%)$\downarrow$} & \multicolumn{3}{c}{MCE(\%)$\downarrow$} & \multicolumn{3}{c}{SCE(\%)$\downarrow$} & \multicolumn{3}{c}{ACE(\%)$\downarrow$}\\
\cmidrule(lr){3-5}
\cmidrule(lr){6-8}
\cmidrule(lr){9-11}
\cmidrule(lr){12-14}
\multicolumn{1}{r}{}& \multicolumn{1}{r}{}& \multicolumn{1}{c}{\multirow{2}{*}{\textit{All}}} & \multicolumn{1}{c}{\multirow{2}{*}{\textit{Boundary}}} & \textit{Local-Avg} & \multicolumn{1}{c}{\multirow{2}{*}{\textit{All}}} & \multicolumn{1}{c}{\multirow{2}{*}{\textit{Boundary}}} & \textit{Local-Avg} & \multicolumn{1}{c}{\multirow{2}{*}{\textit{All}}} & \multicolumn{1}{c}{\multirow{2}{*}{\textit{Boundary}}} & \textit{Local-Avg} & \multicolumn{1}{c}{\multirow{2}{*}{\textit{All}}} & \multicolumn{1}{c}{\multirow{2}{*}{\textit{Boundary}}} & \textit{Local-Avg} \\
\multicolumn{1}{r}{} & \multicolumn{1}{r}{} & \multicolumn{1}{r}{} & \multicolumn{1}{r}{} & [\textit{Local-Max}] & \multicolumn{1}{r}{} & \multicolumn{1}{r}{} & [\textit{Local-Max}] & \multicolumn{1}{r}{} & \multicolumn{1}{r}{} & [\textit{Local-Max}] & \multicolumn{1}{r}{} & \multicolumn{1}{r}{} & [\textit{Local-Max}] \\
\specialrule{.15em}{.05em}{.05em}
\multicolumn{1}{c}{\multirow{24}{*}{\shortstack{Tiramisu\\CamVid\\(233)}}}&\multicolumn{1}{c}{\multirow{2}{*}{UC}}&\multicolumn{1}{c}{\multirow{2}{*}{{\setlength{\fboxsep}{0pt}\colorbox{green!30}{7.79(4.94)}}}}&\multicolumn{1}{c}{\multirow{2}{*}{{\setlength{\fboxsep}{0pt}\colorbox{green!30}{22.79(5.76)}}}}&{\setlength{\fboxsep}{0pt}\colorbox{green!30}{9.23(10.63)}}&\multicolumn{1}{c}{\multirow{2}{*}{{\setlength{\fboxsep}{0pt}\colorbox{green!30}{22.64(12.72)}}}}&\multicolumn{1}{c}{\multirow{2}{*}{{\setlength{\fboxsep}{0pt}\colorbox{green!30}{30.42(10.65)}}}}&{\setlength{\fboxsep}{0pt}\colorbox{green!30}{30.33(16.63)}}&\multicolumn{1}{c}{\multirow{2}{*}{{\setlength{\fboxsep}{0pt}\colorbox{green!30}{9.91(5.02)}}}}&\multicolumn{1}{c}{\multirow{2}{*}{{\setlength{\fboxsep}{0pt}\colorbox{green!30}{24.62(5.69)}}}}&{\setlength{\fboxsep}{0pt}\colorbox{green!30}{13.16(11.72)}}&\multicolumn{1}{c}{\multirow{2}{*}{{\setlength{\fboxsep}{0pt}\colorbox{green!30}{9.90(5.01)}}}}&\multicolumn{1}{c}{\multirow{2}{*}{{\setlength{\fboxsep}{0pt}\colorbox{green!30}{24.43(5.75)}}}}&{\setlength{\fboxsep}{0pt}\colorbox{green!30}{13.15(11.73)}}\\
\multicolumn{1}{r}{}& \multicolumn{1}{r}{}& \multicolumn{1}{r}{} & \multicolumn{1}{r}{} & {\setlength{\fboxsep}{0pt}\colorbox{green!30}{[25.35(12.80)]}}  & \multicolumn{1}{r}{} & \multicolumn{1}{r}{} & {\setlength{\fboxsep}{0pt}\colorbox{green!30}{[56.15(14.61)]}}  & \multicolumn{1}{r}{} & \multicolumn{1}{r}{} & {\setlength{\fboxsep}{0pt}\colorbox{green!30}{[30.60(12.48)]}}  & \multicolumn{1}{r}{} & \multicolumn{1}{r}{} & {\setlength{\fboxsep}{0pt}\colorbox{green!30}{[30.60(12.46)]}} \\
\cmidrule(lr){2-14}
&\multicolumn{1}{c}{\multirow{2}{*}{IsoReg~\cite{zadrozny2002transforming}}}&\multicolumn{1}{c}{\multirow{2}{*}{3.77(3.71)}}&\multicolumn{1}{c}{\multirow{2}{*}{{\setlength{\fboxsep}{0pt}\colorbox{green!30}{16.86(5.99)}}}}&\setlength{\fboxsep}{0pt}\colorbox{green!30}{7.79(8.56)}&\multicolumn{1}{c}{\multirow{2}{*}{\setlength{\fboxsep}{0pt}\colorbox{green!30}{18.19(11.70)}}}&\multicolumn{1}{c}{\multirow{2}{*}{\setlength{\fboxsep}{0pt}\colorbox{green!30}{24.59(10.00)}}}&27.66(15.89)&\multicolumn{1}{c}{\multirow{2}{*}{\setlength{\fboxsep}{0pt}\colorbox{green!30}{9.91(3.86)}}}&\multicolumn{1}{c}{\multirow{2}{*}{\setlength{\fboxsep}{0pt}\colorbox{green!30}{19.89(5.65)}}}&\setlength{\fboxsep}{0pt}\colorbox{green!30}{13.94(10.71)}&\multicolumn{1}{c}{\multirow{2}{*}{\setlength{\fboxsep}{0pt}\colorbox{green!30}{10.07(3.85)}}}&\multicolumn{1}{c}{\multirow{2}{*}{\setlength{\fboxsep}{0pt}\colorbox{green!30}{19.72(5.70)}}}&\setlength{\fboxsep}{0pt}\colorbox{green!30}{14.08(10.74)}\\
\multicolumn{1}{r}{}& \multicolumn{1}{r}{}& \multicolumn{1}{r}{} & \multicolumn{1}{r}{} & \setlength{\fboxsep}{0pt}\colorbox{green!30}{[21.18(12.73)]}  & \multicolumn{1}{r}{} & \multicolumn{1}{r}{} & \setlength{\fboxsep}{0pt}\colorbox{green!30}{[40.66(20.14)]}  & \multicolumn{1}{r}{} & \multicolumn{1}{r}{} & \setlength{\fboxsep}{0pt}\colorbox{green!30}{[29.79(12.51)]}  & \multicolumn{1}{r}{} & \multicolumn{1}{r}{} & \setlength{\fboxsep}{0pt}\colorbox{green!30}{[29.92(12.45)]}  \\
\cmidrule(lr){2-14}
&\multicolumn{1}{c}{\multirow{2}{*}{VS~\cite{guo2017calibration}}}&\multicolumn{1}{c}{\multirow{2}{*}{\setlength{\fboxsep}{0pt}\colorbox{green!30}{5.85(4.27)}}}&\multicolumn{1}{c}{\multirow{2}{*}{\setlength{\fboxsep}{0pt}\colorbox{green!30}{17.95(6.46)}}}&\setlength{\fboxsep}{0pt}\colorbox{green!30}{11.24(11.11)}&\multicolumn{1}{c}{\multirow{2}{*}{\setlength{\fboxsep}{0pt}\colorbox{green!30}{21.14(8.44)}}}&\multicolumn{1}{c}{\multirow{2}{*}{\setlength{\fboxsep}{0pt}\colorbox{green!30}{32.25(12.68)}}}&\setlength{\fboxsep}{0pt}\colorbox{green!30}{38.47(18.10)}&\multicolumn{1}{c}{\multirow{2}{*}{\setlength{\fboxsep}{0pt}\colorbox{green!30}{10.84(5.56)}}}&\multicolumn{1}{c}{\multirow{2}{*}{\setlength{\fboxsep}{0pt}\colorbox{green!30}{22.84(5.62)}}}&\setlength{\fboxsep}{0pt}\colorbox{green!30}{14.90(12.59)}&\multicolumn{1}{c}{\multirow{2}{*}{\setlength{\fboxsep}{0pt}\colorbox{green!30}{10.80(5.55)}}}&\multicolumn{1}{c}{\multirow{2}{*}{\setlength{\fboxsep}{0pt}\colorbox{green!30}{22.39(5.73)}}}&\setlength{\fboxsep}{0pt}\colorbox{green!30}{14.83(12.62)}\\
\multicolumn{1}{r}{}& \multicolumn{1}{r}{}& \multicolumn{1}{r}{} & \multicolumn{1}{r}{} & \setlength{\fboxsep}{0pt}\colorbox{green!30}{[24.97(14.50)]}  & \multicolumn{1}{r}{} & \multicolumn{1}{r}{} & \setlength{\fboxsep}{0pt}\colorbox{green!30}{[44.92(19.20)]}  & \multicolumn{1}{r}{} & \multicolumn{1}{r}{} & \setlength{\fboxsep}{0pt}\colorbox{green!30}{[31.13(14.99)]}  & \multicolumn{1}{r}{} & \multicolumn{1}{r}{} & \setlength{\fboxsep}{0pt}\colorbox{green!30}{[31.01(14.95)]}  \\
\cmidrule(lr){2-14}
&\multicolumn{1}{c}{\multirow{2}{*}{ETS~\cite{zhang2020mix}}}&\multicolumn{1}{c}{\multirow{2}{*}{3.71(3.65)}}&\multicolumn{1}{c}{\multirow{2}{*}{\setlength{\fboxsep}{0pt}\colorbox{green!30}{16.28(6.08)}}}&\setlength{\fboxsep}{0pt}\colorbox{green!30}{7.76(8.46)}&\multicolumn{1}{c}{\multirow{2}{*}{\setlength{\fboxsep}{0pt}\colorbox{green!30}{17.63(10.33)}}}&\multicolumn{1}{c}{\multirow{2}{*}{\setlength{\fboxsep}{0pt}\colorbox{green!30}{23.06(9.25)}}}&27.63(15.94)&\multicolumn{1}{c}{\multirow{2}{*}{\setlength{\fboxsep}{0pt}\colorbox{green!30}{9.98(3.85)}}}&\multicolumn{1}{c}{\multirow{2}{*}{\setlength{\fboxsep}{0pt}\colorbox{green!30}{19.48(5.62)}}}&\setlength{\fboxsep}{0pt}\colorbox{green!30}{14.05(10.70)}&\multicolumn{1}{c}{\multirow{2}{*}{\setlength{\fboxsep}{0pt}\colorbox{green!30}{10.12(3.84)}}}&\multicolumn{1}{c}{\multirow{2}{*}{\setlength{\fboxsep}{0pt}\colorbox{green!30}{19.30(5.67)}}}&\setlength{\fboxsep}{0pt}\colorbox{green!30}{14.14(10.72)}\\
\multicolumn{1}{r}{}& \multicolumn{1}{r}{}& \multicolumn{1}{r}{} & \multicolumn{1}{r}{} & \setlength{\fboxsep}{0pt}\colorbox{green!30}{[20.86(12.73)]}  & \multicolumn{1}{r}{} & \multicolumn{1}{r}{} & \setlength{\fboxsep}{0pt}\colorbox{green!30}{[41.09(20.13)]}  & \multicolumn{1}{r}{} & \multicolumn{1}{r}{} & \setlength{\fboxsep}{0pt}\colorbox{green!30}{[29.78(12.46)]}  & \multicolumn{1}{r}{} & \multicolumn{1}{r}{} & \setlength{\fboxsep}{0pt}\colorbox{green!30}{[29.85(12.42)]}  \\
\cmidrule(lr){2-14}
&\multicolumn{1}{c}{\multirow{2}{*}{DirODIR~\cite{kull2019beyond}}}&\multicolumn{1}{c}{\multirow{2}{*}{\setlength{\fboxsep}{0pt}\colorbox{green!30}{6.63(5.51)}}}&\multicolumn{1}{c}{\multirow{2}{*}{\setlength{\fboxsep}{0pt}\colorbox{green!30}{25.32(8.14)}}}&\setlength{\fboxsep}{0pt}\colorbox{green!30}{11.79(13.66)}&\multicolumn{1}{c}{\multirow{2}{*}{\setlength{\fboxsep}{0pt}\colorbox{green!30}{15.77(8.27)}}}&\multicolumn{1}{c}{\multirow{2}{*}{\setlength{\fboxsep}{0pt}\colorbox{green!30}{34.92(11.45)}}}&\setlength{\fboxsep}{0pt}\colorbox{green!30}{33.54(19.77)}&\multicolumn{1}{c}{\multirow{2}{*}{\setlength{\fboxsep}{0pt}\colorbox{green!30}{12.42(7.33)}}}&\multicolumn{1}{c}{\multirow{2}{*}{\setlength{\fboxsep}{0pt}\colorbox{green!30}{29.01(7.26)}}}&\setlength{\fboxsep}{0pt}\colorbox{green!30}{17.33(16.00)}&\multicolumn{1}{c}{\multirow{2}{*}{\setlength{\fboxsep}{0pt}\colorbox{green!30}{12.37(7.34)}}}&\multicolumn{1}{c}{\multirow{2}{*}{\setlength{\fboxsep}{0pt}\colorbox{green!30}{28.84(7.33)}}}&\setlength{\fboxsep}{0pt}\colorbox{green!30}{17.32(16.00)}\\
\multicolumn{1}{r}{}& \multicolumn{1}{r}{}& \multicolumn{1}{r}{} & \multicolumn{1}{r}{} & \setlength{\fboxsep}{0pt}\colorbox{green!30}{[25.01(16.57)]}  & \multicolumn{1}{r}{} & \multicolumn{1}{r}{} & \setlength{\fboxsep}{0pt}\colorbox{green!30}{[43.56(22.37)]}  & \multicolumn{1}{r}{} & \multicolumn{1}{r}{} & \setlength{\fboxsep}{0pt}\colorbox{green!30}{[32.75(18.49)]}  & \multicolumn{1}{r}{} & \multicolumn{1}{r}{} & \setlength{\fboxsep}{0pt}\colorbox{green!30}{[32.66(18.42)]}  \\
\cmidrule(lr){2-14}
&\multicolumn{1}{c}{\multirow{2}{*}{TS~\cite{guo2017calibration}}}&\multicolumn{1}{c}{\multirow{2}{*}{3.45(3.52)}}&\multicolumn{1}{c}{\multirow{2}{*}{{\setlength{\fboxsep}{0pt}\colorbox{green!30}{12.66(5.43)}}}}&{\setlength{\fboxsep}{0pt}\colorbox{green!30}{7.31(7.72)}}&\multicolumn{1}{c}{\multirow{2}{*}{{\setlength{\fboxsep}{0pt}\colorbox{green!30}{16.02(11.09)}}}}&\multicolumn{1}{c}{\multirow{2}{*}{\setlength{\fboxsep}{0pt}\colorbox{green!30}{23.57(12.88)}}}&27.29(16.23)&\multicolumn{1}{c}{\multirow{2}{*}{{\setlength{\fboxsep}{0pt}\colorbox{green!30}{9.42(3.90)}}}}&\multicolumn{1}{c}{\multirow{2}{*}{\setlength{\fboxsep}{0pt}\colorbox{green!30}{17.85(4.55)}}}&{\setlength{\fboxsep}{0pt}\colorbox{green!30}{13.50(10.14)}}&\multicolumn{1}{c}{\multirow{2}{*}{{\setlength{\fboxsep}{0pt}\colorbox{green!30}{9.44(3.92)}}}}&\multicolumn{1}{c}{\multirow{2}{*}{\setlength{\fboxsep}{0pt}\colorbox{green!30}{17.61(4.59)}}}&{\setlength{\fboxsep}{0pt}\colorbox{green!30}{13.50(10.17)}}\\
\multicolumn{1}{r}{}& \multicolumn{1}{r}{}& \multicolumn{1}{r}{} & \multicolumn{1}{r}{} & {\setlength{\fboxsep}{0pt}\colorbox{green!30}{[17.69(11.91)]}}  & \multicolumn{1}{r}{} & \multicolumn{1}{r}{} & \setlength{\fboxsep}{0pt}\colorbox{green!30}{[37.25(18.98)]}  & \multicolumn{1}{r}{} & \multicolumn{1}{r}{} & \setlength{\fboxsep}{0pt}\colorbox{green!30}{[27.72(11.37)]}  & \multicolumn{1}{r}{} & \multicolumn{1}{r}{} & \setlength{\fboxsep}{0pt}\colorbox{green!30}{[27.76(11.33)]}  \\
\cmidrule(lr){2-14}
&\multicolumn{1}{c}{\multirow{2}{*}{IBTS}}&\multicolumn{1}{c}{\multirow{2}{*}{3.63(3.65)}}&\multicolumn{1}{c}{\multirow{2}{*}{{\setlength{\fboxsep}{0pt}\colorbox{green!30}{12.57(6.07)}}}}&{\setlength{\fboxsep}{0pt}\colorbox{green!30}{7.25(7.67)}}&\multicolumn{1}{c}{\multirow{2}{*}{{\setlength{\fboxsep}{0pt}\colorbox{green!30}{16.01(10.21)}}}}&\multicolumn{1}{c}{\multirow{2}{*}{\setlength{\fboxsep}{0pt}\colorbox{green!30}{23.24(13.00)}}}&27.04(15.94)&\multicolumn{1}{c}{\multirow{2}{*}{{\setlength{\fboxsep}{0pt}\colorbox{green!30}{9.47(3.89)}}}}&\multicolumn{1}{c}{\multirow{2}{*}{\setlength{\fboxsep}{0pt}\colorbox{green!30}{17.98(4.88)}}}&{\setlength{\fboxsep}{0pt}\colorbox{green!30}{13.48(10.12)}}&\multicolumn{1}{c}{\multirow{2}{*}{{\setlength{\fboxsep}{0pt}\colorbox{green!30}{9.49(3.91)}}}}&\multicolumn{1}{c}{\multirow{2}{*}{\setlength{\fboxsep}{0pt}\colorbox{green!30}{17.75(4.92)}}}&{\setlength{\fboxsep}{0pt}\colorbox{green!30}{13.48(10.16)}}\\
\multicolumn{1}{r}{}& \multicolumn{1}{r}{}& \multicolumn{1}{r}{} & \multicolumn{1}{r}{} & {\setlength{\fboxsep}{0pt}\colorbox{green!30}{[17.60(11.91)]}}  & \multicolumn{1}{r}{} & \multicolumn{1}{r}{} & \setlength{\fboxsep}{0pt}\colorbox{green!30}{[37.61(19.27)]}  & \multicolumn{1}{r}{} & \multicolumn{1}{r}{} & \setlength{\fboxsep}{0pt}\colorbox{green!30}{[27.69(11.38)]}  & \multicolumn{1}{r}{} & \multicolumn{1}{r}{} & \setlength{\fboxsep}{0pt}\colorbox{green!30}{[27.76(11.33)]}  \\
\cmidrule(lr){2-14}
&\multicolumn{1}{c}{\multirow{2}{*}{\textbf{LTS}}}&\multicolumn{1}{c}{\multirow{2}{*}{3.40(3.59)}}&\multicolumn{1}{c}{\multirow{2}{*}{11.80(5.20)}}&\textbf{6.89(7.64)}&\multicolumn{1}{c}{\multirow{2}{*}{\textbf{12.44(7.48)}}}&\multicolumn{1}{c}{\multirow{2}{*}{\setlength{\fboxsep}{0pt}\colorbox{green!30}{22.17(9.53)}}}&27.64(16.67)&\multicolumn{1}{c}{\multirow{2}{*}{\setlength{\fboxsep}{0pt}\colorbox{green!30}{8.76(4.05)}}}&\multicolumn{1}{c}{\multirow{2}{*}{\setlength{\fboxsep}{0pt}\colorbox{green!30}{17.77(4.26)}}}&\setlength{\fboxsep}{0pt}\colorbox{green!30}{12.66(10.04)}&\multicolumn{1}{c}{\multirow{2}{*}{\setlength{\fboxsep}{0pt}\colorbox{green!30}{8.73(4.03)}}}&\multicolumn{1}{c}{\multirow{2}{*}{\setlength{\fboxsep}{0pt}\colorbox{green!30}{17.32(4.32)}}}&\setlength{\fboxsep}{0pt}\colorbox{green!30}{12.61(10.07)}\\
\multicolumn{1}{r}{}& \multicolumn{1}{r}{}& \multicolumn{1}{r}{} & \multicolumn{1}{r}{} & \setlength{\fboxsep}{0pt}\colorbox{green!30}{[16.61(11.81)]}  & \multicolumn{1}{r}{} & \multicolumn{1}{r}{} & \setlength{\fboxsep}{0pt}\colorbox{green!30}{[37.92(20.47)]}  & \multicolumn{1}{r}{} & \multicolumn{1}{r}{} & \setlength{\fboxsep}{0pt}\colorbox{green!30}{[26.78(11.22)]}  & \multicolumn{1}{r}{} & \multicolumn{1}{r}{} & \setlength{\fboxsep}{0pt}\colorbox{green!30}{[26.76(11.22)]}  \\
\cmidrule(lr){2-14}
&\multicolumn{1}{c}{\multirow{2}{*}{MMCE~\cite{kumar2018trainable}}}&\multicolumn{1}{c}{\multirow{2}{*}{4.45(4.03)}}&\multicolumn{1}{c}{\multirow{2}{*}{--}}&--&\multicolumn{1}{c}{\multirow{2}{*}{18.83(10.82)}}&\multicolumn{1}{c}{\multirow{2}{*}{--}}&--&\multicolumn{1}{c}{\multirow{2}{*}{8.59(5.98)}}&\multicolumn{1}{c}{\multirow{2}{*}{--}}&--&\multicolumn{1}{c}{\multirow{2}{*}{8.50(5.00)}}&\multicolumn{1}{c}{\multirow{2}{*}{--}}&--\\
\multicolumn{1}{r}{}& \multicolumn{1}{r}{}& \multicolumn{1}{r}{} & \multicolumn{1}{r}{} & [--]  & \multicolumn{1}{r}{} & \multicolumn{1}{r}{} & [--]  & \multicolumn{1}{r}{} & \multicolumn{1}{r}{} & [--]  & \multicolumn{1}{r}{} & \multicolumn{1}{r}{} & [--]  \\
\cmidrule(lr){2-14}
&\multicolumn{1}{c}{\multirow{2}{*}{MMCE~\cite{kumar2018trainable}+\textbf{LTS}}}&\multicolumn{1}{c}{\multirow{2}{*}{4.15(3.54)}}&\multicolumn{1}{c}{\multirow{2}{*}{--}}&--&\multicolumn{1}{c}{\multirow{2}{*}{17.98(10.69)}}&\multicolumn{1}{c}{\multirow{2}{*}{--}}&--&\multicolumn{1}{c}{\multirow{2}{*}{7.28(3.80)}}&\multicolumn{1}{c}{\multirow{2}{*}{--}}&--&\multicolumn{1}{c}{\multirow{2}{*}{7.17(3.84)}}&\multicolumn{1}{c}{\multirow{2}{*}{--}}&--\\
\multicolumn{1}{r}{}& \multicolumn{1}{r}{}& \multicolumn{1}{r}{} & \multicolumn{1}{r}{} & [--]  & \multicolumn{1}{r}{} & \multicolumn{1}{r}{} & [--]  & \multicolumn{1}{r}{} & \multicolumn{1}{r}{} & [--]  & \multicolumn{1}{r}{} & \multicolumn{1}{r}{} & [--]  \\
\cmidrule(lr){2-14}
&\multicolumn{1}{c}{\multirow{2}{*}{FL~\cite{mukhoti2020calibrating}}}&\multicolumn{1}{c}{\multirow{2}{*}{3.47(3.11)}}&\multicolumn{1}{c}{\multirow{2}{*}{\textbf{8.68(5.45)}}}&\setlength{\fboxsep}{0pt}\colorbox{green!30}{9.01(7.19)}&\multicolumn{1}{c}{\multirow{2}{*}{14.77(13.28)}}&\multicolumn{1}{c}{\multirow{2}{*}{\textbf{17.62(13.53)}}}&28.37(15.86)&\multicolumn{1}{c}{\multirow{2}{*}{\setlength{\fboxsep}{0pt}\colorbox{green!30}{7.46(3.43)}}}&\multicolumn{1}{c}{\multirow{2}{*}{\textbf{14.08(4.49)}}}&\setlength{\fboxsep}{0pt}\colorbox{green!30}{14.09(9.78)}&\multicolumn{1}{c}{\multirow{2}{*}{\setlength{\fboxsep}{0pt}\colorbox{green!30}{7.43(3.45)}}}&\multicolumn{1}{c}{\multirow{2}{*}{\textbf{13.63(4.57)}}}&\setlength{\fboxsep}{0pt}\colorbox{green!30}{14.06(9.83)}\\
\multicolumn{1}{r}{}& \multicolumn{1}{r}{}& \multicolumn{1}{r}{} & \multicolumn{1}{r}{} & \setlength{\fboxsep}{0pt}\colorbox{green!30}{[13.84(11.67)]}  & \multicolumn{1}{r}{} & \multicolumn{1}{r}{} & \setlength{\fboxsep}{0pt}\colorbox{green!30}{[33.33(18.08)]}  & \multicolumn{1}{r}{} & \multicolumn{1}{r}{} & \setlength{\fboxsep}{0pt}\colorbox{green!30}{[23.60(12.11)]}  & \multicolumn{1}{r}{} & \multicolumn{1}{r}{} & \setlength{\fboxsep}{0pt}\colorbox{green!30}{[23.62(12.05)]}  \\
\cmidrule(lr){2-14}
&\multicolumn{1}{c}{\multirow{2}{*}{FL~\cite{mukhoti2020calibrating}+\textbf{LTS}}}&\multicolumn{1}{c}{\multirow{2}{*}{\textbf{3.13(3.64)}}}&\multicolumn{1}{c}{\multirow{2}{*}{11.06(5.55)}}&6.96(8.21)&\multicolumn{1}{c}{\multirow{2}{*}{14.51(11.07)}}&\multicolumn{1}{c}{\multirow{2}{*}{19.61(9.82)}}&\textbf{26.91(16.06)}&\multicolumn{1}{c}{\multirow{2}{*}{\textbf{6.78(4.05)}}}&\multicolumn{1}{c}{\multirow{2}{*}{15.28(4.76)}}&\textbf{11.85(10.69)}&\multicolumn{1}{c}{\multirow{2}{*}{\textbf{6.73(4.05)}}}&\multicolumn{1}{c}{\multirow{2}{*}{14.76(4.84)}}&\textbf{11.83(10.73)}\\
\multicolumn{1}{r}{}& \multicolumn{1}{r}{}& \multicolumn{1}{r}{} & \multicolumn{1}{r}{} & [\textbf{12.66(12.87)}]  & \multicolumn{1}{r}{} & \multicolumn{1}{r}{} & [\textbf{32.27(19.08)}]  & \multicolumn{1}{r}{} & \multicolumn{1}{r}{} & [\textbf{22.04(13.05)}]  & \multicolumn{1}{r}{} & \multicolumn{1}{r}{} & [\textbf{22.10(12.96)}]  \\
\specialrule{.15em}{.05em}{.05em}
\end{tabular} 
\end{adjustbox}
\caption{Calibration results for Tiramisu semantic segmentation model on CamVid dataset. Results are reported in mean(std) format. The number of testing samples are listed in parentheses underneath the dataset name. UC denotes the uncalibrated result. $\downarrow$ denotes that lower is better. Best results are bolded and green indicates statistically significant differences w.r.t. FL+LTS. Note that due to GPU memory limits, results of MMCE and MMCE+LTS are for downsampled images, thus can not be directly compared with other methods. The goal of including them is to show that LTS can improve MMCE. LTS generally achieves the best performance on almost all metrics in the \textit{All} region, \textit{Boundary} region and \textit{Local} region.
}
\label{tab:additional_metrics}
\end{table*}

\section{Multi-atlas Segmentation and Joint Label Fusion}
\label{app:JLF_deduction}
\noindent
We give a brief overview of multi-atlas segmentation (MAS)~\cite{iglesias2015multi} and label fusion. Let $T_I$ represents the target image that needs to be segmented. Denote the $n$ atlas images and their corresponding manual segmentations as $A^1 = (A^i_I, A^i_S), A^2 = (A^2_I, A^2_S), ..., A^n = (A^n_I, A^n_S)$. MAS first employs a reliable deformable image registration method to warp all atlas images into the space of the target image $T_I$, i.e. $\tilde{A}^i = (\tilde{A}^i_I, \tilde{A}^i_S), i=1,...,n$. Each $\tilde{A^i_S}$ is considered as a candidate segmentation for $T_I$. Finally, a label fusion method~\cite{iglesias2015multi} $\mathscr{G}$ is used to produce the final estimated segmentation $\hat{T}_S$ for $T_I$, i.e.
\begin{equation}
\label{eq:est}
\hat{T}_S = \mathscr{G}(\tilde{A}^1, \tilde{A}^2, ..., \tilde{A}^n, T_I).
\end{equation}
The goal of label fusion is to use all the information from each individual candidate segmentation to generate a consensus segmentation that is better than any individual candidate segmentation. One of the most common and popular approaches of label fusion is weighted voting at each pixel/voxel of the target image, i.e.
\begin{equation}
\label{eq:weight_label_fusion}
\hat{T}_S(x) = \argmax_{l \in L} \sum_{i=1}^{n} w^i_x \cdot \mathds{1}[\tilde{A}^i_S(x) = l],
\end{equation}
where $l \in L=\{0,\dots,K\}$ is the set of labels ($K$ structures; 0 indicating background), $\mathds{1}[\cdot]$ is the indicator function, and $w^i_x$ is the weight that associates with the $i$-th atlas candidate segmentation $\tilde{A}^i_S$ at position $x$. There are a lot of possible weighting schemes. For example, majority voting (MV) and plurality voting (PV)~\cite{hansen1990neural,heckemann2006automatic} are the simplest ones that assume each atlas contributes with equal reliability to the estimate of the target segmentation, i.e. $w^i_x$ is a constant value for all $i$ and $x$. Moving forward, spatially varying weighted voting (SVWV)~\cite{artaechevarria2009combination,coupe2011patch,sabuncu2010generative} relaxes the assumption to allow for spatially varying weights, i.e. $w^i_x$ can be different for $i$ and $x$. One simple way to estimate the weight $w^i_x$ is to set it as the probability of $\tilde{A}^i_S(x)=T_S(x)$, i.e. $w^i_x = p(\tilde{A}^i_S(x) = T_S(x))$.
Though SVWV significantly improves the performance over MV and PV, it fails to consider the situation that atlases may make correlated errors. Thus, joint label fusion (JLF)~\cite{wang2012multi} has been proposed which down-weights pairs of atlases that consistently make similar errors. Specifically, JLF tries to find the optimal weights $\omega^i_x$ by minimizing the expected error between $\hat{T}_S (x)$ and the true segmentation $T_S (x)$:
\begin{equation}
\label{eq:min_expectation}
E \left [ (T_S (x) - \hat{T}_S (x))^2  \right ].
\end{equation}
Thus, label fusion weights can be computed from Eq.~\eqref{eq:weight} by minimizing the total expectation of segmentation errors of Eq.~\eqref{eq:min_expectation} constrained to $\sum_{i=1}^n \omega^i_x = 1$:
\begin{equation}
\label{eq:weight}
\textbf{w}_x = \frac{\textbf{M}^{-1}_{x} \textbf{1}_n }{\textbf{1}^t_n \textbf{M}^{-1}_{x} \textbf{1}_n},
\end{equation} 
where $\textbf{1}_n$ is a vector of all 1 and $t$ is the transpose. $\textbf{w}_x$ is the vector of weights and $w^i_x$ is its $i$-th entry (correspond to the $i$-th atlas). $\textbf{M}_x$ is a pairwise dependency matrix of size $n \times n$ where each entry $\textbf{M}_x (i, j)$ is the estimated joint probability that atlas $\tilde{A}^i_S$ (row) and $\tilde{A}^j_S$ (column) both provide wrong label suggestions for the target image $T_I$ at location $x$. $\textbf{M}_x (i, j)$ is approximated as follows:
\begin{equation}
\label{eq:probability}
\begin{split}
\textbf{M}_x (i, j) &= p(\tilde{A}^i_S(x) \neq T_S(x), \tilde{A}^j_S(x) \neq T_S(x)) \\
& \approx p(\tilde{A}^i_S(x) \neq T_S(x))p(\tilde{A}^j_S(x) \neq T_S(x)) \\
& = (1 - p(\tilde{A}^i_S(x) = T_S(x)))(1 - p(\tilde{A}^j_S(x) = T_S(x))).
\end{split}
\end{equation}

\noindent
Based on the above-mentioned label fusion approaches, the segmentation accuracy of MAS relies heavily on the accuracy of estimating the probability of the $i$-th atlas having the same label as the target image, i.e. $p(\tilde{A}^i_S(x) = T_S(x))$. Estimation of $p(\tilde{A}^i_S(x) = T_S(x))$ is rarely explored. Typically,
patch-based sum of squared differences (SSD) between image intensities are used~\cite{artaechevarria2009combination,coupe2011patch,sabuncu2010generative,wang2012multi}. Recently, deep convolutional networks based approaches~\cite{ding2019votenet,ding2020votenet+,xie2019improving} have been proposed to improve over the SSD intensity measures and have achieved great success. Here, specifically, we employ a deep convolutional neural network called VoteNet+~\cite{ding2020votenet+} to estimate the probabilities. We then conduct experiments for probability calibration to determine how much improving the calibration can improve the joint label fusion result and in turn the segmentation accuracy.

\end{document}